\documentclass{ustthesis}

\usepackage{amsmath,amssymb,amsthm}
\usepackage{enumitem}
\usepackage{graphicx}
\usepackage{bm}
\usepackage{microtype}
\usepackage[colorlinks,citecolor=blue]{hyperref}
\usepackage[margin=25mm,textheight=247mm,textwidth=145mm]{geometry}

\newtheorem{theorem}{\bf Theorem}[chapter]
\newtheorem{lemma}[theorem]{Lemma}
\newtheorem{proposition}[theorem]{Proposition}
\newtheorem{corollary}[theorem]{Corollary}

\theoremstyle{definition}
\newtheorem{definition}[theorem]{Definition}
\newtheorem{assumption}[theorem]{Assumption}

\newtheorem{remark}[theorem]{Remark}

\DeclareMathOperator*{\argmin}{argmin}

\numberwithin{equation}{chapter}

\def \bE {\mathbb{E}}

\def \bN {\mathbb{N}}

\def \bP {\mathbb{P}}

\def \bR {\mathbb{R}}

\def \cA {\mathcal{A}}
\def \cB {\mathcal{B}}

\def \cD {\mathcal{D}}
\def \cE {\mathcal{E}}
\def \cF {\mathcal{F}}
\def \cG {\mathcal{G}}
\def \cH {\mathcal{H}}

\def \cL {\mathcal{L}}
\def \cM {\mathcal{M}}
\def \cN {\mathcal{N}}
\def \cO {\mathcal{O}}
\def \cP {\mathcal{P}}

\def \cR {\mathcal{R}}
\def \cS {\mathcal{S}}
\def \cT {\mathcal{T}}
\def \cU {\mathcal{U}}

\def \cW {\mathcal{W}}
\def \cX {\mathcal{X}}

\def \Ba {{\boldsymbol{a}}}
\def \Bb {{\boldsymbol{b}}}

\def \Be {{\boldsymbol{e}}}

\def \Bi {{\boldsymbol{i}}}

\def \Bm {{\boldsymbol{m}}}
\def \Bn {{\boldsymbol{n}}}

\def \Bs {{\boldsymbol{s}}}

\def \Bx {{\boldsymbol{x}}}
\def \By {{\boldsymbol{y}}}
\def \Bz {{\boldsymbol{z}}}

\def \Lip {\,{\rm Lip}\,}
\def \Id {\,{\rm Id}\,}
\def \Pdim {\,{\rm Pdim}\,}
\def \Bin {\,{\rm Bin}\,}
\def \mid {\,{\rm mid}\,}
\def \sgn {\,{\rm sgn}\,}
\def \MMD {\,{\rm MMD}\,}

\usepackage[T1]{fontenc}
\usepackage{kpfonts}
\usepackage{libertine}
\usepackage[scaled=0.85]{beramono}
\usepackage[english]{babel}

\title{Learning Distributions by Generative Adversarial Networks: Approximation and Generalization}  
\author{Yunfei Yang}     
\degree{\PhD}              
\subject{Mathematics}      
\department{Mathematics}       
\advisor{Prof. Yang Wang}     
\depthead{Prof. Kun Xu}    
\defencedate{2022}{05}{20}      

\begin{document}


\maketitle

%



\acknowledgments

First and foremost, I would like to express my deep gratitude to my supervisor, Prof. Yang Wang, for his valuable advice, patient guidance and constant support during my PhD study. Prof. Wang has provided many interesting directions and ideas to my research, including the study of this thesis. He is also a very patient mentor and I can always get supports from him whenever I have difficulties.

Besides, I would like to thank my thesis supervision committee members, Prof. Can Yang and Prof. Jian-Feng Cai, for their suggestions and help in my research. I am also grateful to my collaborators, especially Prof. Yuling Jiao, without whom this thesis would not have such accomplishment.

I would like to offer my special thanks to Huawei PhD Fellowship Program for supporting my study and research. I have also learned a lot from colleagues in Huawei during my internship. I wish to express my sincere appreciation to Dr. Zhen Li, my advisor in Huawei, for the discussions and help in daily life.

I also want to thank my group members and friends in HKUST for their helpful discussion in research, encouragement and accompany during the study. Finally, I would like to express my sincere gratitude to my parents and sister for their unconditional love and support throughout my life.

\endacknowledgments


\tableofcontents
\listoffigurestables






\abstract

We study how well generative adversarial networks (GAN) learn probability distributions from finite samples by analyzing the convergence rates of these models. Our analysis is based on a new oracle inequality that decomposes the estimation error of GAN into the discriminator and generator approximation errors, generalization error and optimization error. To estimate the discriminator approximation error, we establish error bounds on approximating H\"older functions by ReLU neural networks, with explicit upper bounds on the Lipschitz constant of the network or norm constraint on the weights. For generator approximation error, we show that neural network can approximately transform a low-dimensional source distribution to a high-dimensional target distribution and bound such approximation error by the width and depth of neural network. Combining the approximation results with generalization bounds of neural networks from statistical learning theory, we establish the convergence rates of GANs in various settings, when the error is measured by a collection of integral probability metrics defined through H\"older classes, including the Wasserstein distance as a special case. In particular, for distributions concentrated around a low-dimensional set, we show that the convergence rates of GANs do not depend on the high ambient dimension, but on the lower intrinsic dimension.

\endabstract




\chapter{Introduction}

Deep learning is a family of machine learning and artificial intelligence methods based on artificial neural networks. It typically refers to training complex and high-dimensional models with hierarchy structure to learn representations of data. Since 2006, deep learning methods, such as convolutional neural networks, recurrent neural networks, deep reinforcement learning and transformer, have dramatically improved the state-of-the-art in many fields including computer vision, natural language processing, speech recognition, object detection, machine translation and bioinformatics \cite{goodfellow2016deep,lecun2015deep}. 

As one of the important development in deep learning, Generative Adversarial Networks (GAN) have received considerable attention and led to an explosion of new ideas, techniques and applications in deep learning, since it was designed by Goodfellow et al. \cite{goodfellow2014generative} in 2014. GAN is a framework of learning data distribution by simultaneously training two neural networks (generator and discriminator) against each other in a minimax two-player game. It has been empirically shown that this technique can generate new data with the same statistics as the training set and perform extremely well in image synthesis, medical imaging and natural language generation \cite{radford2016unsupervised,reed2016generative,zhu2017unpaired,
karras2018progressive,yi2019generative,bowman2016generating}. However, theoretical explanations for the empirical successes of GANs and other deep learning methods are not well established. Many problems on the theory and training dynamics of GANs are largely unsolved.

To understand the empirical performance of generative adversarial networks, one needs to theoretically answer the fundamental question: how well GANs learn distributions from finite samples? In this thesis, we try to provide some answers to this question by studying the effectiveness of these models. We will show that GANs are consistent estimators of distributions and establish their convergence rates in terms of the number of samples, which are optimal for learning distributions in some sense. This gives statistical guarantee for the usage of GANs in practice. Furthermore, we also quantify the required sizes of discriminator and generator that achieve the optimal convergence rates. Hopefully, this provides some guide on the design of neural networks for GANs in practice.

From the learning point of view, the effectiveness of a model can be divided into three parts: approximation, optimization and generalization. Let us take the classical setting of regression by neural networks as an example. In this setting, we train a neural network to learn an unknown function by minimizing certain loss on observed data. Approximation characterizes the bias of the model by estimating the distance between the neural network class and target function. The optimal approximation rates of deep neural networks for classical smooth function spaces have been derived recently in \cite{yarotsky2017error,yarotsky2018optimal,lu2021deep}. Optimization addresses how well we can find a solution with a minimal loss. Recent works \cite{allenzhu2019convergence,du2019gradient} showed that stochastic gradient descent can find global minima in polynomial time for over-parameterized neural networks under certain conditions. Generalization refers to the model's ability to adapt properly to unseen data. In statistical learning theory, it is often controlled by certain complexities of the neural network class, such as Pseudo-dimension and Rademacher complexity \cite{anthony2009neural,mohri2018foundations,shalevshwartz2014understanding}. If the training is successful, one can derive optimal convergence rates of deep neural networks for learning smooth functions by combining the approximation and generalization bounds \cite{schmidthieber2020nonparametric,nakada2020adaptive}.

In this thesis, we develop similar analysis for generative adversarial networks by analyzing the approximation and generalization. In GAN, we have two source of approximation error. The first one is due to the generator, which is used to transform a simple source distribution to approximate the complex unknown target distribution. Hence, to estimate this approximation error, one need to study the capacity of generative networks for approximating distributions. The second approximation error is from the discriminator. If the performance of the model is evaluated by Integral Probability Metric (IPM) between the target distribution and the distribution generated by the trained generator, then the discriminator can be regard as an approximation to the evaluation class that defines the IPM. In this case, the discriminator approximation error can be controlled by the function approximation capacity of neural networks. Similar to regression, the generalization of GAN can be analyzed by the statistical learning theory \cite{anthony2009neural,mohri2018foundations,shalevshwartz2014understanding} and bounded by the complexity of neural networks. Therefore, if the training is successful, we can combine the approximation and generalization results together and derive convergence rates for GANs.

\section{Main contributions}

The contents of this thesis are mainly from our recent works \cite{huang2022error,yang2022capacity,jiao2022approximation,yang2020approximation}. The main contributions can be divided into three categories. 

\textbf{(1) Function approximation by neural networks.} We prove two types of approximation bounds for deep ReLU neural networks. The first one quantifies the approximation error by the width and depth of neural networks. Specifically, we establish error bounds on approximating H\"older functions by neural networks, with an explicit upper bound on the Lipschitz constant of the constructed neural network functions. It is also shown that such approximation order is optimal up to logarithmic factors. The second function approximation result is for neural networks with norm constraint on the weights. We obtain approximation upper and lower bounds in terms of the norm constraint for such networks, if the network size is sufficiently large.

\textbf{(Related works)} The expressiveness and approximation capacity of neural networks have been an active research area in the past few decades. Early works in the 1990s showed that shallow neural networks, with one hidden layer and various activation functions, are universal in the sense that they can approximate any continuous functions on compact sets, provided that the width is sufficiently large \cite{cybenko1989approximation,hornik1991approximation,pinkus1999approximation,barron1993universal}. In particular, Barron \cite{barron1993universal} showed that shallow neural networks can achieve attractive approximation rates for functions satisfying certain decaying conditions on Fourier's frequency domain. The recent breakthrough of deep learning has attracted much research on the approximation theory of deep neural networks. The approximation rates of ReLU deep neural networks are extensively studied for many function classes, such as continuous functions \cite{yarotsky2017error,yarotsky2018optimal,shen2020deep}, smooth functions \cite{yarotsky2020phase,lu2021deep}, piecewise smooth functions \cite{petersen2018optimal}, shift-invariant spaces \cite{yang2020approximation} and band-limited functions \cite{montanelli2019deep}. In particular, \cite{yarotsky2017error,yarotsky2018optimal,yarotsky2020phase} characterized the approximation error by the number of parameters, and \cite{shen2020deep,lu2021deep} obtained approximation bounds in term of width and depth (or the number of neurons). Our constructions of neural networks use ideas similar to those in these papers. The approximation order, in terms of width and depth, of our constructed neural network is the same as \cite{lu2021deep}, which is proved to be optimal. But we also give an explicit bound on the Lipschitz constant of the constructed network function, which is essential for our analysis of GANs and may be of independent interest for other study. To the best of our knowledge, the approximation bounds for norm constrained neural networks is new in the literature. Since our upper bound only depend on the norm constraint, it can be used to analyze over-parameterized neural networks, which is a hot topic in recent study \cite{allenzhu2019convergence,du2019gradient,liu2022loss}. The approximation theory of convolutional neural networks (CNN) is discussed in \cite{zhou2020universality,zhou2020theory}, which showed that any fully connected neural networks can be realized by CNN with parameters of the same order. Hence, some of our approximation results can also be applied to CNN.

\textbf{(2) Distribution approximation by generative networks.} We analyze the approximation capacity of generative networks in three metrics: Wasserstein distances, maximum mean discrepancy (MMD) and $f$-divergences. Our results show that, for Wasserstein distances and MMD, generative networks are universal approximators in the sense that,  under mild conditions, they can approximately transform low-dimensional distributions to any high-dimensional distributions. The approximation bounds are obtained in terms of the width and depth of neural networks. We also show that the approximation orders in Wasserstein distances only depend on the intrinsic dimension of the target distribution. On the contrary, for $f$-divergences, it is impossibles to approximate the target distribution using neural networks, if the dimension of the source distribution is smaller than the intrinsic dimension of the target.

\textbf{(Related works)} Despite the vast amount of research on function approximation by neural networks, there are only a few papers studying the representational capacity of generative networks for approximating distributions. Let us compare our results with the most related works \cite{lee2017ability,bailey2018size,perekrestenko2020constructive,lu2020universal}. The paper \cite{lee2017ability} considered a special form of target distributions, which are push-forward measures of the source distributions via composition of Barron functions. These distributions, as they proved, can be approximated by deep generative networks. But it is not clear what probability distributions can be represented in the form they proposed. The works \cite{bailey2018size,perekrestenko2020constructive} also showed that generated networks are universal approximatior under certain restricted conditions. In \cite{bailey2018size}, the source and target distributions are restricted to uniform and Gaussian distributions. \cite{perekrestenko2020constructive} proved the case that the source distribution is uniform and the target distribution has Lipschitz-continuous density function with bounded support. We extend their results to a more general setting that the source distribution is absolutely continuous and the target only satisfies some moment conditions. In \cite{lu2020universal}, the authors showed that the gradients of neural networks, as transforms of distributions, are universal when the source and target distributions are of the same dimension. Their proof relies on the theory of optimal transport \cite{villani2008optimal}, which is only available between distributions of the same dimensions. Hence their approach cannot be simply extended to the case that the source and target distributions are of different dimensions. Our results show that neural networks can approximately transport low-dimensional distributions to high-dimensional distributions, which suggests some possible generalization of the optimal transport theory.

\textbf{(3) Convergence rates of GANs.} We develop a new oracle inequality for GAN estimators, which decomposes the estimation
error into optimization error, generator and discriminator approximation error and generalization error. When the optimization is successful, we establish the convergence rates of GANs
under a collection of integral probability metrics defined through H\"older classes, including the Wasserstein distance as a special case. We also show that GANs are able to adaptively
learn data distributions with low-dimensional structures or have H\"older densities, when the network architectures are chosen properly. In particular, for distributions concentrated
around a low-dimensional set, we show that the learning rates of GANs only depend on the intrinsic dimension of the distribution.

\textbf{(Related works)} The generalization errors of GANs have been studied in several recent works. The paper \cite{arora2017generalization} showed that, in general,  GANs do not generalize under the Wasserstein distance and the Jensen-Shannon divergence with any polynomial number of samples. Alternatively, they estimated the generalization bound under the ``neural net distance'', which is the IPM with respect to the discriminator network. The follow-up work \cite{zhang2018discrimination} improved the generalization bound in \cite{arora2017generalization} by explicitly quantifying the complexity of the discriminator network. However, these generalization theories make the assumption that the generator can approximate the data distribution well under the neural net distance, while the construction of such generator network is unknown. Also, the neural net distance is too weak that it can be small when two distributions are not very close \cite{arora2017generalization}. Similar to our results, \cite{bai2019approximability} showed that GANs are able to learn distributions in Wasserstein distance. But their theory requires each layer of the neural network generator to be invertible, and hence the width of the generator has to be the same with the input dimension, which is not the usual practice in applications. In contrast, we do not make any invertibility assumptions, and allow the discriminator and the generator networks to be wide. 

The work of \cite{chen2020statistical} is the most related to ours. They studied statistical properties of GANs and established convergence rate $\cO(n^{-\alpha/(2\alpha+d)}(\log n)^2)$ for distributions with H\"older densities and sample size $n$, when the evaluation class is another H\"older class $\cH^\alpha$. Their estimation on generator approximation is based on the optimal transport theory, which requires that the input and the output dimensions of the generator to be the same. We study the same problem as \cite{chen2020statistical} and improve their convergence rate to $\cO(n^{-\alpha/d} \lor n^{-1/2} \log n)$ for general probability distributions without any restrictions on the input and the output dimensions of the generator. Furthermore,  our results circumvent the curse of dimensionality if the data distribution has a low-dimensional structure, and establish the convergence rate $\cO((n^{-\alpha/d^*}\lor n^{-1/2})\log n)$ when the distribution concentrates around a set with Minkowski dimension $d^*$. 

There is another line of work \cite{liang2021how,singh2018nonparametric,uppal2019nonparametric} concerning the nonparametric density estimation under IPMs. In particular, the authors of \cite{liang2021how} and \cite{singh2018nonparametric} established the minimax optimal rate $\cO(n^{-(\alpha+\beta)/(2\beta+d)} \lor n^{-1/2} )$ for learning a Sobolev density class with smoothness index $\beta>0$, when the evaluation class is another Sobolev class with smoothness $\alpha$. The paper \cite{uppal2019nonparametric} generalized the minimax rate to Besov IPMs, where both the target density and the evaluation classes are Besov classes. Our result matches this optimal rate with $\beta=0$ without any assumption on the regularity of the data distribution.

The rest of this thesis is organized as follows. Chapter \ref{chapter: GAN} introduces the basic setup and proves the error decomposition of GANs. In Chapter \ref{chapter: sample comp}, we discuss some complexities of neural networks that control the generalization error. Chapter \ref{chapter: fun app} derives function approximation bounds for neural networks, which can be used to bound discriminator approximation error in GANs. Chapter \ref{chapter: dis app} studies the distribution approximation capacity of generative networks. In Chapter \ref{chapter: rate}, we combine the approximation and generalization bounds and establish the convergence rates of GANs. Finally, Chapter \ref{chapter: conclusion} concludes the thesis and discuss possible directions for future study.

\section{Preliminaries and notations}

The set of positive integers is denoted by $\bN:=\{1,2,\dots\}$. For convenience, we also use the notation $\bN_0:= \bN \cup \{0\}$.
The cardinality of a set $S$ is denoted by $|S|$.
We use $\|\Bx\|_p$ to denote the $p$-norm of a vector $\Bx\in\bR^d$.
If $X$ and $Y$ are two quantities, we denote $X\land Y:= \min\{X,Y\}$ and $X\lor Y:= \max\{X,Y\}$. We use $X \lesssim Y$ or $Y \gtrsim X$ to denote the statement that $X\le CY$ for some constant $C>0$. We denote $X \asymp Y$ when $X \lesssim Y \lesssim X$. The composition of two functions $f:\bR^d \to \bR$ and $g:\bR^k \to \bR^d$ is denoted by $f\circ g(x) := f(g(x))$. We use $\cF \circ \cG := \{ f\circ g: f\in \cF,g\in \cG \}$ to denote the composition of two function classes.

For two probability distributions (measures) $\mu$ and $\nu$, $\mu \perp \nu$ denotes that $\mu$ and $\nu$ are singular, $\mu \ll \nu$ denotes that $\mu$ is absolutely continuous with respect to $\nu$ and in this case the Radon–Nikodym derivative is denoted by $d\mu/d\nu$. We say $\mu$ is absolutely continuous if it is absolutely continuous with respect to the Lebesgue measure, which is equivalent to the statement that $\mu$ has probability density function. If $\nu$ is defined on $\bR^k$ and $g:\bR^k\to\bR^d$ is a measurable mapping, then the push-forward distribution $g_\# \nu:= \nu \circ g^{-1}$ of a measurable set $S\subseteq \bR^d$ is defined as $g_\# \nu(S) := \nu(g^{-1}(S))$, where $g^{-1}(S) := \{\Bx\in \bR^k: g(\Bx)\in S \}$.

Next, we introduce the notion for regularity of a function. For a multi-index $\Bs=(s_1,\dots,s_d)\in \bN_0^d$, we use the usual notation $\Bs ! = \prod_{i=1}^d s_i !$. The monomial on $\Bx=(x_1,\dots,x_d)^\intercal$ is denoted by $\Bx^\Bs:=x_1^{s_1}\cdots x_d^{s_d}$. The $\Bs$-derivative of a function $f$ is denoted by $\partial^\Bs f:= (\frac{\partial}{\partial x_1})^{s_1} \dots (\frac{\partial}{\partial x_d})^{s_d} f$. And we use the convention that $\partial^\Bs f :=f $ if $\|\Bs\|_1=0$.

\begin{definition}[Lipschitz functions]\label{Lipschitz function}
Let $\cX\subseteq \bR^d$ and $f:\cX \to \bR$, the Lipschitz constant of $f$ is denoted by
\[
\Lip (f) := \sup_{\Bx,\By\in \cX, \Bx\neq \By} \frac{|f(\Bx)-f(\By)|}{\|\Bx-\By\|_\infty}.
\]
We denote $\Lip(\cX,K)$ as the set of all functions $f:\cX \to \bR$ with $\Lip (f)\le K$. For any $B>0$, we denote $\Lip(\cX,K,B):= \{ f\in \Lip(\cX,K): \|f\|_{L^\infty(\cX)} \le B \}$.
\end{definition}

\begin{definition}[H\"older classes]\label{Holder class}
Let $d\in \bN$ and $\alpha = r+\alpha_0>0$, where $r\in \bN_0 $ and $\alpha_0\in (0,1]$. We denote the H\"older class $\cH^\alpha(\bR^d)$ as
\[
\cH^\alpha(\bR^d) := \left\{ h:\bR^d\to \bR, \max_{\|\Bs\|_1\le r} \sup_{\Bx\in \bR^d} |\partial^\Bs h(\Bx)| \le 1, \max_{\|\Bs\|_1=r} \sup_{\Bx\neq \By} \frac{|\partial^\Bs h(\Bx)- \partial^\Bs h(\By)|}{\|\Bx-\By\|_\infty^{\alpha_0}}\le 1 \right\},
\]
where the multi-index $\Bs\in \bN_0^d$. For any $\cX\subseteq \bR^d$, denote $\cH^\alpha(\cX) := \{ h:\cX\to \bR, h\in \cH^\alpha(\bR^d) \}$ as the restriction of $\cH^\alpha(\bR^d)$ to $\cX$. In particular, for $\cX=[0,1]^d$, denote $\cH^\alpha := \cH^\alpha([0,1]^d)$.
\end{definition}

It should be noticed that for $\alpha=r+1$, we do \emph{not} assume that $h\in C^{r+1}$. Instead, we only require that $h\in C^r$ and its derivatives of order $r$ are Lipschitz continuous with respect to the metric $\|\cdot\|_\infty$. We also note that, if $\alpha\le 1$, $|h(\Bx) - h(\By)| \le \|\Bx-\By\|_\infty^\alpha$; if $\alpha>1$, $|h(\Bx) - h(\By)| \le d \|\Bx-\By\|_\infty$. In particular, with the above definitions, $\cH^1([0,1]^d) = \Lip([0,1]^d,1,1)$.

Finally, we list a set of notations that are used throughout this thesis in Table \ref{notation}. Some of the notations will be introduced in later chapters.

\renewcommand{\arraystretch}{1.3}
\begin{table}[htbp]
\centering
\begin{tabular}{|l|l|}
\hline
Notation & Definition \\
\hline
$\Lip(\cX,K,B)$ & The set of $f:\cX\to \bR$ with $\Lip(f)\le K$ and $\|f\|_{L^\infty(\cX)}\le B$, Definition \ref{Lipschitz function}  \\
\hline
$\cH^\alpha(\cX)$ & H\"older class of regularity $\alpha$ on $\cX$, Definition \ref{Holder class}  \\
\hline
$\cN\cN(W,L)$ & Neural network with width $W$ and depth $L$, parameterized by Eq. (\ref{NN standard form})  \\
\hline
$\cN\cN(W,L,K)$ & Neural network $f_\theta\in \cN\cN(W,L)$ with norm constraint $\kappa(\theta)\le K$, Eq. (\ref{norm constraint})  \\
\hline
$d_\cH(\mu,\gamma)$ & Integral probability metric (IPM) between distributions $\mu$ and $\gamma$, Eq. (\ref{IPM}) \\
\hline
$\cW_p(\mu,\gamma)$ & $p$-th Wasserstein distance between distributions $\mu$ and $\gamma$, Eq. (\ref{W_p}) \\
\hline
$\MMD(\mu,\gamma)$ & Maximum mean discrepancy between distributions $\mu$ and $\gamma$, Eq. (\ref{MMD}) \\
\hline
$\cD_f(\mu \| \gamma)$ & $f$-divergence between distributions $\mu$ and $\gamma$, Eq. (\ref{f-div}) \\
\hline
$\cE(\cH,\cF,\Omega)$ & Approximation error of $\cH$ on $\Omega$ by approximator in $\cF$, Lemma \ref{error decomposition} \\
\hline
$\cR_n(S)$ & Rademacher complexity of a set $S\subseteq \bR^n$, Definition \ref{Rademacher complexity} \\
\hline
$\cF(\Bx_{1:n})$ & The set of function values $\{(f(\Bx_1),\dots,f(\Bx_n)):f\in \cF \} \subseteq \bR^n$ \\
\hline 
$\cN_c(S,\rho,\epsilon)$ & $\epsilon$-covering number of $S$ under metric $\rho$, Definition \ref{covering and packing} \\
\hline 
$\cN_p(S,\rho,\epsilon)$ & $\epsilon$-packing number of $S$ under metric $\rho$, Definition \ref{covering and packing} \\
\hline 
$\Pdim(\cF)$ & Pseudo-dimension of function class $\cF$, Definition \ref{Pdim}\\
\hline
$\dim_H$, $\dim_M$ & Hausdorff dimension and Minkowski dimension, Definition \ref{Hausdorff Minkowski dimensions}\\
\hline
$\cS^d(x_0,\dots,x_{N})$ & CPwL functions $f:\bR \to \bR^d$ with breakpoints $x_0<x_1<\dots<x_{N}$, Section \ref{sec: interpolation} \\
\hline
\end{tabular}
\caption{A list of notations used throughout the thesis.}\label{notation}
\end{table}

\chapter{Generative Adversarial Networks}\label{chapter: GAN}

In this chapter, we introduce the  basic setup and notations for neural networks and GANs. We also derive an error decomposition for GANs, which will be used to study the convergence rates of GANs in later chapters.

\section{Neural networks}

A feed-forward artificial neural network is a computing system inspired by the biological neural networks. Mathematically, we can define (fully connected feed-forward) neural networks as follows: 
Let $L,N_1,\dots, N_L$ be positive integers. A neural network function $\phi:\bR^{d} \to \bR^{k}$ is a function that can be parameterized in the form
\begin{equation}\label{NN standard form}
\begin{aligned}
\phi_0(\Bx) &= \Bx, \\
\phi_{\ell+1}(\Bx) &= \sigma(A_{\ell} \phi_{\ell}(\Bx)+\Bb_\ell), \quad \ell = 0,\dots,L-1, \\
\phi(\Bx) &= A_L \phi_L(\Bx) + \Bb_L,
\end{aligned}
\end{equation}
where $A_\ell \in \bR^{N_{\ell+1}\times N_{\ell}}$, $\Bb_\ell\in \bR^{N_{\ell+1}}$ with $N_0 =d$ and $N_{L+1} =k$. The activation function $\sigma:\bR \to \bR$ is applied element-wise. We will always assume that $\sigma(x) := \max\{x,0\}= x\lor 0$ is the Rectified Linear Unit function (ReLU), which is widely used in modern applications \cite{nair2010rectified}. The numbers $W:=\max\{N_1,\dots,N_L\}$ and $L$ are called the width and depth of neural network, respectively. We denote the neural network $\cN\cN_{d,k}(W,L)$ as the set of functions that can be parameterized in the form (\ref{NN standard form}) with width $W$ and depth $L$. In this thesis, we often omit the subscripts and simply denote it by $\cN\cN(W,L)$, when the input dimension $d$ and output dimension $k$ are clear from contexts. Sometimes, we will use the notation $\phi_\theta \in \cN\cN(W,L)$ to emphasize that the neural network function $\phi_\theta$ is parameterized by
\[
\theta:= ((A_0,\Bb_0),\dots,(A_L,\Bb_L)).
\]

Next, we are going to define norm constraint on the weights for the neural network $\cN\cN(W,L)$, which will be useful when we want to regularize the network. To begin with, we consider a special class of neural network functions $\cS\cN\cN(W,L)$ which contains functions of the form
\begin{equation}\label{NN special form}
\widetilde{\phi}(\Bx) = \widetilde{A}_L \sigma(\widetilde{A}_{L-1}\sigma(\cdots \sigma(\widetilde{A}_0 \widetilde{\Bx}) ) ), \quad \widetilde{\Bx} :=
\begin{pmatrix}
\Bx \\
1
\end{pmatrix},
\end{equation}
where $\widetilde{A}_\ell \in \bR^{N_{\ell+1} \times N_\ell}$ and $\max\{N_1,\dots,N_L\}= W$. Since these functions can also be written in the form (\ref{NN standard form}) with 
\[
(A_0,\Bb_0) = \widetilde{A}_0, \quad (A_\ell, \Bb_\ell) = (\widetilde{A}_\ell, \boldsymbol{0}), \quad 1\le \ell \le L,
\]
we know that $\cS\cN\cN(W,L)\subseteq \cN\cN(W,L)$. There is a natural way to introduce norm constraint on the weights for $\cS\cN\cN(W,L)$ \cite{bartlett2017spectrally,golowich2018size}: for any $K\ge 0$, we denote by $\cS\cN\cN(W,L,K)$ the set of functions in the form (\ref{NN special form}) that satisfies
\[
\prod_{\ell=0}^L \|\widetilde{A}_\ell\| \le K,
\]
where $\|A\|$ is some norm of a matrix $A = (a_{i,j}) \in \bR^{m\times n}$. For simplicity, we will only consider the operator norm defined by $\| A\| := \sup_{\|\Bx\|_\infty \le 1} \|A\Bx\|_\infty$ in this thesis. It is well-known that $\| A\|$ is the maximum $1$-norm of the rows of $A$:
\begin{equation}\label{norm}
\| A\| = \max_{1\le i\le m} \sum_{j=1}^{n} |a_{i,j}|.
\end{equation}
Hence, we make a constraint on the $1$-norm of the incoming weights of each neuron.

To introduce norm constraint for the class $\cN\cN(W,L)$, we observe that any $\phi \in \cN\cN(W,L)$ parameterized as (\ref{NN standard form}) can be written in the form (\ref{NN special form}) with
\[
\widetilde{A}_L = (A_L,\Bb_L), \quad
\widetilde{A}_\ell =
\begin{pmatrix}
A_\ell & \Bb_\ell \\
\boldsymbol{0} & 1
\end{pmatrix}
,\ \ell = 0,\dots,L-1,
\]
and
\begin{equation}\label{norm relation}
\prod_{\ell=0}^L \|\widetilde{A}_\ell\| = \|(A_L,\Bb_L)\| \prod_{\ell =0}^{L-1} \max\{\| (A_\ell,\Bb_\ell)\|,1\}.
\end{equation}
Hence, we define the norm constrained neural network $\cN\cN(W,L,K)$ as the set of functions $\phi_\theta \in \cN\cN(W,L)$ of the form (\ref{NN standard form}) that satisfies the following norm constraint on the weights
\begin{equation}\label{norm constraint}
\kappa(\theta) := \|(A_L,\Bb_L)\| \prod_{\ell =0}^{L-1} \max\{\| (A_\ell,\Bb_\ell)\|,1\} \le K.
\end{equation}
The following proposition summarizes the relation between the two neural network function classes $\cN\cN(W,L,K)$ and $\cS\cN\cN(W,L,K)$. It shows that we can essentially regard these two classes as the same when studying their expressiveness.

\begin{proposition}\label{network class relation}
$\cS\cN\cN(W,L,K) \subseteq \cN\cN(W,L,K) \subseteq \cS\cN\cN(W+1,L,K)$.
\end{proposition}
\begin{proof}
By the definition (\ref{norm constraint}) and the relation (\ref{norm relation}), it is easy to see that $\cN\cN(W,L,K) \subseteq \cS\cN\cN(W+1,L,K)$. Conversely, for any $\widetilde{\phi}\in \cS\cN\cN(W,L,K)$ of the form (\ref{NN special form}), by the absolute homogeneity of ReLU function, we can always rescale $\widetilde{A}_\ell$ such that $\|\widetilde{A}_L\| \le K$ and $\|\widetilde{A}_\ell\|=1$ for $\ell \neq L$. Since the function $\widetilde{\phi}$ can also be parameterized in the form (\ref{NN standard form}) with $\theta = ( \widetilde{A}_0,  (\widetilde{A}_1,\boldsymbol{0}),\dots,(\widetilde{A}_L,\boldsymbol{0}))$ and $\kappa(\theta) = \prod_{\ell=0}^L \|\widetilde{A}_\ell\|\le K$, we have $\widetilde{\phi}\in \cN\cN(W,L,K)$.
\end{proof}

The next proposition shows that we can always normalize the weights of $\phi \in \cN\cN(W,L,K)$ such that the norm of each weight matrix in hidden layers is at most one.

\begin{proposition}[Rescaling]\label{normalize}
Every $\phi\in \cN\cN(W,L,K)$ can be written in the form (\ref{NN standard form}) such that $\|(A_L,\Bb_L)\|\le K$ and $\| (A_\ell,\Bb_\ell)\|\le 1$ for $0\le \ell \le L-1$.
\end{proposition}
\begin{proof}
We first parameterize $\phi$ in the form (\ref{NN standard form}) and denote $k_\ell:=\max \{\| (A_\ell,\Bb_\ell)\|,1\}$ for all $0\le \ell \le L-1$. We let $A'_\ell = A_\ell/k_\ell$, $\Bb'_\ell = \Bb_\ell/(\prod_{i=0}^{\ell}k_i)$, $A'_L=A_L \prod_{i=0}^{L-1}k_i$ and consider the new parameterization of $\phi$:
\[
\phi'_{\ell+1}(\Bx) = \sigma(A'_\ell \phi'_\ell(\Bx)+\Bb'_\ell), \quad \phi'_0(\Bx)=\Bx.
\]
It is easy to check that
\[
\|(A'_L,\Bb_L)\| = \left\| \left(A_L, \frac{\Bb_L}{\prod_{i=0}^{L-1}k_i} \right) \right\| \prod_{i=0}^{L-1}k_i \le \|(A_L,\Bb_L)\| \prod_{i=0}^{L-1}k_i \le K,
\]
where the second inequality is due to $k_i\ge 1$ and the representation (\ref{norm}) of the norm, and
\[
\| (A'_\ell, \Bb'_\ell) \| = \frac{1}{k_\ell} \left\| \left(A_\ell, \frac{\Bb_\ell}{\prod_{i=0}^{\ell-1}k_i} \right) \right\| \le \frac{1}{k_\ell} \| (A_\ell, \Bb_\ell) \| \le 1.
\]

Next, we show that $\phi_\ell(\Bx) = \left( \prod_{i=0}^{\ell-1} k_i \right) \phi'_{\ell}(\Bx)$ by induction. For $\ell =1$, by the absolute homogeneity of ReLU function,
\[
\phi_1(\Bx) = \sigma(A_0 \Bx+\Bb_0) = k_0 \sigma(A'_0 \Bx+\Bb'_0) = k_0 \phi'_1(\Bx).
\]
Inductively, one can conclude that
\begin{align*}
\phi_{\ell+1}(\Bx) &= \sigma(A_{\ell} \phi_{\ell}(\Bx)+\Bb_\ell) = \left( \prod_{i=0}^\ell k_i \right) \sigma\left( A'_\ell \frac{\phi_\ell(\Bx)}{\prod_{i=0}^{\ell-1} k_i} + \Bb'_\ell \right) \\
&= \left( \prod_{i=0}^\ell k_i \right)  \sigma\left( A'_\ell \phi'_\ell(\Bx) + \Bb'_\ell \right) = \left( \prod_{i=0}^\ell k_i \right) \phi'_{\ell+1}(\Bx),
\end{align*}
where the third equality is due to induction. Therefore,
\[
\phi(\Bx) = A_L\phi_L(\Bx) + \Bb_L = A_L \left( \prod_{i=0}^{L-1} k_i \right) \phi'_{L}(\Bx) + \Bb_L= A'_L \phi'_L (\Bx)+ \Bb_L,
\]
which means $\phi$ can be parameterized by $((A'_0,\Bb'_0),\dots,(A'_{L-1},\Bb'_{L-1}), (A'_L,\Bb_L))$ and we finish the proof.
\end{proof}

In the following proposition, we summarize some basic operations on neural networks. These operations will be useful for construction of neural networks, when we study the approximation capacity.

\begin{proposition}\label{basic construct}
Let $\phi_1\in \cN\cN_{d_1,k_1}(W_1,L_1,K_1)$ and $\phi_2\in \cN\cN_{d_2,k_2}(W_2,L_2,K_2)$.
\begin{enumerate}[label=\textnormal{(\roman*)},parsep=0pt]
\item \textnormal{\textbf{(Inclusion)}} If $d_1=d_2$, $k_1=k_2$, $W_1\le W_2$, $L_1\le L_2$ and $K_1\le K_2$, then $\cN\cN_{d_1,k_1}(W_1,L_1,K_1) \subseteq \cN\cN_{d_2,k_2}(W_2,L_2,K_2)$.

\item \textnormal{\textbf{(Composition)}} If $k_1 = d_2$, then $\phi_2 \circ \phi_1 \in \cN\cN_{d_1,k_2}(\max\{W_1,W_2\},L_1+L_2, K_2\max\{K_1,1\})$.
Furthermore, if $A\in \bR^{d_2\times d_1}$, $\Bb\in \bR^{d_2}$ and define the function $\phi(\Bx) :=\phi_2(A\Bx+\Bb)$ for $\Bx\in \bR^{d_1}$, then $\phi\in \cN\cN_{d_1,k_2}(W_2,L_2, K_2 \max\{\|(A,\Bb)\|,1\})$.

\item \textnormal{\textbf{(Concatenation)}} If $d_1=d_2$, define $\phi(\Bx):=(\phi_1(\Bx),\phi_2(\Bx))$, then $\phi\in \cN\cN_{d_1,k_1+k_2}(W_1+W_2,\max\{L_1,L_2\},\max\{K_1,K_2\})$.

\item \textnormal{\textbf{(Linear Combination)}} If $d_1=d_2$ and $k_1=k_2$, then, for any $c_1,c_2\in\bR$, $c_1\phi_1 + c_2\phi_2 \in \cN\cN_{d_1,k_1}(W_1+W_2, \max\{L_1,L_2\}, |c_1|K_1+|c_2|K_2)$.
\end{enumerate}
\end{proposition}

\begin{proof}
By Proposition \ref{normalize}, we can parameterize $\phi_i$, $i=1,2$, in the form (\ref{NN standard form}) with parameters $((A^{(i)}_0,\Bb^{(i)}_0),\dots,(A^{(i)}_{L_i},\Bb^{(i)}_{L_i}))$ such that $\| (A^{(i)}_{L_i},\Bb^{(i)}_{L_i})\| \le K_i$ and $\| (A^{(i)}_\ell, \Bb^{(i)}_\ell)\| \le 1$ for $0\le \ell \le L_i-1$.

(i) We can assume that $A^{(1)}_\ell \in \bR^{W_2\times W_2}$ and $\Bb^{(1)}_\ell \in \bR^{W_2}$, $0\le \ell \le L_1-1$, by adding suitable zero rows and columns to $A^{(1)}_\ell$ and $\Bb^{(1)}_\ell$ if necessary (this operation does not change the norm). Then, $\phi_1$ can also be parameterized by the parameters
\[
\left( \left(A^{(1)}_0, \Bb^{(1)}_0\right),\dots, \left(A^{(1)}_{L_1-1},\Bb^{(1)}_{L_1-1}\right), \underbrace{\left(\Id,\boldsymbol{0} \right),\dots, \left(\Id,\boldsymbol{0} \right)}_{L_2-L_1 \mbox{ times }}, (A^{(1)}_{L_1},\Bb_{L_1}) \right),
\]
where $\Id$ is the identity matrix. Hence, $\phi_1 \in \cN\cN_{d_2,k_2}(W_2,L_2,K_2)$.

(ii) By (i), we can assume $W_1=W_2$ without loss of generality. Then, $\phi_2 \circ \phi_1$ can be parameterized by
\[
\left( \left(A^{(1)}_0, \Bb^{(1)}_0\right),\dots, \left(A^{(1)}_{L_1-1},\Bb^{(1)}_{L_1-1}\right), \left(A^{(2)}_0 A^{(1)}_{L_1}, A^{(2)}_0 \Bb^{(1)}_{L_1} + \Bb^{(2)}_0\right), \left(A^{(2)}_1, \Bb^{(2)}_1\right),\dots, \left(A^{(2)}_{L_2},\Bb^{(2)}_{L_2}\right) \right).
\]
We observe that
\begin{align*}
\left\| \left(A^{(2)}_0 A^{(1)}_{L_1}, A^{(2)}_0 \Bb^{(1)}_{L_1} + \Bb^{(2)}_0\right)\right\| &= \left\| \left( A^{(2)}_0, \Bb^{(2)}_0 \right)
\begin{pmatrix}
A^{(1)}_{L_1} & \Bb^{(1)}_{L_1} \\
\boldsymbol{0} & 1
\end{pmatrix} \right\| \\
&\le \left\| \left( A^{(2)}_0, \Bb^{(2)}_0 \right) \right\| \left\|
\begin{pmatrix}
A^{(1)}_{L_1} & \Bb^{(1)}_{L_1} \\
\boldsymbol{0} & 1
\end{pmatrix} \right\| \le \max\{K_1,1\}.
\end{align*}
Hence, $\phi_2 \circ \phi_1 \in \cN\cN_{d_1,k_2}(W_1,L_1+L_2, K_2\max\{K_1,1\})$.

The result for the function $\phi(\Bx) :=\phi_2(A\Bx+\Bb)$ can be derived similarly, because it is a composition of $\phi_2$ with $\phi_1(\Bx) = A\Bx+\Bb$, which can be regard as a nerual network with depth zero.

(iii) By (i), we can assume that $L_1=L_2$. Then, $\phi$ can be parameterized by the parameters $((A_0,\Bb_0),\dots,(A_{L_1},\Bb_{L_1}))$ where
\[
A_\ell :=
\begin{pmatrix}
A^{(1)}_\ell & \boldsymbol{0}  \\
\boldsymbol{0} & A^{(2)}_\ell
\end{pmatrix},
\qquad \Bb_\ell :=
\begin{pmatrix}
\Bb^{(1)}_\ell \\
\Bb^{(2)}_\ell
\end{pmatrix}.
\]
The conclusion follows easily from
\[
\| (A_\ell,\Bb_\ell)\| = \left\|
\begin{pmatrix}
A^{(1)}_\ell & \boldsymbol{0} & \Bb^{(1)}_\ell  \\
\boldsymbol{0} & A^{(2)}_\ell & \Bb^{(2)}_\ell
\end{pmatrix} \right\| = \max\left\{ \left\|(A^{(1)}_{\ell}, \Bb^{(1)}_\ell) \right\|, \left\|(A^{(2)}_{\ell}, \Bb^{(2)}_\ell)\right\| \right\},
\]
because of the expression (\ref{norm}) of the norm.

(iv) Replacing the matrix $(A_{L_1},\Bb_{L_1})$ in (iii) by $(c_1A^{(1)}_{L_1}, c_2A^{(2)}_{L_1}, c_1\Bb^{(1)}_L+c_2\Bb^{(2)}_L)$, the conclusion follows from
\[
\left\| \left(c_1A^{(1)}_{L_1}, c_2A^{(2)}_{L_1}, c_1\Bb^{(1)}_L+c_2\Bb^{(2)}_L \right)\right\|\le |c_1| \left\|\left(A^{(1)}_{L_1},\Bb^{(1)}_{L_1}\right)\right\| + |c_2| \left\|\left(A^{(2)}_{L_1},\Bb^{(2)}_{L_1}\right)\right\| \le |c_1|K_1+|c_2|K_2,
\]
where we use the property of (\ref{norm}) in the first inequality.
\end{proof}

In the statistical analysis of learning algorithms, we often require that the hypothesis class is uniformly bounded. For neural networks, this can be achieved by adding an additional clipping layer to the output. For example, let us denote, for any $B>0$,
\[
\cN\cN_{d,k}^B(W,L,K) = \left\{\phi\in\cN\cN_{d,k}(W,L,K): \phi(\Bx)\in [-B,B]^k, \forall \Bx\in \bR^d \right\},
\]
which represent the neural network classes uniformly bounded by $B$. Observe that we can always truncate the output of $\phi\in \cN\cN_{d,k}(W,L,K)$ by applying $\chi_B(x) = (x \lor -B) \land B$ element-wise. Since
\[
\chi_B(x) = \sigma(x) - \sigma(-x) -(B+1)\sigma\left(\tfrac{x}{B+1}-\tfrac{B}{B+1}\right) +(B+1)\sigma\left(-\tfrac{x}{B+1}-\tfrac{B}{B+1}\right),
\]
it is not hard to see that the truncation $\chi_B(\phi) \in \cN\cN_{d,k}^B(\max\{W,4k\},L+1,(2B+4)\max\{K,1\})$ by Proposition \ref{basic construct}.

\section{Framework of GANs}

The task of distribution estimation is to estimate an unknown probability distribution $\mu$ from its observed samples. Different from classical density estimation methods, generative adversarial networks implicitly learn the data distribution by training a generator that approximately transport low-dimensional simple distribution $\nu$ to the target $\mu$. More specifically, to estimate a target distribution $\mu$ defined on $\bR^d$, one chooses an easy-to-sample source distribution $\nu$ on $\bR^k$ (for example, uniform or Gaussian distribution) and  computes the generator $g:\bR^k \to \bR^d$ by minimizing certain distance (or discrepancy) between $\mu$ and the push-forward distribution $\gamma = g_\# \nu$. We will mainly focus on the Integral Probability Metric (IPM, see \cite{muller1997integral}) defined by 
\begin{equation}\label{IPM}
d_\cH(\mu,\gamma) := \sup_{h\in \cH} \bE_\mu[h] - \bE_\gamma [h],
\end{equation}
where $\cH$ is a function class that contains functions $h:\bR^d\to \bR$. By specifying $\cH$ differently, one can obtain a list of commonly-used metrics:
\begin{itemize}[parsep=0pt]
\item when $\cH=\{h:\Lip (h)\le 1\}$ is the $1$-Lipschitz function class, then $d_\cH=\cW_1$ is the $1$-Wasserstein distance (see (\ref{W_p}) for definition), which is used in the Wasserstein GAN \cite{arjovsky2017wasserstein};

\item when $\cH=\{h:\|h\|_{L^\infty} \le 1, \Lip (h) \le 1 \}$ is a uniformly bounded Lipschitz function class, then $d_\cH$ is the Dudley metric, which metricizes weak convergence \cite{dudley2018real};

\item when $\cH=\{h\in C(\bR^d): \|h\|_{L^\infty}\le 1 \}$ is the set of continuous function, then $d_\cH$ is the total variation distance;

\item when $\cH$ is a Sobolev function class with certain regularity, $d_\cH$ is used in Sobolev GAN \cite{mroueh2018sobolev};

\item when $\cH$ is the unit ball of some reproducing kernel Hilbert space, then $d_\cH$ is the maximum mean discrepancy \cite{gretton2012kernel,dziugaite2015training, li2015generative}, see also Chapter \ref{chapter: dis app}.
\end{itemize}
We will mainly study the case that $\cH$ is a H\"older class $\cH^\alpha (\bR^d)$, which covers a wide range of applications.

Note that the vanilla GAN proposed by \cite{goodfellow2014generative} uses the Jensen–Shannon divergence $\cD_{JS}$, rather than the IPM $d_\cH$. As discussed in \cite{nowozin2016f}, the vanilla GAN can be regarded as a special $f$-GAN, which use $f$-divergence $\cD_f$ as discrepancy. We will discuss the drawback of using $f$-divergences from an approximation point of view in Section \ref{sec: f-div}.

In practice, the evaluation class $\cH$ is approximated by another function class $\cF$, which is easy to implement, and we compute the generator by solving the following minimax optimization problem, at the population level,
\begin{equation}\label{minimax}
\argmin_{g \in \cG} d_\cF(\mu, g_\# \nu) = \argmin_{g\in \cG} \sup_{f\in \cF} \left\{ \bE_{\Bx\sim \mu} [f(\Bx)] - \bE_{\Bz\sim \nu} [f(g(\Bz))] \right\},
\end{equation}
where the generator class $\cG$ and discriminator class $\cF$ are often parameterized by neural networks. Since we only have a set of random samples $X_{1:n}=(X_i)_{i=1}^n$ that are independent and identically distributed (i.i.d.) as $\mu$ in practical applications, we estimate the expectation in (\ref{minimax}) by the empirical average and hence GANs learn the distribution $\mu$ by solving the optimization problem
\begin{equation}\label{GAN1}
\argmin_{g \in \cG} d_\cF(\widehat{\mu}_n, g_\# \nu) = \argmin_{g \in \cG} \sup_{f\in \cF} \left\{ \frac{1}{n} \sum_{i=1}^{n} f(X_i) - \bE_{\nu} [f \circ g] \right\},
\end{equation}
where $\widehat{\mu}_n = \frac{1}{n} \sum_{i=1}^{n} \delta_{X_i}$ is the empirical distribution. In a more practical setting, we can only estimate $\nu$ through its empirical distribution $\widehat{\nu}_m = \frac{1}{m} \sum_{j=1}^{m} \delta_{Z_j}$, then the optimization problem (\ref{GAN1}) becomes 
\begin{equation}\label{GAN2}
\argmin_{g \in \cG} d_\cF(\widehat{\mu}_n, g_\# \widehat{\nu}_m) = \argmin_{g \in \cG} \sup_{f\in \cF} \left\{ \frac{1}{n} \sum_{i=1}^{n} f(X_i) - \frac{1}{m} \sum_{j=1}^{m} f(g(Z_j)) \right\}.
\end{equation}
Intuitively, when $m$ is sufficiently large, the solutions of (\ref{GAN1}) and (\ref{GAN2}) should be close. We will certify this in Chapter \ref{chapter: rate} by showing that they can achieve the same convergence rate if $m$ is larger than some order of $n$.

\section{Error decomposition of GANs} \label{sec: error decomp}

Let $g^*_n$ and $g^*_{n,m}$ be solutions of the optimization problems (\ref{GAN1}) and (\ref{GAN2}) with optimization error $\epsilon_{opt}\ge 0$. In other words,
\begin{align}
g^*_n &\in \left\{g\in \cG: d_\cF(\widehat{\mu}_n, g_\# \nu) \le \inf_{\phi\in \cG} d_\cF(\widehat{\mu}_n, \phi_\# \nu) + \epsilon_{opt} \right\}, \label{gan estimator g_n} \\
g^*_{n,m} &\in \left\{g\in \cG: d_\cF(\widehat{\mu}_n, g_\# \widehat{\nu}_m) \le \inf_{\phi\in \cG} d_\cF(\widehat{\mu}_n, \phi_\# \widehat{\nu}_m) + \epsilon_{opt} \right\}. \label{gan estimator g_nm}
\end{align}
If the training of GAN is successful, the push-forward distributions $(g^*_n)_\# \nu$ and $(g^*_{n,m})_\# \nu$ should be close to the target distribution $\mu$. In order to analyze the convergence rates of $(g^*_n)_\# \nu$ and $(g^*_{n,m})_\# \nu$, we decompose the error into several terms and estimate them separately in later chapters. The error decomposition is summarized in the following lemma.

\begin{lemma}\label{error decomposition}
Assume $\mu$ and $g_\#\nu$ are supported on $\Omega\subseteq \bR^d$ for all $g\in \cG$. Suppose $\cF$ is a symmetric function class defined on $\Omega$, i.e., $f\in \cF$ implies $-f \in \cF$. Let $g^*_n$ and $g^*_{n,m}$ be the GAN estimators (\ref{gan estimator g_n}) and (\ref{gan estimator g_nm}) respectively. Then, for any function class $\cH$ defined on $\Omega$,
\begin{align*}
d_\cH(\mu,(g^*_n)_\# \nu) &\le \epsilon_{opt} + 2\cE(\cH,\cF,\Omega)  + \inf_{g \in \cG} d_\cF(\widehat{\mu}_n,g_\# \nu) + d_\cF(\mu,\widehat{\mu}_n) \land d_\cH(\mu,\widehat{\mu}_n), \\
d_\cH(\mu,(g^*_{n,m})_\# \nu) &\le \epsilon_{opt} + 2\cE(\cH,\cF,\Omega)  + \inf_{g \in \cG} d_\cF(\widehat{\mu}_n,g_\# \nu) + d_\cF(\mu,\widehat{\mu}_n) \land d_\cH(\mu,\widehat{\mu}_n) + 2d_{\cF \circ \cG}(\nu,\widehat{\nu}_m),
\end{align*}
where $\cE(\cH,\cF,\Omega)$ is the approximation error of $\cH$ on $\Omega$:
\[
\cE(\cH,\cF,\Omega) := \sup_{h\in \cH} \inf_{f\in \cF} \|h-f\|_{L^\infty(\Omega)}.
\]
\end{lemma}

Note that the error $d_{\cH}(\mu,(g^*_n)_\# \nu)$ is decomposed into four error terms: (1) the optimization error $\epsilon_{opt}$ depending on how well we can solve the optimization problem; (2) discriminator approximation error $\cE(\cH,\cF,\Omega)$ measuring how well the discriminator $\cF$ approximates the evaluation class $\cH$; (3) generator approximation error $\inf_{g \in \cG} d_\cF(\widehat{\mu}_n,g_\# \nu)$ measuring the approximation capacity of the generator; and (4) generalization error (statistical error) $d_{\cF}(\mu,\widehat{\mu}_n) \land d_{\cH}(\mu,\widehat{\mu}_n)$ due to the fact that we only have finite samples of $\mu$. For the estimator $(g^*_{n,m})_\# \nu$, we have an extra generalization error $d_{\cF \circ \cG}(\nu,\widehat{\nu}_m)$ because we estimate $\nu$ by its empirical distribution. We will study the generalization error in chapter \ref{chapter: sample comp}, the discriminator approximation error in chapter \ref{chapter: fun app}, the generator approximation error in chapter \ref{chapter: dis app} and estimate the convergence rates in chapter \ref{chapter: rate}.

The proof of Lemma \ref{error decomposition} is based on the following useful lemma, which states that for any two probability distributions, the difference in IPMs with respect to two distinct evaluation classes will not exceed two times the approximation error between the two evaluation classes. 

\begin{lemma}\label{IPM comparision}
For any probability distributions $\mu$ and $\gamma$ supported on $\Omega\subseteq \bR^d$,
\[
d_\cH(\mu,\gamma) \le d_\cF(\mu,\gamma) + 2\cE(\cH,\cF,\Omega).
\]
\end{lemma}
\begin{proof}
For any $\epsilon>0$, there exists $h_\epsilon\in \cH$ such that
\[
d_\cH(\mu,\gamma) =  \sup_{h\in \cH} \{ \bE_\mu [h] - \bE_\gamma [h] \} \le \bE_\mu [h_\epsilon] - \bE_\gamma [h_\epsilon] +\epsilon.
\]
Choose $f_\epsilon\in \cF$ such that $\|h_\epsilon-f_\epsilon\|_{L^\infty (\Omega)} \le \inf_{f\in \cF} \|h_\epsilon-f\|_{L^\infty (\Omega)} +\epsilon$, then
\begin{align*}
d_\cH(\mu,\gamma) \le& \bE_\mu [h_\epsilon-f_\epsilon] - \bE_\gamma [h_\epsilon-f_\epsilon]  + \bE_\mu [f_\epsilon] - \bE_\gamma [f_\epsilon] +\epsilon \\
\le& 2\|h_\epsilon-f_\epsilon\|_{L^\infty (\Omega)} + \bE_\mu [f_\epsilon] - \bE_\gamma [f_\epsilon] +\epsilon \\
\le& 2\inf_{f\in \cF} \|h_\epsilon-f\|_{L^\infty (\Omega)} + 2\epsilon + d_\cF(\mu,\gamma) + \epsilon \\
\le& 2\cE(\cH,\cF,\Omega) + d_\cF(\mu,\gamma) + 3\epsilon,
\end{align*}
where we use the assumption that $\mu$ and $\gamma$ are supported on $\Omega$ in the second inequality, and use the definition of IPM $d_\cF$ in the third inequality. Letting $\epsilon \to 0$, we get the desired result.
\end{proof}

The next lemma gives an error decomposition of GAN estimators associated with an estimator $\widetilde{\mu}_n$ of the target distribution $\mu$. Lemma \ref{error decomposition} is a special case of this lemma with $\widetilde{\mu}_n = \widehat{\mu}_n$ being the empirical distribution. In the proof, we use two properties of IPM: the triangle inequality $d_\cF(\mu,\gamma) \le d_\cF(\mu,\tau) + d_\cF(\tau,\gamma)$ and, if $\cF$ is symmetric, then $d_\cF(\mu,\gamma) = d_\cF(\gamma,\mu)$. These properties can be easily derived using the definition.

\begin{lemma}\label{general error decomposition}
Assume $\mu$ and $g_\#\nu$ are supported on $\Omega\subseteq \bR^d$ for all $g\in \cG$. Suppose $\cF$ is a symmetric function class defined on $\Omega$. For any probability distribution $\widetilde{\mu}_n$ supported on $\Omega$, let $\widetilde{g}_n$ and $\widetilde{g}_{n,m}$ be the associated GAN estimators defined by
\begin{align*}
\widetilde{g}_n &\in \left\{g\in \cG: d_\cF(\widetilde{\mu}_n, g_\# \nu) \le \inf_{\phi\in \cG} d_\cF(\widetilde{\mu}_n, \phi_\# \nu) + \epsilon_{opt} \right\},\\
\widetilde{g}_{n,m} &\in \left\{g\in \cG: d_\cF(\widetilde{\mu}_n, g_\# \widehat{\nu}_m) \le \inf_{\phi\in \cG} d_\cF(\widetilde{\mu}_n, \phi_\# \widehat{\nu}_m) + \epsilon_{opt} \right\}.
\end{align*}
Then, for any function class $\cH$ defined on $\Omega$,
\begin{align*}
d_\cH(\mu,(\widetilde{g}_n)_\# \nu) &\le \epsilon_{opt} + 2\cE(\cH,\cF,\Omega)  + \inf_{g \in \cG} d_\cF(\widetilde{\mu}_n,g_\# \nu) + d_\cF(\mu,\widetilde{\mu}_n) \land d_\cH(\mu,\widetilde{\mu}_n), \\
d_\cH(\mu,(\widetilde{g}_{n,m})_\# \nu) &\le \epsilon_{opt} + 2\cE(\cH,\cF,\Omega)  + \inf_{g \in \cG} d_\cF(\widetilde{\mu}_n,g_\# \nu) + d_\cF(\mu,\widetilde{\mu}_n) \land d_\cH(\mu,\widetilde{\mu}_n) + 2d_{\cF \circ \cG}(\nu,\widehat{\nu}_m).
\end{align*}
\end{lemma}
\begin{proof}
By lemma \ref{IPM comparision} and the triangle inequality, for any $g\in \cG$,
\begin{align*}
d_\cH(\mu,g_\# \nu) &\le 2\cE(\cH,\cF,\Omega) + d_\cF(\mu, g_\# \nu) \\
&\le 2\cE(\cH,\cF,\Omega) + d_\cF(\mu,\widetilde{\mu}_n) + d_\cF(\widetilde{\mu}_n,g_\# \nu).
\end{align*}
Alternatively, we can apply the triangle inequality first and then use lemma \ref{IPM comparision}:
\begin{align*}
d_\cH(\mu,g_\# \nu) &\le d_\cH(\mu,\widetilde{\mu}_n) + d_\cH(\widetilde{\mu}_n,g_\# \nu) \\
&\le d_\cH(\mu,\widetilde{\mu}_n) + d_\cF(\widetilde{\mu}_n,g_\# \nu) + 2\cE(\cH,\cF,\Omega).
\end{align*}
Combining these two bounds, we have
\begin{equation}\label{error decomposition inequality}
d_\cH(\mu,g_\# \nu) \le 2\cE(\cH,\cF,\Omega)  + d_\cF(\widetilde{\mu}_n,g_\# \nu) + d_\cF(\mu,\widetilde{\mu}_n) \land d_\cH(\mu,\widetilde{\mu}_n).
\end{equation}
Letting $g=\widetilde{g}_n$ and observing that $d_\cF(\widetilde{\mu}_n,(\widetilde{g}_n)_\# \nu) \le \inf_{g \in \cG} d_\cF(\widetilde{\mu}_n,g_\# \nu) + \epsilon_{opt}$, we get the bound for $d_\cH(\mu,(\widetilde{g}_n)_\# \nu)$.

For $\widetilde{g}_{n,m}$, we only need to bound
$ d_\cF(\widetilde{\mu}_n,(\widetilde{g}_{n,m})_\# \nu)$. By the triangle inequality,
\[
d_\cF(\widetilde{\mu}_n,(\widetilde{g}_{n,m})_\# \nu) \le d_\cF(\widetilde{\mu}_n,(\widetilde{g}_{n,m})_\# \widehat{\nu}_m) + d_\cF((\widetilde{g}_{n,m})_\# \widehat{\nu}_m,(\widetilde{g}_{n,m})_\# \nu).
\]
By the definition of IPM, the last term can be bounded as
\[
d_\cF((\widetilde{g}_{n,m})_\# \widehat{\nu}_m,(\widetilde{g}_{n,m})_\# \nu) \le d_{\cF \circ \cG}(\widehat{\nu}_m,\nu).
\]
By the definition of $\widetilde{g}_{n,m}$ and the triangle inequality, we have, for any $g\in \cG$,
\begin{align*}
d_\cF(\widetilde{\mu}_n,(\widetilde{g}_{n,m})_\# \widehat{\nu}_m) -\epsilon_{opt} &\le d_\cF(\widetilde{\mu}_n,g_\# \widehat{\nu}_m) \le  d_\cF(\widetilde{\mu}_n,g_\# \nu) + d_\cF(g_\# \nu, g_\# \widehat{\nu}_m) \\ &\le d_\cF(\widetilde{\mu}_n,g_\# \nu) + d_{\cF \circ \cG}(\nu,\widehat{\nu}_m).
\end{align*}
Taking infimum over all $g\in \cG$, we have
\[
d_\cF(\widetilde{\mu}_n,(\widetilde{g}_{n,m})_\# \widehat{\nu}_m) \le \epsilon_{opt} + \inf_{g \in \cG} d_\cF(\widetilde{\mu}_n,g_\# \nu) + d_{\cF \circ \cG}(\nu,\widehat{\nu}_m).
\]
Therefore,
\[
d_\cF(\widetilde{\mu}_n,(\widetilde{g}_{n,m})_\# \nu) \le \epsilon_{opt} + \inf_{g \in \cG} d_\cF(\widetilde{\mu}_n,g_\# \nu) + 2d_{\cF \circ \cG}(\nu,\widehat{\nu}_m).
\]
Combining this with the inequality (\ref{error decomposition inequality}) with $g=\widetilde{g}_{n,m}$, we get the bound for $d_\cH(\mu,(\widetilde{g}_{n,m})_\# \nu)$.
\end{proof}

\chapter{Sample Complexity of Neural Networks}\label{chapter: sample comp}

The generalization error $d_\cF(\mu,\widehat{\mu}_n)$ is the difference between the expectation $\bE_\mu[f]$ and the empirical average $\bE_{\widehat{\mu}_n}[f]$ over functions $f$ in the class $\cF$. The statistical learning theory \cite{anthony2009neural,mohri2018foundations,shalevshwartz2014understanding} controls this error by certain complexities of the function class $\cF$. In this chapter, we introduce some of these complexities, which measure the richness of the function class in different aspects, and use them to bound the generalization error.

\section{Rademacher complexity}

The Rademacher complexity is widely used in the analysis of machine learning algorithms \cite{bartlett2002rademacher}. This complexity measures the correlation between a set of vectors and random noise. Given a sample dataset, we can quantifies the expressiveness of a function class by the (empirical) Rademacher complexity of the function values on the samples.

\begin{definition}[Rademacher complexity]\label{Rademacher complexity}
The Rademacher complexity of a set $S\subseteq \bR^n$ is defined by
\[
\cR_n(S) := \bE_{\boldsymbol{\xi}} \left[ \sup_{\Bs\in S} \frac{\boldsymbol{\xi} \cdot \Bs}{n}  \right] =\bE_{\boldsymbol{\xi}} \left[ \sup_{\Bs\in S} \frac{1}{n} \sum_{i=1}^n \xi_i s_i  \right],
\]
where $\boldsymbol{\xi} = (\xi_1,\dots,\xi_n)^\intercal$ is a Rademacher random vector, with $\xi_i$s independent random variables assuming values $+1$ and $-1$ with probability $1/2$ each.
\end{definition}

Let $f(X_{1:n}) := (f(X_1),\dots, f(X_n))^\intercal \in \bR^n$ be the vector of values taken by function $f$ over the sample $X_{1:n} = (X_i)_{i=1}^n$ and $\cF(X_{1:n}):=\{f(X_{1:n}): f\in \cF \} \subseteq \bR^n$ be the collection of these vectors. Then, the (empirical) Rademacher complexity $\cR_n(\cF(X_{1:n}))$ measures how well the function class $\cF$ correlates with random noise on the sample $X_{1:n}$. This describes the complexity of the function class $\cF$: more complex class $\cF$ can generate more vectors $f(X_{1:n})$ and hence better correlate with random noise on average. Using standard symmetrization argument \cite{mohri2018foundations,shalevshwartz2014understanding}, we can show that the generalization error $d_\cF(\mu,\widehat{\mu}_n)$ can be bounded by the Rademacher complexity of the class $\cF$ in expectation.

\begin{lemma}\label{symmetrization}
Let $X_{1:n} = (X_i)_{i=1}^n$ be i.i.d. samples from $\mu$ and $\widehat{\mu}_n = \frac{1}{n} \sum_{i=1}^{n} \delta_{X_i}$, then
\[
\bE_{X_{1:n}} [d_\cF(\mu,\widehat{\mu}_n)] \le 2 \bE_{X_{1:n}} [\cR_n(\cF(X_{1:n}))].
\]
\end{lemma}
\begin{proof}
We introduce a ghost dataset $X'_{1:n}=(X_i')_{i=1}^n$ drawn i.i.d. from $\mu$ and independent of $X_{1:n}$, then
\begin{align*}
\bE_{X_{1:n}} [d_\cF(\mu,\widehat{\mu}_n)] &= \bE_{X_{1:n}} \left[\sup_{f\in \cF} \bE_{\Bx\sim \mu}[f(\Bx)] - \frac{1}{n} \sum_{i=1}^n f(X_i) \right] \\
&= \bE_{X_{1:n}} \left[\sup_{f\in \cF} \bE_{X'_{1:n}}  \frac{1}{n} \sum_{i=1}^n f(X'_i) - \frac{1}{n} \sum_{i=1}^n f(X_i) \right] \\
&\le \bE_{X_{1:n},X'_{1:n}} \left[\sup_{f\in \cF} \frac{1}{n} \sum_{i=1}^n (f(X'_i) -  f(X_i)) \right].
\end{align*}
Let $\boldsymbol{\xi} = (\xi_1,\dots,\xi_n)^\intercal$ be a Rademacher random vector independent of $X_{1:n}$ and $X'_{1:n}$. Then, by symmetrization argument,
\begin{align*}
\bE_{X_{1:n}} [d_\cF(\mu,\widehat{\mu}_n)] &\le \bE_{X_{1:n},X'_{1:n}} \left[\sup_{f\in \cF} \frac{1}{n} \sum_{i=1}^n (f(X'_i) -  f(X_i)) \right] \\
&= \bE_{X_{1:n},X'_{1:n},\boldsymbol{\xi}} \left[\sup_{f\in \cF} \frac{1}{n} \sum_{i=1}^n \xi_i (f(X'_i) -  f(X_i)) \right] \\
&\le \bE_{X_{1:n},X'_{1:n},\boldsymbol{\xi}} \left[\sup_{f\in \cF} \frac{1}{n} \sum_{i=1}^n \xi_i f(X'_i) + \sup_{f\in \cF} \frac{1}{n} \sum_{i=1}^n -\xi_i f(X_i) \right] \\
&= 2 \bE_{X_{1:n},\boldsymbol{\xi}} \left[\sup_{f\in \cF} \frac{1}{n} \sum_{i=1}^n \xi_i f(X_i) \right] =  2 \bE_{X_{1:n}} [\cR_n(\cF(X_{1:n}))],
\end{align*}
where the second last equality is due to the fact that $X_i$ and $X'_i$ have the same distribution and the fact that $\xi_i$ and $-\xi_i$ have the same distribution.
\end{proof}

\begin{remark}
The bound on the expectation $\bE_{X_{1:n}} [d_\cF(\mu,\widehat{\mu}_n)]$ can be turned into a high
probability bound by using concentration inequalities, such as McDiarmid's inequality. See \cite{boucheron2013concentration,mohri2018foundations,shalevshwartz2014understanding} for more details.
\end{remark}

The sample complexity of learning norm constrained neural networks have been studied in the recent works \cite{neyshabur2015norm,neyshabur2018pac,bartlett2017spectrally,golowich2018size}. We state the Rademacher complexity bounds for the class $\cS\cN\cN(W,L,K)$ in the following lemma. By Proposition \ref{network class relation}, these bounds can also be applied to the class $\cN\cN(W,L,K)$.

\begin{lemma}\label{Rademacher bound}
For any $\Bx_1,\dots,\Bx_n \in [-B,B]^d$, let $S:= \{(\phi(\Bx_1),\dots,\phi(\Bx_n)) :\phi \in \cS\cN\cN_{d,1}(W,L,K) \} \subseteq \bR^n$, then
\[
\cR_n(S) \le \frac{2K}{n} \sqrt{L+2+\log(d+1)} \max_{1\le j\le d+1} \sqrt{ \sum_{i=1}^n x_{i,j}^2} \le \frac{2\max\{B,1\} K\sqrt{L+2+\log(d+1)}}{\sqrt{n}},
\]
where $x_{i,j}$ is the $j$-th coordinate of the vector $\widetilde{\Bx}_i=(\Bx_i^\intercal,1)^\intercal\in \bR^{d+1}$. When $W\ge 2$,
\[
\cR_n(S) \ge \frac{K}{2\sqrt{2}n} \max_{1\le j\le d+1} \sqrt{ \sum_{i=1}^n x_{i,j}^2} \ge \frac{K}{2\sqrt{2n}}.
\]
\end{lemma}
\begin{proof}
The upper bound is from \cite[Theorem 2]{golowich2018size}. For the lower bound, we consider the linear function class $\cF :=\{\Bx\mapsto \Ba^\intercal \widetilde{\Bx}: \Ba\in \bR^{d+1}, \|\Ba\|_1\le K/2 \}$. Observing that $\Ba^\intercal \widetilde{\Bx} = \sigma(\Ba^\intercal \widetilde{\Bx}) - \sigma(-\Ba^\intercal \widetilde{\Bx})$, we conclude that $\cF\subseteq \cS\cN\cN(2,1,K) \subseteq \cS\cN\cN(W,L,K)$. Therefore,
\begin{align*}
\cR_n(S) &= \frac{1}{n} \bE_{\boldsymbol{\xi}} \left[ \sup_{\phi\in \cS\cN\cN(W,L,K)} \sum_{i=1}^n \xi_i \phi(\Bx_i) \right] \ge \frac{1}{n} \bE_{\boldsymbol{\xi}} \left[ \sup_{\|\Ba\|_1\le K/2} \sum_{i=1}^n \xi_i \Ba^\intercal \widetilde{\Bx}_i \right] \\
& = \frac{K}{2n} \bE_{\boldsymbol{\xi}} \left\| \sum_{i=1}^n \xi_i \widetilde{\Bx}_i \right\|_\infty = \frac{K}{2n} \bE_{\boldsymbol{\xi}} \max_{1\le j\le d+1} \left| \sum_{i=1}^n \xi_i x_{i,j} \right| \\
&\ge \frac{K}{2n} \max_{1\le j\le d+1} \bE_{\boldsymbol{\xi}} \left| \sum_{i=1}^n \xi_i x_{i,j} \right| \ge \frac{K}{2\sqrt{2}n} \max_{1\le j\le d+1} \sqrt{ \sum_{i=1}^n x_{i,j}^2},
\end{align*}
where the last inequality is due to Khintchine inequality, see \cite[Lemma 4.1]{ledoux1991probability} and \cite{haagerup1981best}.
\end{proof}

\section{Covering number and Pseudo-dimension}

We have bounded the generalization error $d_\cF(\mu,\widehat{\mu}_n)$ by the Rademacher complexity of the function class $\cF$. However, the Rademacher complexity is difficult to compute in general. For classical function classes, such as H\"older functions, it is more convenient to describe their complexity by covering number \cite{kolmogorov1961}. For deep neural networks, one can estimate the generalization error through covering number bounds and obtain optimal learning rate for many machine learning tasks, such as nonparametric regression problem  \cite{schmidthieber2020nonparametric,nakada2020adaptive}.

\begin{definition}[Covering and Packing numbers]\label{covering and packing}
Let $\rho$ be a metric on $\cM$ and $S\subseteq \cM$. For $\epsilon>0$, a set $T \subseteq \cM$ is called an $\epsilon$-covering (or $\epsilon$-net) of $S$ if for any $x\in S$ there exits $y\in T$ such that $\rho(x,y)\le \epsilon$. A subset $U \subseteq S$ is called an $\epsilon$-packing of $S$ (or $\epsilon$-separated) if any two elements $x\neq y$ in $U$ satisfies $\rho(x,y)>\epsilon$. The $\epsilon$-covering and $\epsilon$-packing numbers of $S$ are denoted respectively by
\begin{align*}
\cN_c(S,\rho,\epsilon) &:= \min\{|T|: T \mbox{ is an $\epsilon$-covering of } S \}, \\
\cN_p(S,\rho,\epsilon) &:= \max\{|U|: U \mbox{ is an $\epsilon$-packing of } S \}.
\end{align*}
\end{definition}
It is not hard to check that $\cN_p(S,\rho,2\epsilon) \le \cN_c(S,\rho,\epsilon) \le \cN_p(S,\rho,\epsilon)$. Hence, we can use the covering number and packing number interchangeably. The following lemma bounds the Rademacher complexity of a set by its covering number. It is referred to as the chaining technique, which is attributed to Dudley \cite{ledoux1991probability,shalevshwartz2014understanding}.

\begin{lemma}[Chaining, {\cite[Lemma 27.4]{shalevshwartz2014understanding}}]\label{chaining}
Assume $\|\Bs \|_2\le c$ for any $\Bs\in S \subseteq \bR^n$. Then, for any integer $M\ge 0$, 
\[
\cR_n(S) \le \frac{c2^{-M}}{\sqrt{n}} + \frac{6c}{n} \sum_{m=1}^M 2^{-m} \sqrt{\log(\cN_c(S,\|\cdot\|_2, c2^{-m}))}.
\]
\end{lemma}

Using the chaining technique, we can bound the Rademacher complexity of a function class by Dudley's entropy integral. 

\begin{lemma}\label{entropy integral}
Let $\cF$ be a function class defined on $\Omega$ and $\Bx_1,\dots, \Bx_n \in \Omega$. If $\sup_{f\in \cF}\|f\|_{L^\infty(\Omega)} \le B$, then
\[
\cR_n(\cF(\Bx_{1:n})) \le 4 \inf_{0< \delta<B/2}\left( \delta + \frac{3}{\sqrt{n}} \int_{\delta}^{B/2} \sqrt{\log \cN_c(\cF(\Bx_{1:n}),\|\cdot\|_\infty,\epsilon)} d\epsilon \right),
\]
where we denote $\cF(\Bx_{1:n}) = \{(f(\Bx_1),\dots,f(\Bx_n)):f\in \cF \} \subseteq \bR^n$.
\end{lemma}
\begin{proof}
We define a distance on $\bR^n$ by 
\[
\rho_2(\Bx,\By) := \left(\frac{1}{n}\sum_{i=1}^n(x_i-y_i)^2 \right)^{1/2} = \frac{1}{\sqrt{n}} \|\Bx-\By\|_2.
\]
Then, one can check that $\cN_c(S,\|\cdot\|_2, \epsilon) = \cN_c(S,\rho_2, \epsilon/\sqrt{n})$ for any $S\subseteq \bR^n$. By Lemma \ref{chaining}, for any integer $M\ge 0$,
\begin{align*}
\cR_n(\cF(\Bx_{1:n})) &\le 2^{-M}B + \frac{6B}{\sqrt{n}} \sum_{m=1}^M 2^{-m} \sqrt{\log(\cN_c(\cF(\Bx_{1:n}),\|\cdot\|_2, \sqrt{n}2^{-m}B))} \\
&= 2^{-M}B + \frac{12}{\sqrt{n}} \sum_{m=1}^M 2^{-m-1}B \sqrt{\log(\cN_c(\cF(\Bx_{1:n}),\rho_2, 2^{-m}B))} \\
&\le 2^{-M}B + \frac{12}{\sqrt{n}} \int_{2^{-M-1}B}^{B/2} \sqrt{\log(\cN_c(\cF(\Bx_{1:n}),\rho_2, \epsilon))} d\epsilon,
\end{align*}
where we use the fact that the covering number is a decreasing function of $\epsilon$ in the last inequality. Now, for any $\delta \in (0,B/2)$, there exists an integer $M\ge 0$ such that $2^{-M-2}B\le \delta<2^{-M-1}B$. Therefore, we have
\begin{align*}
\cR_n(\cF(\Bx_{1:n})) \le \inf_{0< \delta<B/2}\left( 4\delta + \frac{12}{\sqrt{n}} \int_{\delta}^{B/2} \sqrt{\log(\cN_c(\cF(\Bx_{1:n}),\rho_2, \epsilon))} d\epsilon \right).
\end{align*}
Since $\rho_2(\Bx,\By) \le \|\Bx-\By\|_\infty$, we have $\cN_c(\cF(\Bx_{1:n}),\rho_2, \epsilon) \le \cN_c(\cF(\Bx_{1:n}),\|\cdot\|_\infty, \epsilon)$, which completes the proof.
\end{proof}

Another useful complexity in statistical learning theory is the Pseudo-dimension (or VC dimension introduced by Vapnik and Chervonenkis \cite{vapnik1971uniform}). We refer to \cite{anthony2009neural} for detail discussion on its application in neural network learning.

\begin{definition}[Pseudo-dimension]\label{Pdim}
Let $\cF$ be a class of real-valued functions defined on $\Omega$. The pseudo-dimension of $\cF$, denoted by $\Pdim(\cF)$, is the largest integer $N$ for which there exist points $x_1,\dots,x_N \in \Omega$ and constants $c_1,\dots,c_N\in \bR$ such that
\[
\left| \left\{ \sgn(f(x_1)-c_1),\dots,\sgn(f(x_N)-c_N): f\in \cF \right\} \right| =2^N.
\]
\end{definition}

\begin{lemma}\label{Rc bound by pdim}
Let $\cF$ be a function class defined on $\Omega$. If $\sup_{f\in \cF}\|f\|_{L^\infty(\Omega)} \le B$ and the pseudo-dimension of $\cF$ is $\Pdim(\cF)<\infty$, then for any $\Bx_1,\dots, \Bx_n \in \Omega$,
\[
\cR_n(\cF(\Bx_{1:n})) \le CB  \sqrt{\frac{ \Pdim(\cF) \log n}{n}}
\]
for some universal constant $C>0$.
\end{lemma}
\begin{proof}
If $n\ge \Pdim(\cF)$, we have the following bound from \cite[Theorem 12.2]{anthony2009neural},
\[
\cN_c(\cF(\Bx_{1:n}),\|\cdot\|_\infty,\epsilon) \le \left( \frac{2eBn}{\epsilon \Pdim(\cF)} \right)^{\Pdim(\cF)}.
\]
If $n< \Pdim(\cF)$, since $\cF(\Bx_{1:n}) \subseteq \{\By\in \bR^n:\|\By\|_\infty\le B \}$ can be covered by at most $\lceil\frac{2B}{\epsilon} \rceil^n$ balls with radius $\epsilon$ in $\|\cdot\|_\infty$ distance, we always have $\cN_c(\cF(\Bx_{1:n}),\|\cdot\|_\infty,\epsilon)\le \lceil\frac{2B}{\epsilon} \rceil^n$. In any cases,
\[
\log \cN_c(\cF(\Bx_{1:n}),\|\cdot\|_\infty,\epsilon) \le \Pdim(\cF) \log \frac{2eBn}{\epsilon}.
\]
As a consequence, by Lemma \ref{entropy integral},
\begin{align*}
\cR_n(\cF(\Bx_{1:n})) &\le \inf_{0< \delta<B/2}\left( 4\delta + \frac{12}{\sqrt{n}} \int_{\delta}^{B/2} \sqrt{ \Pdim(\cF) \log(2eBn/\epsilon) } d\epsilon \right) \\
&\le \inf_{0< \delta<B/2}\left( 4\delta + 6B \sqrt{\frac{ \Pdim(\cF) \log(2eBn/\delta)}{n}} \right) \\
&\le CB  \sqrt{\frac{ \Pdim(\cF) \log n}{n}}
\end{align*}
for some universal constant $C>0$.
\end{proof}

For ReLU neural networks, \cite{bartlett2019nearly} showed that the pseudo-dimension can be bounded as
\[
\Pdim(\cN\cN(W,L)) \lesssim UL \log U,
\]
where $U$ is the number of parameters and $U\asymp W^2L$ when $L\ge 2$ for fully connected network $\cN\cN(W,L)$. Combining this bound with Lemma \ref{symmetrization} and \ref{Rc bound by pdim}, we can bound the generalization error by the size of neural network:
\[
\bE[d_\cF(\mu,\widehat{\mu}_n)] \lesssim B \sqrt{\frac{ \Pdim(\cF) \log n}{n}} \lesssim B \sqrt{\frac{ UL \log U \log n}{n}}
\]
where $\cF$ is a neural network with $U$ parameters, depth $L$ and  uniformly bounded by $B$.

\chapter{Function Approximation by Neural Networks}\label{chapter: fun app}

In this chapter, we study the approximation of H\"older class $\cH^\alpha$ by deep neural networks. We first discuss how well neural networks interpolate given data in Section \ref{sec: interpolation}. The interpolation result is a building block of our construction of neural networks for approximating functions. It will also be useful when we consider the distribution approximation by generative networks in next chapter. In Section \ref{sec: app width and depth}, we characterize the approximation error of H\"older function $h\in \cH^\alpha$ by the width and depth of neural networks. Our approximation bounds and construction of neural networks are similar to \cite{lu2021deep}, but we also estimate the Lipschitz constant of the constructed neural network, which will be essential when we analyze the convergence rates of GANs in Chapter \ref{chapter: rate}. Section \ref{sec: app norm constraint} discusses the function approximation by norm constrained neural network $\cN\cN(W,L,K)$, which has a direct control on the Lipschitz constant. We obtain approximation bound for such networks in terms of the norm constraint $K$, when the network size is sufficiently large. In Section \ref{sec: app lower bound}, we derive approximation lower bounds by using the Pseudo-dimension and Rademacher complexity of neural networks.

\section{Linear interpolation by neural networks}\label{sec: interpolation}

Since the ReLU activation function is piecewise linear, any function $\phi\in \cN\cN(W,L)$ is continuous piecewise linear (CPwL). In one dimensional case, for any set of data points $(x_i,y_i)_{i=1}^N$ with $x_i<x_{i+1}$, there exits a CPwL function $f$ that satisfies $f(x_i)=y_i$ and $f$ is linear on the interval $(x_i,x_{i+1})$. The paper \cite{daubechies2021nonlinear} showed that such CPwL function $f$ can be implemented by a ReLU neural network if the network size is sufficiently large. 

More generally, for $d\in \bN$, we can consider the function class $\cS^d(x_0,\dots,x_{N+1})$, which is the set of all CPwL functions $f:\bR \to \bR^d$ that have breakpoints only at $x_0<x_1<\dots<x_N<x_{N+1}$ and are constant on $(-\infty,x_0]$ and $[x_{N+1},\infty)$. The following lemma generalizes the result of \cite{daubechies2021nonlinear} to high dimension.

\begin{lemma}\label{CPwL}
Suppose $W\ge 6d$, $L\ge 1$ and $N\le W\lfloor \frac{W}{6d} \rfloor L$. Then for any $x_0<x_1<\dots<x_N<x_{N+1}$, we have $\cS^d(x_0,\dots,x_{N+1}) \subseteq \cN\cN(W+d+1,2L)$.
\end{lemma}

This lemma shows that if $N \lesssim W^2L/d$, we have $\cS^d(x_0,\dots,x_{N+1}) \subseteq \cN\cN(W,L)$. We remark that the construction in this lemma is asymptotically optimal in the sense that if $\cS^d(x_0,\dots,x_{N+1}) \subseteq \cN\cN(W,L)$ for some $W,L\ge 2$, then the condition $N \lesssim W^2L/d$ is necessary. To see this, we consider the function $F(\theta) := (\phi_\theta(x_0),\dots, \phi_\theta(x_{N+1}))$, where $\phi_\theta\in\cN\cN(W,L)$ is a ReLU neural network with parameters $\theta$. Let $U$ be the number of parameters of the neural network $\cN\cN(W,L)$. By the assumption that $\cS^d(x_0,\dots,x_{N+1}) \subseteq \cN\cN(W,L)$, the function $F: \bR^{U}\to \bR^{d(N+2)}$ is surjective and hence the Hausdorff dimension of $F(\bR^{U})$ is $d(N+2)$. Since $F(\theta)$ is a piecewise multivariate polynomial of $\theta$, it is Lipschitz continuous on any bounded balls. It is well-known that Lipschitz maps do not increase Hausdorff dimension (see \cite[Theorem 2.8]{evans2018measure}). Since $F(\bR^{U})$ is a countable union of images of bounded balls, its Hausdorff dimension is at most $U$, which implies $d(N+2)\le U$. Because of $U=(L-1)W^2 +(L+d+1)W+d$, we have $N\lesssim W^2L/d$.

To prove Lemma \ref{CPwL}, we follow the construction in \cite[Lemma 3.3 and 3.4]{daubechies2021nonlinear}. It is easy to check that $\cS^d(x_0,\dots,x_{N+1})$ is a linear space. We denote by $\cS^d_0(x_0,\dots,x_{N+1})$ the $dN$-dimensional linear subspace of $\cS^d(x_0,\dots,x_{N+1})$ that contains all functions which vanish outside $[x_0,x_{N+1}]$. When $d=1$ and $N=qW$ for some integers $q$ and $W$, we can construct a basis of $S^1_0(x_0,\dots,x_{N+1})$ as follows: for $1\le m\le q$ and $1\le j\le W$, let $h_{m,j}$ be the hat function which vanishes outside $[x_{jq-m},x_{jq+1}]$, takes the value one at $x_{jq}$ and is linear on each of the intervals $[x_{jq-m},x_{jq}]$ and $[x_{jq},x_{jq+1}]$. The breakpoint $x_{jq}$ is called the principal breakpoint of $h_{m,j}$. We order these hat functions by their leftmost breakpoints $x_{n-1}$ and rename them as $h_n$, $n=1,\dots N$, that is $h_n=h_{m,j}$ where $n-1=jq-m$. It is easy to check that $h_n$'s are a basis for $S^1_0(x_0,\dots,x_{N+1})$. The following lemma is a modification of \cite[Lemma 3.3]{daubechies2021nonlinear}.

\begin{lemma}\label{CPwL two layer}
For any breakpoints $x_0<\dots<x_{N+1}$ with $N=qW$, $q=\lfloor \frac{W}{6d} \rfloor$, $W\ge 6d$, we have $\cS^d_0(x_0,\dots,x_{N+1})\subseteq \cN\cN_{1,d}(W,2)$.
\end{lemma}
\begin{proof}
For any function $f=(f_1,\dots,f_d)\in \cS^d_0(x_0,\dots,x_{N+1})$, each component can be written as $f_i = \sum_{n=1}^N c_{i,n}h_n$. For each $f_i$, we can decompose the indices as $\{1,\dots,N\} = \Lambda^i_+ \cup \Lambda^i_-$, where for each $n\in \Lambda^i_+$, we have $c_{i,n}\ge 0$ and for each $n\in \Lambda^i_-$, we have $c_{i,n}< 0$. We then divide each of $\Lambda^i_+$ and $\Lambda^i_-$ into at most $3q$ sets, which are denoted by $\Lambda^i_k$, $1\le k\le 6q$, such that if $n,n'\in \Lambda^i_k$, then the principal breakpoints $x_{jq}, x_{j'q}$ of $h_n, h_{n'}$ respectively, satisfy the separation property $|j-j'|\ge 3$. Then, we can write
\[
f_i = \sum_{k=1}^{6q} f_{i,k}, \qquad f_{i,k} := \sum_{n\in \Lambda^i_k} c_{i,n} h_n, 
\]
where we set $f_{i,k} =0$ if $\Lambda^i_k = \emptyset$. By construction, in the second summation, the $h_n$, $n\in \Lambda^i_k$, have disjoint supports and the $c_{i,n}$ have same sign.

Next, we show that each $f_{i,k}$ is of the form $\pm \sigma(g_{i,k}(x))$, where $g_{i,k}$ is some linear combination of the $\sigma(x-x_{jq})$. First consider the case that the coefficients $c_{i,n}$ in $\Lambda^i_k$ are all positive. Then, we can construct a CPwL function $g_{i,k}$ that takes the value $c_{i,n}$ for the principal breakpoints $x_{jq}$ of $h_n$ with $n\in \Lambda^i_k$ and takes negative values for other principal breakpoints such that it vanishes at the leftmost and rightmost breakpoints of all $h_n$ with $n\in \Lambda^i_k$. This is possible due to the separation property of $\Lambda^i_k$ (an explicit construction strategy can be found in the appendix of \cite{daubechies2021nonlinear}). By this construction, we have $f_{i,k}(x) = \sigma(g_{i,k}(x))$. A similar construction can be applied to the case that all coefficients $c_{i,n}$ in $\Lambda^i_k$ are negative and leads to $f_{i,k}(x) = -\sigma(g_{i,k}(x))$.

Finally, each $f_i = \sum_{k=1}^{6q} f_{i,k}$ can be computed by a network whose first layer has $W$ neurons that compute $\sigma(x-x_{jq})$, $1\le j\le W$, second layer has at most $6q$ neurons that compute $\sigma(g_{i,k}(x))$, and output layer weights are $\pm 1$ or $0$. Since the first layers of these networks are the same, we can stack their second layers and output layers in parallel to produce $f=(f_1,\dots,f_d)$, then the width of the stacked second layer is at most $6dq\le W$. Hence, $f\in \cN\cN(W,2)$.
\end{proof}

We can use Lemma \ref{CPwL two layer} as a building block to represent CPwL functions with large number of breakpoints and give a proof of Lemma \ref{CPwL}. 

\begin{proof}[Proof of Lemma \ref{CPwL}]
By applying a linear transform to the input and adding extra breakpoints if necessary, we can assume that $x_0=0$ and $x_{N+1}=1$, where $N=qWL$ with $q=\lfloor \frac{W}{6d} \rfloor$. For any $f=(f_1,\dots,f_d)\in \cS^d(x_0,\dots,x_{N+1})$, we denote $\By_n = (y_{n,1},\dots, y_{n,d}) := f(x_n)$, where $y_{n,i} = f_i(x_n)$. We define 
$$
g_{0,i}(x) := y_{0,i} + (y_{N+1,i}-y_{0,i})(\sigma(x) - \sigma(x-1)),
$$
then $g_{0,i}$ is linear on $(0,1)$ and $g_{0,i}(x)=f_i(x)$ on $(-\infty,0]\cup[1,\infty)$. Let $g_0=(g_{0,1},\dots,g_{0,d})$, then $f-g_0 \in \cS^d_0(x_0,\dots,x_{N+1})$. We can decompose $f-g_0 = \sum_{l=1}^L g_l$, where $g_l\in \cS^d_0(x_0,\dots,x_{N+1})$ is the CPwL function that agree with $f-g_0$ at the points $x_i$ with $i=(l-1)qW+1,\dots,lqW$ and takes the value zero at other breakpoints. Obviously, $g_l \in \cS^d_0(x_{(l-1)qW},\dots,x_{lqW+1})$ and hence $g_l\in \cN\cN(W,2)$ by Lemma \ref{CPwL two layer}.

Next, we construct a network with special architecture of width $W+d+1$ and depth $2L$ that computes $f$. We reserve the first top neuron on each hidden layer to copy the non-negative input $\sigma(x)$. And the last $d$ neurons are used to collect intermediate results and are allowed to be ReLU-free. Since each $g_l \in \cN\cN(W,2)$, we concatenate the $L$ networks that compute $g_l$, $l=1,\dots,L$, and thereby produce $f-g_0$. Observe that $g_0 \in \cN\cN(d+1,1)$, we can use the last $d$ neurons on the first two layers to compute $g_0(x)$. Therefore, $f$ can be produced using this special network. The whole network architecture is shown in figure \ref{special network architecture}.

\begin{figure}[htbp]
\centering
\includegraphics[width=\textwidth]{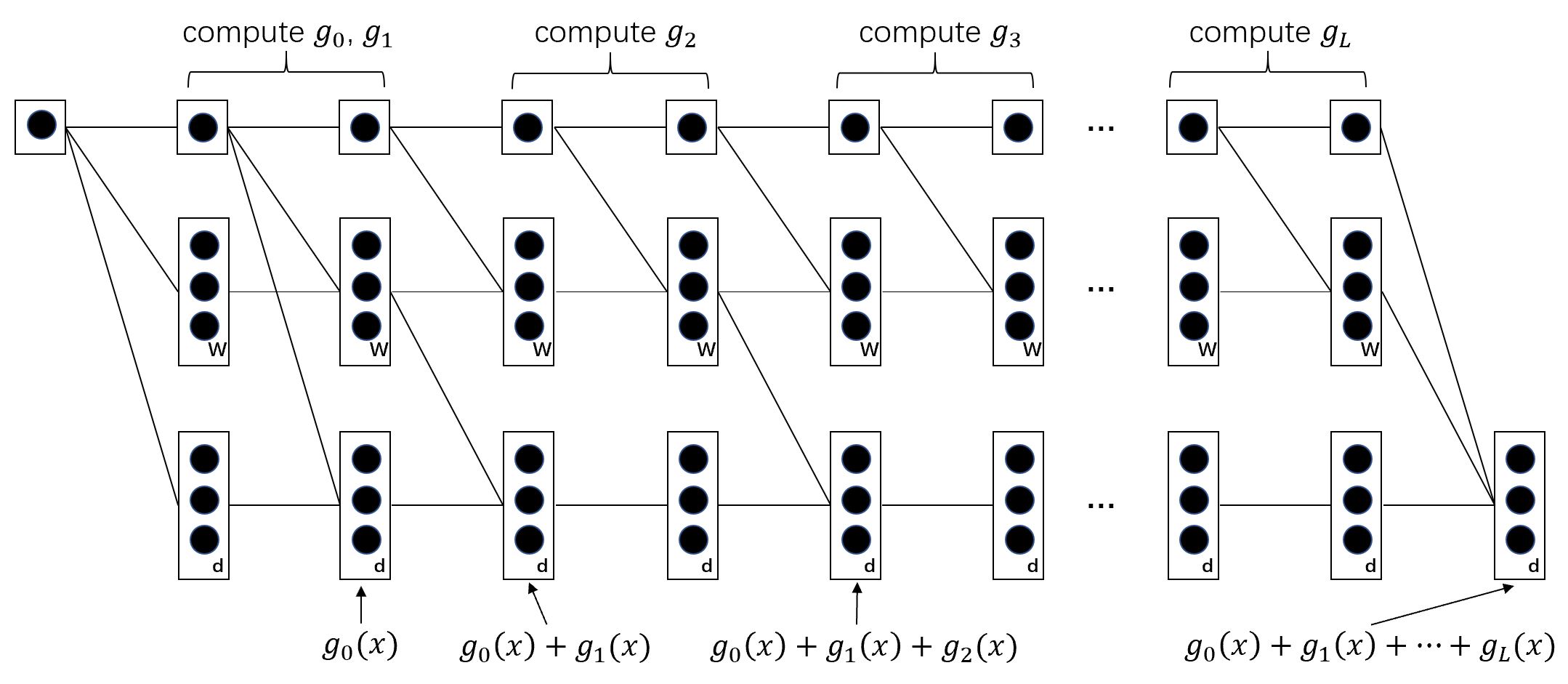}
\caption{The architecture of a neural network that produces $f= \sum_{l=0}^{L}g_l$. The letter on the lower right corner of each rectangle indicates the number of neurons in the rectangle.}
\label{special network architecture}
\end{figure}

Finally, suppose $S_l(x)$ is the output of the last $d$ neurons in layer $l$. Since $S_l(x)$ must be bounded, there exists a constant $C_l$ such that $S_l(x)+C_l>0$ and hence $S_l(x) = \sigma(S_l(x)+C_l)-C_l$. Thus, even though we allow the last $d$ neurons to be ReLU-free, the special network can also be implemented by a ReLU network with the same size. Consequently, $f\in \cN\cN(W+d+1,2L)$, which completes the proof.
\end{proof}

\section{Approximation bounds in terms of width and depth}\label{sec: app width and depth}

In this section, we construct neural networks to approximate a function $h\in \cH^\alpha$ with smoothness index $\alpha$. The main idea is to approximate the Taylor expansion of $h$ by neural networks. Using Taylor's Theorem with integral remainder, we have the following approximation bound for Taylor polynomial \cite[Lemma A.8]{petersen2018optimal}.

\begin{lemma}\label{Taylor Theorem}
Let $\alpha = r+\alpha_0>0$ with $r\in \bN_0 $ and $\alpha_0\in (0,1]$. For any $h\in \cH^\alpha(\bR^d)$ and $\Bx,\Bx_0\in [0,1]^d$, 
\[
\left| h(\Bx) - \sum_{\|\Bs\|_1 \le r} \frac{\partial^\Bs h(\Bx_0)}{\Bs !} (\Bx-\Bx_0)^\Bs \right| \le d^r \|\Bx-\Bx_0\|_\infty^\alpha.
\]
\end{lemma}

The approximation of the Taylor expansion can be divided into three parts:
\begin{itemize}[parsep=0pt]
\item Partition $[0,1]^d$ into small cubes $\cup_\Bm Q_\Bm$, and construct a network $\psi$ that approximately maps each $\Bx\in Q_\Bm$ to a fixed point $x_\Bm \in Q_\Bm$. Hence, $\psi$ approximately discretize $[0,1]^d$.
	
\item For any $\Bs$, construct a network $\phi_\Bs$ that approximates the Taylor coefficient $\Bx\in Q_\Bm \mapsto \partial^\Bs h(\Bx_\Bm)$. Once $[0,1]^d$ is discretized, this approximation is reduced to a data fitting problem.
	
\item Construct a network $P_\Bs(\Bx)$ to approximate the monomial $\Bx^\Bs$. In particular, we can construct a network $\phi_\times$ that approximates the product function.
\end{itemize}
Then our construction of neural network can be written in the form
\[
\phi(\Bx) = \sum_{\|\Bs\|_1 \le r} \phi_\times \left( \frac{\phi_\Bs(\Bx)}{\Bs !}, P_\Bs(\Bx-\psi(\Bx)) \right) .
\]
The main result is summarized in the following theorem. We collect the required preliminary results in next two subsections and give a proof in Section \ref{sec: proof app theorem width depth}.

\begin{theorem}\label{app theorem width depth}
Assume $h\in \cH^\alpha([0,1]^d)$ with $\alpha = r+\alpha_0$, $r\in \bN_0$ and $\alpha_0\in (0,1]$. For any $W\ge 6$, $L\ge 2$, there exists $\phi \in \cN\cN(49(r+1)^2 3^d d^{r+1}W \lceil\log_2 W\rceil, 15(r+1)^2 L \lceil \log_2 L\rceil+2d)$ such that $\|\phi\|_{L^\infty} \le 1$, $\Lip (\phi) \le (r+1) d^r L(WL)^{\sigma(4\alpha-4)/d} (1260d W^2L^2 2^{L^2}+ 19r 7^r)$ and
\[
\| \phi - h\|_{L^\infty([0,1]^d)} \le 6(r+1)^2 d^{r \lor 1} \lfloor (WL)^{2/d}\rfloor^{-\alpha}.
\]
\end{theorem}

This theorem implies that, for any $h\in \cH^\alpha([0,1]^d)$, there exists a neural network $\phi$ with width $\lesssim W\log W$ and depth $\lesssim L\log L$ such that $\phi \in \Lip(\bR^d,K,1)$ with Lipschitz constant $K \lesssim (WL)^{2+\sigma(4\alpha-4)/d} L 2^{L^2}$ and $\| \phi - h\|_{L^\infty([0,1]^d)} \lesssim (WL)^{-2\alpha/d}$. Hence, if we choose $W_2 \asymp W\log W$ and $L_2 \asymp L\log L$, then
\[
W \asymp W_2/\log W_2 =: \widetilde{W}_2, \quad L \asymp L_2/ \log L_2 =:\widetilde{L}_2,
\]
and $\phi \in \cN\cN(W_2,L_2) \cap \Lip(\bR^d, K,1)$ with
\[
K \lesssim (WL)^{2+\sigma(4\alpha-4)/d} L 2^{L^2} \lesssim (\widetilde{W}_2\widetilde{L}_2)^{2+\sigma(4\alpha-4)/d} \widetilde{L}_2 2^{\widetilde{L}_2^2}.
\]
And the approximation error is
\[
\| \phi - h\|_{L^\infty([0,1]^d)} \lesssim (WL)^{-2\alpha/d} \lesssim (W_2L_2 / (\log W_2 \log L_2))^{-2\alpha/d}.
\]
In particular, we have the following corollary. Recall that we have denoted the approximation error as
\[
\cE(\cH^\alpha, \cN\cN(W,L),[0,1]^d) := \sup_{f\in \cH^\alpha} \inf_{\phi \in \cN\cN(W,L)} \| f- \phi\|_{L^\infty ([0,1]^d)}.
\]

\begin{corollary}\label{app bound depth and width}
For any $d\in \bN$ and $\alpha>0$,
\[
\cE(\cH^\alpha,\cN\cN(W,L),[0,1]^d) \lesssim (WL / (\log W \log L))^{-2\alpha/d}.
\]
\end{corollary}

\subsection{Data fitting}

Given any $N+2$ samples $\{(x_i,y_i) \in \bR^2:i=0,1,\dots,N+1\}$ with $x_0<x_1<\cdots<x_N<x_{N+1}$, there exists a unique piece-wise linear function $\phi$ that satisfies the following three condition
\begin{enumerate}[parsep=0pt]
\item $\phi(x_i)=y_i$ for $i=0,1,\dots,N+1$.

\item $\phi$ is linear on each interval $[x_i,x_{i+1}]$, $i=0,1,\dots,N$

\item $\phi(x) = y_0$ for $x\in (-\infty,x_0)$ and $\phi(x) = y_{N+1}$ for $x\in (x_{N+1},\infty)$.
\end{enumerate}
We say $\phi$ is the linear interpolation of the given samples. Note that for any $x\in \bR$,
\[
\min_{0\le i\le N+1} y_i \le \phi(x) \le \max_{0\le i\le N+1} y_i, \quad \mbox{and} \quad \Lip (\phi) = \max_{0\le i\le N} \left| \frac{y_{i+1}-y_i}{x_{i+1}-x_i} \right|.
\]
Using the notation of Section \ref{sec: interpolation}, we have $\phi\in \cS^1(x_0,\dots,x_{N+1})$. As a special case of Lemma \ref{CPwL}, the next lemma estimates the required size of network to interpolate the given samples.

\begin{lemma}\label{linear interpolation}
For any $W \ge 6$, $L\in \bN$ and any samples $\{(x_i,y_i)\in \bR^2:i=0,1,\dots,N+1\}$ with $x_0<x_1<\cdots<x_N<x_{N+1}$, where $N\le \lfloor W/6 \rfloor WL$, the linear interpolation of these samples $\phi \in \cN\cN(W+2,2L)$.
\end{lemma}

As an application of Lemma \ref{linear interpolation}, we show how to use a ReLU neural network to approximately discretize the input space $[0,1]^d$.

\begin{proposition}\label{partition map}
For any integers $W\ge 6$, $L\ge 2$, $d\ge 1$ and $0<\delta \le \frac{1}{3M}$ with $M=\lfloor (WL)^{2/d}\rfloor$, there exists a network $\phi\in \cN\cN_{1,1}(4W+3,4L)$ such that $\phi(x)\in [0,1]$ for all $x\in \bR$, $\Lip (\phi) \le \frac{2L}{M^2\delta^2}$ and
\[
\phi(x) = \tfrac{m}{M}, \quad \mbox{if } x\in \left[\tfrac{m}{M}, \tfrac{m+1}{M}-\delta\cdot 1_{\{m<M-1\}} \right], m=0,1,\dots,M-1.
\]
\end{proposition}
\begin{proof}
The proof is divided into two cases: $d=1$ and $d\ge 2$.

\textbf{Case 1}: $d=1$. We have $M= W^2L^2$ and denote $N=W^2L$. Then we consider the sample set
\[
\left\{ \left(\tfrac{n}{N},n\right): n=0,1,\dots,N-1 \right\} \cup \left\{\left(\tfrac{n+1}{N}-\delta, n\right):n=0,1,\dots,N-2 \right\} \cup \left\{(1,N-1)\right\}.
\]
Its cardinality is $2N=2W^2L\le \lfloor 4W/6 \rfloor (4W)L+2$. By Lemma \ref{linear interpolation}, the linear interpolation of these samples $\phi_1\in \cN\cN(4W+2,2L)$. In particular, $\phi_1(x)\in [0,N-1]$ for all $x\in \bR$, $\Lip (\phi_1)=1/\delta$ and
\[
\phi_1(x) = n, \quad \mbox{if } x\in \left[\tfrac{n}{N}, \tfrac{n+1}{N}-\delta\cdot 1_{\{n<N-1\}} \right], n=0,1,\dots,N-1.
\]

Next, we consider the sample set
\[
\left\{ \left(\tfrac{\ell}{NL},\ell\right): \ell=0,1,\dots,L-1 \right\} \cup \left\{\left(\tfrac{\ell+1}{NL}-\delta, \ell\right):\ell=0,1,\dots,L-2 \right\} \cup \left\{\left(\tfrac{1}{N},L-1\right)\right\}.
\]
Its cardinality is $2L$. By Lemma \ref{linear interpolation}, the linear interpolation of these samples $\phi_2\in \cN\cN(8,2L)$. In particular, $\phi_2(x) \in [0,L-1]$ for all $x\in \bR$, $\Lip(\phi_2)=1/\delta$ and for $n=0,1,\dots,N-1$, $\ell=0,1,\dots,L-1$, we have
\[
\phi_2 \left(x-\tfrac{1}{N} \phi_1(x) \right) = \phi_2 \left(x-\tfrac{n}{N} \right) = \ell, \quad \mbox{if } x\in \left[\tfrac{nL+\ell}{NL}, \tfrac{nL+\ell+1}{NL}-\delta\cdot 1_{\{nL+\ell<NL-1\}} \right].
\]

Define $\phi(x):= \frac{1}{N} \phi_1(x) + \frac{1}{NL} \phi_2 \left(\sigma(x)-\frac{1}{N} \phi_1(x) \right) \in [0,1]$. Then, by Proposition \ref{basic construct}, it is easy to see that $\phi \in \cN\cN(4W+3,4L)$. For each $x\in \left[\frac{m}{M}, \frac{m+1}{M}-\delta\cdot 1_{\{m<M-1\}} \right]$ with $m\in \{0,1,\dots,M-1\}=\{0,1,\dots,NL-1\}$, there exists a unique representation $m=nL+\ell$ for $n\in \{0,1,\dots,N-1\}$, $\ell\in \{0,1,\dots,L-1\}$, and we have
\[
\phi(x)= \tfrac{1}{N} \phi_1(x) + \tfrac{1}{NL} \phi_2 \left(\sigma(x)-\tfrac{1}{N} \phi_1(x) \right) = \tfrac{n}{N} + \tfrac{\ell}{NL} = \tfrac{m}{M}.
\]
Observing that the Lipschitz constant of the function $x\mapsto \sigma(x)-\frac{1}{N}\phi_1(x)$ is $\frac{1}{N \delta}$, the Lipschitz constant of $\phi$ is at most $\frac{1}{N} \frac{1}{\delta} +\frac{1}{NL} \frac{1}{\delta} \frac{1}{N\delta} \le \frac{2L}{M^2\delta^2}$.

\textbf{Case 2}: $d\ge 2$. We consider the sample set
\[
\left\{ \left(\tfrac{m}{M}, \tfrac{m}{M}\right): m=0,1,\dots,M-1 \right\} \cup \left\{\left(\tfrac{m+1}{M}-\delta, \tfrac{m}{M}\right): m=0,1,\dots,M-1 \right\} \cup \left\{\left(1,\tfrac{M-1}{M}\right) \right\}.
\]
Its cardinality is $2M \le 2W^{2/d}L^{2/d}\le \lfloor 4W/6 \rfloor (4W)L+2$. By Lemma \ref{linear interpolation}, the linear interpolation of these samples $\phi\in \cN\cN(4W+2,2L)$. In particular, $\phi(x)\in [0,1]$ for all $x\in \bR$,
\[
\phi(x) = \tfrac{m}{M}, \quad \mbox{if } x\in \left[\tfrac{m}{M}, \tfrac{m+1}{M}-\delta\cdot 1_{\{m<M-1\}} \right], m=0,1,\dots,M-1,
\]
and the Lipschitz constant of $\phi$ is $\frac{1}{M\delta}\le \frac{2L}{M^2\delta^2}$.
\end{proof}

Lemma \ref{linear interpolation} shows that a network $\cN\cN(W,L)$ can exactly fit $N\asymp W^2L$ samples. We are going to show that it can approximately fit $N\asymp (W/\log W)^2(L/\log L)^2$ samples. The construction is based on the bit extraction technique \cite{bartlett1998almost,bartlett2019nearly}. The following lemma shows how to extract a specific bit using ReLU neural networks. For convenient, we denote the binary representation as
\[
\Bin 0.x_1x_2\dots x_L := \sum_{j=1}^L x_j 2^{-j} \in[0,1],
\]
where $x_j \in \{0,1\}$ for all $j=1,2,\dots,L$.

\begin{lemma}\label{bit extraction}
For any $L\in\bN$, there exists $\phi \in \cN\cN_{2,1}(8,2L)$ such that $\phi(x,\ell) = x_\ell$ for any $x = \Bin 0.x_1 x_2\dots x_L$ with $x_\ell\in \{0,1\}$, $\ell=1,2,\dots,L$. Furthermore, $|\phi(x,\ell)-\phi(x',\ell')|\le 2 \cdot 2^{L^2} |x-x'|+ L|\ell-\ell'|$ for any $x,x',\ell,\ell'\in \bR$.
\end{lemma}
\begin{proof}
For any $x = \Bin 0.x_1 x_2\dots x_L$, we define $\xi_j := \Bin 0.x_j x_{j+1}\dots x_L$ for $j=1,2,\dots,L$. Then $\xi_1 = x$ and $\xi_{j+1} = 2\xi_j-x_j = \sigma(2 \sigma(\xi_j) - x_j)$ for $j=1,2,\dots,L-1$. Let
\[
T(x):= \sigma(2^L x-2^{L-1}+1) - \sigma(2^L x-2^{L-1}) =
\begin{cases}
0  &x\le 1/2-2^{-L}, \\
\mbox{linear} & 1/2-2^{-L} <x <1/2, \\
1 & x\ge 1/2.
\end{cases}
\]
It is easy to check that $x_j = T(\xi_j)$ and hence $\xi_{j+1} = \sigma(2 \sigma(\xi_j) - T(\xi_j))$.

Denote $\delta_i = 1$ if $i=0$ and $\delta_i =0$ if $i\neq 0$ is an integer. Observing that
\[
\delta_i = \sigma(i+1) + \sigma(i-1) -2\sigma(i),
\]
and $t_1t_2 = \sigma(t_1+t_2-1)$ for any $t_1,t_2\in \{0,1\}$, we have
\begin{equation}\label{x_l expression}
x_\ell = \sum_{i=1}^L \delta_{\ell-i} x_i = \sum_{i=1}^L \sigma\left( \sigma(\ell-i+1) + \sigma(\ell-i-1) -2\sigma(\ell-i) +x_i -1 \right).
\end{equation}
If we denote the partial sum $s_{\ell,j} := \sum_{i=1}^j \sigma( \sigma(\ell-i+1) + \sigma(\ell-i-1) -2\sigma(\ell-i) +x_i -1 )$, then $x_\ell=s_{\ell,L}$.

For any $t_1,t_2,t_3 \in \bR$, we define a function $\psi(t_1,t_2,t_3)=(y_1,y_2,y_3) \in \bR^3$ by
\begin{align*}
y_1 &:= \sigma(2 \sigma (t_1) - T(t_1) ), \\
y_2 &:= \sigma(t_2) + \sigma(\sigma(t_3)+\sigma(t_3-2)-2\sigma(t_3-1)+T(t_1)-1), \\
y_3 &:= \max\{t_3-1,-L\} = \sigma(t_3-1+L) -L.
\end{align*}
Then, it is easy to check that $\psi \in \cN\cN_{3,3}(8,2)$. Using the expressions (\ref{x_l expression}) we have derived for $x_l$, one has
\[
\psi(\xi_j,s_{\ell,j-1},\ell-j+1) = (\xi_{j+1},s_{\ell,j},\ell-j), \quad \ell,j=1,\dots,L,
\]
where $s_{\ell,0}:=0$ and $\xi_{L+1}:=0$. Hence, by composing $\psi$ $L$ times, we can construct a network $\phi = \psi \circ \cdots \circ \psi \in \cN\cN(8,2L)$ such that $\phi(x,l) = \psi \circ \cdots \circ \psi(x,0,\ell) =s_{\ell,L} = x_\ell$ for $\ell=1,2,\dots,L$, where we drop the first and the third outputs of $\psi$ in the last layer.

It remains to estimate the Lipschitz constant. For any $t_1,t_2,t_3,t_1',t_2',t_3' \in \bR$, suppose $(y_1,y_2,y_3)= \psi(t_1,t_2,t_3)$ and $(y_1',y_2',y_3')= \psi(t_1',t_2',t_3')$. Then $|y_1-y_1'| \le 2^L |t_1-t_1'|$, $|y_3-y_3'|\le |t_3-t_3'|$ and $|y_2-y_2'|\le |t_2-t_2'|+2^L|t_1-t_1'| + |t_3-t_3'|$. Therefore, by induction,
\begin{align*}
|\phi(x,\ell)-\phi(x',\ell')| &\le (2^L + 2^{2L} + 2^{3L} + \cdots + 2^{L\cdot L})|x-x'| + L|\ell-\ell'| \\
&\le 2 \cdot 2^{L^2} |x-x'| + L|\ell-\ell'|,
\end{align*}
for any $x,x',\ell,\ell'\in \bR$.
\end{proof}

Using the bit extraction technique, the next lemma shows a network $\cN\cN(W,L)$ can exactly fit $N\asymp W^2L^2$ binary samples.

\begin{lemma}\label{binary fitting}
Given any $W \ge 6$, $L\ge 2$ and any $\xi_i \in \{0,1\}$ for $i=0,1,\dots,W^2L^2-1$, there exists $\phi\in \cN\cN_{1,1}(8W+4,4L)$ such that $\phi(i)=\xi_i$ for $i=0,1,\dots,W^2L^2-1$ and
$\Lip (\phi) \le 2 \cdot 2^{L^2} + L^2$.
\end{lemma}
\begin{proof}
Denote $M=W^2L$, then, for each  $i=0,1,\dots,W^2L^2-1$, there exists a unique representation $i=mL+\ell$ with $m=0,1,\dots,M-1$ and $\ell=0,1,\dots,L-1$. So we define $b_{m,\ell}:= \xi_i$, where $i=mL+\ell$. We further set $y_m := \Bin 0.b_{m,0} b_{m,1}\dots b_{m,L-1} \in [0,1]$ and $y_M=1$. By Lemma \ref{bit extraction}, there exists $\psi\in \cN\cN_{2,1}(8,2L)$ such that $\psi(y_m,\ell+1) = b_{m,\ell}$ for any $m=0,1,\dots,M-1$, and $\ell=0,1,\dots,L-1$.

We consider the sample set
\[
\{(mL,y_m): m=0,1,\dots,M \} \cup \{(mL-1,y_{m-1}):m=1,\dots,M \}.
\]
Its cardinality is $2M+1=2W^2 L+1\le \lfloor 4W/6 \rfloor (4W)L+2$. By Lemma \ref{linear interpolation}, the linear interpolation of these samples $\phi_1\in \cN\cN(4W+2,2L)$. In particular, $\Lip (\phi_1) \le 1$ and $\phi_1(i) =y_m$, when $i=mL+\ell$, for $m=0,1,\dots,M-1$, and $\ell=0,1,\dots,L-1$.

Similarly, for the sample set
\[
\{(mL,0): m=0,1,\dots,M \} \cup \{(mL-1,L-1):m=1,\dots,M \},
\]
the linear interpolation of these samples $\phi_2\in \cN\cN(4W+2,2L)$. In particular, $\Lip (\phi_2) =L-1$ and $\phi_2(i) =\ell$, when $i=mL+\ell$, for $m=0,1,\dots,M-1$, and $\ell=0,1,\dots,L-1$.

We define $\phi(x):= \psi(\phi_1(x),\phi_2(x)+1)$, then $\phi\in \cN\cN(8W+4,4L)$ and
\[
\phi(i)= \psi(\phi_1(i),\phi_2(i)+1)= \psi(y_m,l+1) = b_{m,\ell} = \xi_i
\]
for $i=mL+\ell$ with $m=0,1,\dots,M-1$, and $\ell=0,1,\dots,L-1$. By Lemma \ref{bit extraction}, we have
\[
|\phi(x)-\phi(x')|\le 2 \cdot 2^{L^2} |\phi_1(x)-\phi_1(x')|+ L|\phi_2(x)-\phi_2(x')| \le (2 \cdot 2^{L^2} + L^2)|x-x'|
\]
for any $x,x'\in \bR$.
\end{proof}

As an application of Lemma \ref{binary fitting}, we show that a network $\cN\cN(W,L)$ can approximately fit $N\asymp (W/\log W)^2(L/\log L)^2$ samples.

\begin{proposition}\label{data fitting}
For any $W \ge 6$, $L\ge 2$, $r\in \bN$ and any $\xi_i \in [0,1]$ for $i=0,1,\dots,W^2L^2-1$, there exists $\phi\in \cN\cN_{1,1}(8r(2W+1) \lceil \log_2 (2W)\rceil +2,4L \lceil \log_2(2L)\rceil+1)$ such that $\Lip (\phi) \le 4 \cdot 2^{L^2} + 2L^2$, $|\phi(i)-\xi_i|\le (WL)^{-2r}$ for $i=0,1,\dots,W^2L^2-1$ and $\phi(t)\in [0,1]$ for all $t\in \bR$.
\end{proposition}
\begin{proof}
Denote $J= \lceil 2r\log_2(WL) \rceil$. For each $\xi_i\in [0,1]$, there exist $b_{i,1}, b_{i,2},\dots,b_{i,J} \in \{0,1\}$ such that
\[
|\xi_i - \Bin 0.b_{i,1}b_{i,2} \dots b_{i,J}| \le 2^{-J}.
\]
By Lemma \ref{binary fitting}, there exist $\phi_1,\phi_2,\dots,\phi_J \in \cN\cN(8W+4,4L)$ such that $\Lip \phi_j \le 2 \cdot 2^{L^2} + L^2$ and $\phi_j(i) =b_{i,j}$ for $i=0,1,\dots,W^2L^2-1$ and $j=1,2,\dots,J$. We define
\[
\widetilde{\phi}(t):= \sum_{j=1}^J 2^{-j} \phi_j(t), \quad t\in \bR.
\]
Then, for $i=0,1,\dots,W^2L^2-1$,
\[
\left|\widetilde{\phi}(i) - \xi_i \right| = \left| \sum_{j=1}^J 2^{-j}b_{i,j} - \xi_i \right|= | \Bin 0.b_{i,1}b_{i,2} \dots b_{i,J} - \xi_i| \le 2^{-J} \le (WL)^{-2r}.
\]
Since $J\le 1+ 2r\log_2(WL) < 2(1+r\log_2 W)(1+\log_2 L) \le 2r\log_2(2W)\log_2(2L)$, $\widetilde{\phi}$ can be implemented by a network with width $8r(2W+1)\lceil \log_2 (2W)\rceil+2$ and depth $4L \lceil \log_2(2L) \rceil$, where we use two neurons in each hidden layer to remember the input and intermediate summation. Furthermore, for any $t,t'\in \bR$,
\[
\left|\widetilde{\phi}(t) - \widetilde{\phi}(t') \right| \le \sum_{j=1}^J 2^{-j} \Lip (\phi_j)|t-t'|\le (4 \cdot 2^{L^2} + 2L^2)|t-t'|.
\]
Finally, we define
\[
\phi(t) := \min\{ \max\{\widetilde{\phi}(t),0\},1 \} = \sigma(\widetilde{\phi}(t)) - \sigma(\widetilde{\phi}(t)-1) \in [0,1].
\]
Then $\phi\in \cN\cN(8r(2W+1)\lceil \log (2W)\rceil,4L \lceil \log(2L)\rceil+1)$, $\Lip (\phi) \le \Lip (\widetilde{\phi})$ and $\phi(i)=\widetilde{\phi}(i)$ for $i=0,1,\dots,W^2L^2-1$.
\end{proof}

\subsection{Approximation of polynomials}

The approximation of polynomials by ReLU neural networks is well-known \cite{yarotsky2017error,lu2021deep}. In the next lemma, we construct a neural network to approximate the product function and give explicit estimates of the approximation error and the Lipschitz continuity of the constructed network.

\begin{lemma}\label{product app}
For any $W,L\in \bN$, there exists $\phi \in \cN\cN_{2,1}(9W+1,L)$ such that for any $x,x',y,y'\in [-1,1]$,
\begin{align*}
|xy - \phi(x,y)| &\le 6W^{-L}, \\
|\phi(x,y) - \phi(x',y')| &\le 7|x-x'| + 7|y-y'|.
\end{align*}
\end{lemma}
\begin{proof}
We follow the construction in \cite{lu2021deep}. We first construct a neural network $\psi$ that approximates the function $f(x)=x^2$ on $[0,1]$. Define a set of teeth functions $T_i$ by
\[
T_1(x) :=
\begin{cases}
2x, \quad & 0\le x\le \tfrac{1}{2}, \\
2(1-x), \quad & \tfrac{1}{2}< x\le 1, \\
0, \quad & \mbox{else},
\end{cases}
\]
and $T_i(x) := T_{i-1}(T_1(x))$ for $x\in [0,1]$ and $i=2,3,\cdots$. It is easy to check that $T_i$ has $2^{i-1}$ teeth, see Figure \ref{plot_of_T} for more details. We note that $T_i$ can be implemented by a one-hidden-layer ReLU network with width $2^i$.

\begin{figure}[htbp]
\centering
\includegraphics[width=\textwidth]{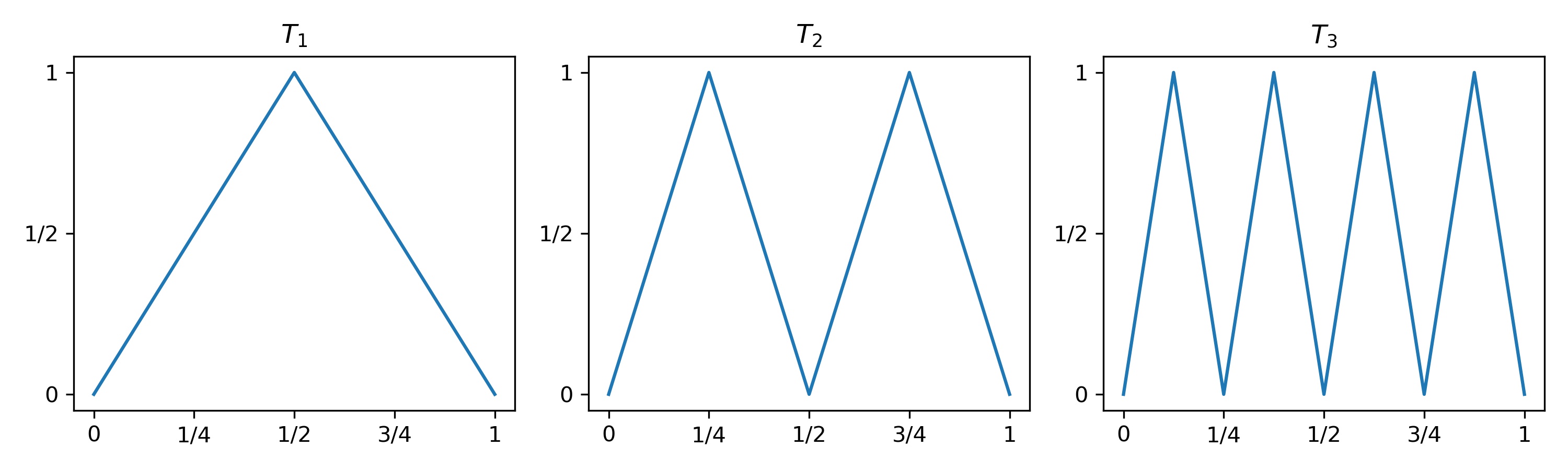}
\caption{Illustrations of teeth functions $T_1$, $T_2$ and $T_3$.}
\label{plot_of_T}
\end{figure}

Let $f_k :[0,1] \to [0,1]$ be the piece-wise linear function such that $f_k(\tfrac{j}{2^k}) = \left( \tfrac{j}{2^k}\right)^2$ for $j=0,1,\dots,2^k$, and $f_k$ is linear on $[\frac{j-1}{2^k},\frac{j}{2^k}]$ for $j=1,2,\dots,2^k$. Then, using the fact $\frac{(x-h)^2+(x+h)^2}{2} - x^2 = h^2$, one can show that
\[
|x^2 - f_k(x)| \le 2^{-2(k+1)}, \quad x\in [0,1], k\in \bN.
\]
Furthermore, $f_{k-1}(x) - f_k(x) = \frac{T_k(x)}{2^{2k}}$ and $x-f_1(x)=\frac{T_1(x)}{4}$. Hence,
\[
f_k(x) = x - (x-f_1(x)) - \sum_{i=2}^k (f_{i-1}(x) - f_i(x)) = x- \sum_{i=1}^k \frac{T_i(x)}{2^{2i}}, \quad x\in[0,1], k\in \bN.
\]

Given $W\in \bN$, there exists a unique $n\in \bN$ such that $(n-1)2^{n-1}+1 \le W\le n2^n$. For any $L\in \bN$, it was shown in \cite[Lemma 5.1]{lu2021deep} that $f_{nL}$ can be implemented by a network $\psi$ with width $3W$ and depth $L$. Hence,
\[
|x^2 -\psi(x)| \le |x^2 - f_{nL}(x)| \le 2^{-2(nL+1)} = 2^{-2nL}/4 \le W^{-L}/4, \quad x\in [0,1],
\]
where we use $W\le n2^n\le 2^{2n}$ in the last inequality.

Using the fact that
\[
xy = 2\left( \left(\tfrac{x+y}{2} \right)^2 - \left( \tfrac{x}{2} \right)^2 - \left( \tfrac{y}{2} \right)^2 \right), \quad x,y\in \bR,
\]
we can approximate the function $f(x,y) = xy$ by
\[
\phi_0(x,y) := 2\left( \psi \left(\tfrac{x+y}{2} \right) - \psi\left( \tfrac{x}{2} \right) - \psi\left( \tfrac{y}{2} \right) \right).
\]
Then, $\phi_0 \in \cN\cN(9W,L)$ and for $x,y\in [0,1]$,
\[
|xy - \phi_0(x,y)| \le 2 \left| \left(\tfrac{x+y}{2} \right)^2 - \psi \left(\tfrac{x+y}{2} \right) \right| + 2 \left| \left(\tfrac{x}{2} \right)^2 - \psi \left(\tfrac{x}{2} \right) \right| + 2 \left| \left(\tfrac{y}{2} \right)^2 - \psi \left(\tfrac{y}{2} \right) \right| \le \tfrac{3}{2}W^{-L}
\]
Furthermore, for any $x,x',y,y'\in [0,1]$,
\begin{align*}
|\phi_0(x,y) - \phi_0(x',y')|
\le& 2 \left|f_{nL} \left(\tfrac{x+y}{2} \right) - f_{nL} \left(\tfrac{x'+y'}{2} \right) \right| + 2 \left|f_{nL} \left(\tfrac{x}{2} \right) - f_{nL} \left(\tfrac{x'}{2} \right) \right| + 2 \left| f_{nL} \left(\tfrac{y}{2} \right) - f_{nL} \left(\tfrac{y'}{2} \right) \right| \\
\le & 4 \left| \tfrac{x+y}{2} - \tfrac{x'+y'}{2} \right| + 2 \left| \tfrac{x}{2} - \tfrac{x'}{2} \right| + 2\left| \tfrac{y}{2} - \tfrac{y'}{2} \right| \\
\le & 3|x-x'| + 3|y-y'|,
\end{align*}
where we use $|f_{nL}(t) - f_{nL}(t')|\le 2|t-t'|$ for $t,t'\in[0,1]$ and $|f_{nL}(t) - f_{nL}(t')|\le |t-t'|$ for $t,t'\in[0,1/2]$.

For any $x,y\in [-1,1]$, set $x_0=(x+1)/2 \in [0,1]$ and $y_0=(y+1)/2\in [0,1]$, then $xy=4x_0y_0-x-y-1$. Using this fact, we define the target function by
\[
\phi(x,y) = 4\phi_0(\tfrac{x+1}{2},\tfrac{y+1}{2}) - \sigma(x+y+2) +1.
\]
Then, $\phi \in \cN\cN(9W+1,L)$ and for $x,y\in [-1,1]$,
\[
|xy - \phi(x,y)|\le 4 |\tfrac{x+1}{2} \tfrac{y+1}{2} - \phi_0(\tfrac{x+1}{2},\tfrac{y+1}{2})|\le 6 W^{-L}.
\]
Furthermore, for any $x,x',y,y'\in [-1,1]$,
\begin{align*}
|\phi(x,y) - \phi(x',y')| \le & 4|\phi_0(\tfrac{x+1}{2},\tfrac{y+1}{2}) - \phi_0(\tfrac{x'+1}{2},\tfrac{y'+1}{2})| + |x+y-x'-y'| \\
\le& 7|x-x'| + 7|y-y'|,
\end{align*}
which completes the proof.
\end{proof}

By applying the approximation of the product function, we can approximate any monomials by neural networks.

\begin{corollary}\label{poly app}
Let $P(\Bx) = \Bx^\Bs$ for $\Bx\in \bR^d$ and $\Bs=(s_1,s_2,\dots,s_d)\in \bN_0^d$ with $\|\Bs\|_1 =k\ge 2$. For any $W,L\in \bN$, there exists $\phi \in \cN\cN(9W+k-1,(k-1)(L+1))$ such that for any $x,y\in [-1,1]^d$, $\phi(x) \in [-1,1]$ and
\begin{align*}
|\phi(\Bx) -P(\Bx) |&\le 6(k-1)W^{-L}, \\
|\phi(\Bx) - \phi(\By)| &\le k7^{k-1} \|\Bx-\By\|_\infty.
\end{align*}
\end{corollary}
\begin{proof}
For any $\Bx=(x_1,x_2,\dots,x_d) \in \bR^d$, let $\Bz=(z_1,z_2,\dots,z_k)\in \bR^k$ be the vector such that $z_i = x_j$ if $\sum_{\ell=1}^{j-1} s_\ell < i \le \sum_{\ell=1}^j s_\ell$ for $j=1,2,\dots,d$. Then $P(\Bx)=\Bx^\Bs = z_1z_2 \cdots z_k$ and there exists a linear map $\phi_0:\bR^d \to \bR^k$ such that $\phi_0(\Bx) =\Bz$.

Let $\psi_1 \in \cN\cN_{2,1}(9W+1,L)$ be the neural network in Lemma \ref{product app}. We define
\[
\psi_2(z_1,z_2) := \min\{ \max\{\psi_1(z_1,z_2),-1\},1 \} = \sigma(\psi_1(z_1,z_2)+1) - \sigma(\psi_1(z_1,z_2)-1) -1 \in [-1,1],
\]
then $\psi_2 \in \cN\cN_{2,1}(9W+1,L+1)$ and $\psi_2$ also satisfies the inequalities in Lemma \ref{product app}. For $i=3,4,\dots,k$, we define $\psi_i: [-1,1]^i \to [-1,1]$ inductively by
\[
\psi_i(z_1,\dots,z_i) := \psi_2(\psi_{i-1}(z_1,\dots,z_{i-1}), z_i).
\]
Since $z_i=\sigma(z_i+1)-1$ for $z_i\in [-1,1]$, it is easy to see that $\psi_i$ can be implemented by a network with width $9W+i-1$ and depth $(i-1)(L+1)$ by induction. Furthermore,
\begin{align*}
&|\psi_i(z_1,\dots,z_i) - z_1\cdots z_i| \\
\le & |\psi_2(\psi_{i-1}(z_1,\dots,z_{i-1}), z_i) - \psi_{i-1}(z_1,\dots,z_{i-1})z_i| + |\psi_{i-1}(z_1,\dots,z_{i-1})z_i - z_1\cdots z_i| \\
\le & 6W^{-L} + |\psi_{i-1}(z_1,\dots,z_{i-1}) - z_1\cdots z_{i-1}| \\
\le & \cdots \le (i-2)6W^{-L} + |\psi_2(z_1,z_2) - z_1z_2| \\
\le & (i-1)6W^{-L}.
\end{align*}
And for any $\Bz=(z_1,z_2,\dots,z_k), \Bz'= (z_1',z_2',\dots,z_k') \in [-1,1]^k$,
\begin{align*}
|\psi_i(z_1,\dots,z_i) - \psi_i(z_1',\dots,z_i')| &\le 7|\psi_{i-1}(z_1,\dots,z_{i-1}) - \psi_{i-1}(z_1',\dots,z_{i-1}')| + 7|z_i-z_i'| \\
&\le \cdots \le 7^{i-2} |\psi_2(z_1,z_2) -\psi_2(z_1',z_2')| + \sum_{j=3}^i 7^{i-j+1}|z_j-z_j'| \\
&\le 7^{i-1} \|\Bz-\Bz'\|_1.
\end{align*}

We define the target function as $\phi(x) := \psi_k(\phi_0(x))$, then $\phi \in \cN\cN(9W+k-1,(k-1)(L+1))$. And for $\Bx,\By\in [-1,1]^d$, denote $\Bz=\phi_0(\Bx)$ and $\Bz'=\phi_0(\By)$, we have
\begin{align*}
|\phi(\Bx) -P(\Bx) | &= |\psi_k(\Bz) - z_1z_2\cdots z_k|\le 6(k-1)W^{-L}, \\
|\phi(\Bx) - \phi(\By)| &= |\psi_k(\Bz) - \psi_k(\Bz')|\le 7^{k-1}\|\Bz-\Bz'\|_1 \le 7^{k-1}\|\Bs\|_1 \|\Bx-\By\|_\infty.
\end{align*}
So we finish the proof.
\end{proof}

\subsection{Proof of Theorem \ref{app theorem width depth}} \label{sec: proof app theorem width depth}

We divide the proof into four steps as follows.

\noindent \textbf{Step 1}: Discretization.

Let $M=\lfloor (WL)^{2/d}\rfloor$ and $\delta = \tfrac{1}{3M^{\alpha \lor 1}}\le \tfrac{1}{3M}$. For each $\Bm=(m_1,m_2,\dots,m_d)\in \{0,1,\dots,M-1 \}^d$, we define
\[
Q_\Bm := \left\{\Bx=(x_1,x_2,\dots,x_d): x_i\in \left[ \tfrac{m_i}{M}, \tfrac{m_i+1}{M} - \delta\cdot 1_{\{m_i<M-1\}} \right], i=1,2,\dots,d \right\}.
\]
Then $\bigcup_{\Bm} Q_\Bm$ approximately discretize $[0,1]^d$ with error $\delta$. Figure \ref{discretization} gives an illustration of the discretization.

\begin{figure}[htbp]
\centering
\includegraphics[width=0.5\textwidth]{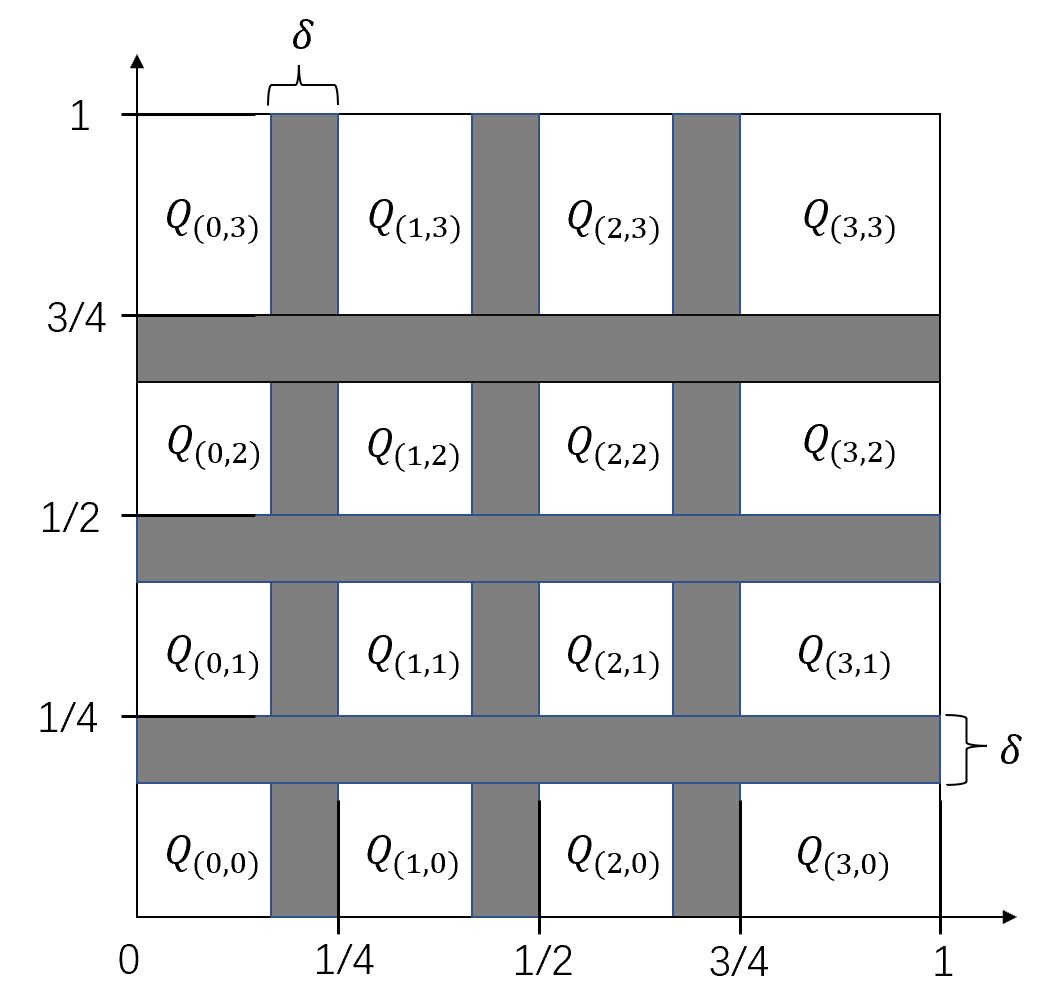}
\caption{An illustration of the discretization of $[0,1]^d$ with $M=4$ and $d=2$}
\label{discretization}
\end{figure}

By Proposition \ref{partition map}, there exists $\psi_1\in \cN\cN_{1,1}(4W+3,4L)$ such that
\[
\psi_1(t) = \tfrac{m}{M}, \quad \mbox{if } t\in \left[\tfrac{m}{M}, \tfrac{m+1}{M}-\delta\cdot 1_{\{m<M-1\}} \right], m=0,1,\dots,M-1,
\]
and $\Lip (\psi_1)\le 2LM^{-2}\delta^{-2}$. We define
\[
\psi(\Bx) := (\psi_1(x_1),\dots,\psi_1(x_d)), \quad \Bx=(x_1,\dots,x_d) \in \bR^d.
\]
Then, $\psi\in \cN\cN_{d,d}(d(4W+3),4L)$ and $\psi(\Bx) = \tfrac{\Bm}{M}$ for $\Bx\in Q_\Bm$. Hence, $\psi$ approximately maps each $\Bx\in [0,1]^d$ to its index in the discretization.

\noindent \textbf{Step 2}: Approximation of Taylor coefficients.

Since $\Bm \in \{0,1,\dots,M-1\}^d$ is one-to-one correspondence to the index $i_\Bm := \sum_{j=1}^d m_j M^{j-1} \in \{0,1,\dots,M^d-1\}$, we define
\[
\psi_0(\Bx):= (M,M^2,\dots,M^d) \cdot \psi(\Bx) = \sum_{j=1}^d \psi_1(x_j) M^j \quad \Bx\in \bR^d,
\]
then $\psi_0\in \cN\cN_{d,1}(d(4W+3),4L)$ and
\[
\psi_0(\Bx) = \sum_{j=1}^d m_j M^{j-1}=i_\Bm \quad \mbox{if } \Bx\in Q_\Bm,\ \Bm \in \{0,1,\dots,M-1\}^d.
\]
For any $\Bx,\Bx'\in \bR^d$, we have
\[
|\psi_0(\Bx)-\psi_0(\Bx')| \le \sum_{j=1}^d M^j |\psi_1(x_j)-\psi_1(x_j')| \le dM^d \Lip (\psi_1) \|\Bx-\Bx'\|_\infty \le 2d LM^{d-2}\delta^{-2} \|\Bx-\Bx'\|_\infty.
\]

For any $\Bs\in \bN_0^d$ satisfying $\|\Bs\|_1\le r$ and each $i=i_\Bm \in \{0,1,\dots,M^d-1\}$, we denote $\xi_{\Bs,i} := (\partial^\Bs h(\Bm /M)+1)/2 \in [0,1]$. Since $M^d \le W^2L^2$, by Proposition \ref{data fitting}, there exists $\varphi_\Bs\in \cN\cN(8(r+1)(2W+1) \lceil \log_2 (2W)\rceil +2,4L \lceil \log_2 (2L)\rceil+1)$ such that $\Lip (\varphi_\Bs) \le 4 \cdot 2^{L^2} + 2L^2 \le 5\cdot 2^{L^2}$ and $|\varphi_\Bs(i) - \xi_{\Bs,i}|\le (WL)^{-2(r+1)}$ for all $i \in \{0,1,\dots,M^d-1\}$. We define
\[
\phi_\Bs(\Bx) := 2\varphi_\Bs(\psi_0(\Bx))-1 \in [-1,1], \quad x\in \bR^d.
\]
Then $\phi_\Bs$ can be implemented by a network with width $8d(r+1)(2W+1)\lceil \log_2 (2W)\rceil +2\le 40d(r+1)W \lceil \log_2 W\rceil$ and depth $4L +4L\lceil \log_2 (2L)\rceil+1 \le 13L \lceil\log_2 L \rceil$. And we have
\begin{equation}\label{phi_alpha lip}
\Lip (\phi_\Bs) \le 2\Lip (\varphi_\Bs) \Lip (\psi_0)\le 20d LM^{d-2}\delta^{-2} 2^{L^2},
\end{equation}
and for any $\Bm\in \{0,1,\dots,M-1 \}^d$, if $\Bx\in Q_\Bm$,
\begin{equation}\label{phi_alpha bound}
|\phi_\Bs(\Bx)-\partial^\Bs h(\Bm/M)| = 2|\varphi_\Bs(i_\Bm)-\xi_{\Bs,i_\Bm}|\le 2(WL)^{-2(r+1)}.
\end{equation}

\noindent \textbf{Step 3}: Approximation of $h$ on $\bigcup_{\Bm\in \{0,1,\dots,M-1 \}^d} Q_\Bm$.

Let $\varphi(t) = \min\{\max\{t,0\},1 \} = \sigma(t) - \sigma(t-1)$ for $t\in \bR$. We extend its definition to $\bR^d$ coordinate-wisely, so $\varphi:\bR^d \to [0,1]^d$ and $\varphi(\Bx) = \Bx$ for any $\Bx\in [0,1]^d$.

By Lemma \ref{product app}, there exists $\phi_\times \in \cN\cN(9W+1,2(r+1)L)$ such that for any $t_1,t_2,t_3,t_4\in [-1,1]$,
\begin{align}
|t_1t_2 - \phi_\times(t_1,t_2)| &\le  6W^{-2(r+1)L}, \label{phi_times bound} \\
|\phi_\times(t_1,t_2) - \phi_\times(t_3,t_4)| &\le  7|t_1-t_3| + 7|t_2-t_4|. \label{phi_times lip}
\end{align}
By corollary \ref{poly app}, for any $\Bs\in \bN_0^d$ with $2\le \|\Bs\|_1\le r$,  there exists $P_\Bs\in \cN\cN(9W+r-1,(r-1)(2(r+1)L+1))$ such that for any $\Bx,\By\in [-1,1]^d$, $P_\Bs(\Bx) \in [-1,1]$ and
\begin{align}
|P_\Bs(\Bx) -\Bx^\Bs |\le 6 (r-1)W^{-2(r+1)L}, \label{P_s bound} \\
|P_\Bs(\Bx) - P_\Bs(\By)| \le r7^{r-1} \|\Bx-\By\|_\infty. \label{P_s lip}
\end{align}
When $\|\Bs\|_1= 1$, it is easy to implemented $P_\Bs(\Bx) = \Bx^\Bs$ by a neural network with Lipschitz constant at most one. Hence, the inequalities (\ref{P_s bound}) and (\ref{P_s lip}) hold for $1\le \|\Bs\|_1\le r$.

For any $\Bx\in Q_\Bm$, $\Bm\in \{0,1,\dots,M-1\}^d$, we can approximate $h(\Bx)$ by a Taylor expansion. Thanks to Lemma \ref{Taylor Theorem}, we have the following error estimation for $\Bx\in Q_\Bm$,
\begin{equation}\label{Taylor app bound}
\left| h(\Bx) - h(\tfrac{\Bm}{M}) - \sum_{1\le \|\Bs\|_1 \le r} \frac{\partial^\Bs h(\tfrac{\Bm}{M})}{\Bs !} (\Bx-\tfrac{\Bm}{M})^\Bs \right| \le d^r \|\Bx-\tfrac{\Bm}{M}\|_\infty^\alpha \le d^r M^{-\alpha}.
\end{equation}
Motivated by this, we define
\begin{align*}
\widetilde{\phi}_0(\Bx) &:= \phi_{\mathbf{0}_d}(\Bx) + \sum_{1\le \|\Bs\|_1 \le r} \phi_\times \left( \tfrac{\phi_\Bs(\Bx)}{\Bs !}, P_\Bs(\varphi(\Bx) - \psi(\Bx) ) \right),\\
\phi_0(\Bx) &:= \sigma(\widetilde{\phi}_0(\Bx)+1) - \sigma(\widetilde{\phi}_0(\Bx)-1) -1 \in [-1,1],
\end{align*}
where we denote $\mathbf{0}_d =(0,\dots,0) \in \bN_0^d$. Observe that the number of terms in the summation can be bounded by
\[
\sum_{\Bs \in \bN_0^d, \|\Bs\|_1\le r} 1 = \sum_{j=0}^r \sum_{\Bs \in \bN_0^d, \|\Bs\|_1= j} 1 \le \sum_{j=0}^r d^j \le (r+1)d^r.
\]
Recall that $\varphi\in \cN\cN(2d,1)$, $\psi\in \cN\cN(d(4W+3),4L)$, $P_\Bs\in \cN\cN(9W+r-1,2(r^2-1)L+r-1)$, $\phi_\Bs\in \cN\cN(40d(r+1)W \lceil\log_2 W\rceil, 13L \lceil\log_2 L\rceil)$ and $\phi_\times \in \cN\cN(9W+1,2(r+1)L)$. Hence, by our construction, $\phi_0$ can be implemented by a neural network with width $49(r+1)^2 d^{r+1}W \lceil\log_2 W\rceil$ and depth $15(r+1)^2 L \lceil\log_2 L\rceil$.

For any $1\le \|\Bs\|_1\le r$ and $\Bx,\By\in \bR^d$, since $\phi_\Bs(\Bx), \phi_\Bs(\By), \varphi(\Bx) - \psi(\Bx), \varphi(\By) - \psi(\By) \in [-1,1]$, by inequalities (\ref{phi_alpha lip}), (\ref{phi_times lip}) and (\ref{P_s lip}), we have
\begin{align*}
&\left| \phi_\times \left( \tfrac{\phi_\Bs(\Bx)}{\Bs !}, P_\Bs(\varphi(\Bx) - \psi(\Bx) ) \right) - \phi_\times \left( \tfrac{\phi_\Bs(\By)}{\Bs !}, P_\Bs(\varphi(\By) - \psi(\By) ) \right) \right| \\
\le & 7|\phi_\Bs(\Bx) - \phi_\Bs(\By)| + 7 |P_\Bs(\varphi(\Bx) - \psi(\Bx)) - P_\Bs(\varphi(\By) - \psi(\By)) | \\
\le & 7 \Lip (\phi_\Bs) \|\Bx-\By\|_\infty + r7^r \| \varphi(\Bx) - \varphi(\By)\|_\infty + r7^r \|\psi(\Bx) - \psi(\By)\|_\infty \\
\le &  140d LM^{d-2}\delta^{-2} 2^{L^2}\|\Bx-\By\|_\infty + r7^r \| \Bx-\By\|_\infty + 2r 7^r LM^{-2}\delta^{-2} \|\Bx-\By\|_\infty \\
\le & LM^{2(\alpha \lor 1)-2} (1260dM^d 2^{L^2}+ 19r 7^r) \|\Bx-\By\|_\infty.
\end{align*}
One can check that the bound also holds for $\|\Bs\|_1=0$ and $r=0$. Hence,
\begin{align*}
\Lip(\phi_0) \le \Lip (\widetilde{\phi}_0)&\le \sum_{\|\Bs\|_1\le r} LM^{2(\alpha \lor 1)-2} (1260dM^d 2^{L^2}+ 19r 7^r) \\
&\le (r+1) d^r L(WL)^{\sigma(4\alpha-4)/d} (1260d W^2L^2 2^{L^2}+ 19r 7^r).
\end{align*}

We can estimate the error $|h(\Bx) - \phi_0(\Bx)|$ as follows. For any $\Bx\in Q_\Bm$, we have $\varphi(\Bx)=\Bx$ and $\psi(\Bx) =\tfrac{\Bm}{M}$. Hence, by the triangle inequality and inequality (\ref{Taylor app bound}),
\begin{align*}
&|h(\Bx) - \phi_0(\Bx)| \le |h(\Bx) - \widetilde{\phi}_0(\Bx)| \\
\le& |h(\tfrac{\Bm}{M}) - \phi_{\mathbf{0}_d}(\Bx)| + \sum_{1\le \|\Bs\|_1 \le r} \left|  \frac{\partial^\Bs h(\tfrac{\Bm}{M})}{\Bs !} (\Bx-\tfrac{\Bm}{M})^\Bs -  \phi_\times \left( \tfrac{\phi_\Bs(\Bx)}{\Bs!}, P_\Bs(\Bx- \tfrac{\Bm}{M} ) \right) \right| + d^r M^{-\alpha} \\
=&: \sum_{\|\Bs\|_1 \le r} \cE_\Bs + d^r \lfloor (WL)^{2/d}\rfloor^{-\alpha}.
\end{align*}
Using the inequality $|t_1t_2 - \phi_\times(t_3,t_4)| \le |t_1t_2 - t_3t_2| + |t_3t_2 - t_3t_4| + |t_3t_4 - \phi_\times(t_3,t_4)| \le |t_1 - t_3| + |t_2 - t_4| + |t_3t_4 - \phi_\times(t_3,t_4)|$ for any $t_1,t_2,t_3,t_4\in [-1,1]$ and the inequalities (\ref{phi_alpha bound}), (\ref{phi_times bound}) and (\ref{P_s bound}), we have for $1\le \|\Bs\|_1\le r$,
\begin{align*}
\cE_\Bs \le& \tfrac{1}{\Bs!} \left|\partial^\Bs h(\tfrac{\Bm}{M}) - \phi_\Bs(\Bx)\right| + \left|(\Bx-\tfrac{\Bm}{M})^\Bs - P_\Bs(\Bx-\tfrac{\Bm}{M})\right| \\
& \quad + \left|\tfrac{\phi_\Bs(\Bx)}{\Bs !}P_\Bs(\Bx- \tfrac{\Bm}{M} ) - \phi_\times \left( \tfrac{\phi_\Bs(\Bx)}{\Bs !}, P_\Bs(\Bx- \tfrac{\Bm}{M} ) \right) \right| \\
\le & 2(WL)^{-2(r+1)} + 6(r-1)W^{-2(r+1)L} + 6W^{-2(r+1)L} \\
\le & (6r+2)(WL)^{-2(r+1)}.
\end{align*}
It is easy to check that the bound is also true for $\|\Bs\|_1=0$ and $r=0$. Therefore,
\begin{align*}
|h(\Bx) - \phi_0(\Bx)| &\le \sum_{\|\Bs\|_1 \le r} (6r+2)(WL)^{-2(r+1)} + d^r \lfloor (WL)^{2/d}\rfloor^{-\alpha} \\
&\le (r+1)d^r(6r+2)(WL)^{-2(r+1)} + d^r \lfloor (WL)^{2/d}\rfloor^{-\alpha} \\
&\le (6r+3)(r+1)d^r \lfloor (WL)^{2/d}\rfloor^{-\alpha} \\
&=: \cE,
\end{align*}
for any  $\Bx\in \bigcup_{\Bm\in \{0,1,\dots,M-1 \}^d} Q_\Bm$.

\noindent \textbf{Step 4}: Approximation of $h$ on $[0,1]^d$.

Next, we construct a neural network $\phi$ that uniformly approximates $h$ on $[0,1]^d$. To present the construction, we denote $\mid(t_1,t_2,t_3)$ as the function that returns the middle value of three inputs $t_1,t_2,t_3 \in \bR$. It is easy to check that
\[
\max\{t_1,t_2 \} = \frac{1}{2} (\sigma(t_1+t_2) - \sigma(-t_1-t_2) + \sigma(t_1-t_2) + \sigma(t_2-t_1) )
\]
Thus, $\max\{t_1,t_2,t_3\} = \max\{\max\{t_1,t_2\},\sigma(t_3)-\sigma(-t_3) \}$ can be implemented by a network with width $6$ and depth $2$. Similar construction holds for $\min\{t_1,t_2,t_3\}$. Since
\[
\mid(t_1,t_2,t_3) = \sigma(t_1+t_2+t_3) - \sigma(-t_1-t_2-t_3) - \max\{t_1,t_2,t_3 \} - \min\{t_1,t_2,t_3 \},
\]
it is easy to see $\mid(\cdot,\cdot,\cdot) \in \cN\cN(14,2)$.

Recall that $\phi_0\in \cN\cN(49(r+1)^2 d^{r+1}W \lceil\log_2 W\rceil,15(r+1)^2 L \lceil\log_2 L\rceil)$. Let $\{\Be_i\}_{i=1}^d$ be the standard basis in $\bR^d$. We inductively define
\[
\phi_i(\Bx) := \mid (\phi_{i-1}(\Bx-\delta \Be_i), \phi_{i-1}(\Bx), \phi_{i-1}(\Bx+\delta \Be_i) ) \in [-1,1], \quad i=1,2,\dots,d.
\]
Then $\phi_d \in \cN\cN(49(r+1)^2 3^d d^{r+1}W \lceil\log_2 W\rceil, 15(r+1)^2 L \lceil \log_2 L\rceil+2d)$. For any $\Bx,\Bx'\in \bR^d$, the functions $\phi_{i-1}(\cdot-\delta \Be_i)$, $\phi_{i-1}(\cdot)$ and $\phi_{i-1}(\cdot+\delta \Be_i)$ are piece-wise linear on the segment that connecting $\Bx$ and $\Bx'$. Hence, the Lipschitz constant of these functions on the segment is the maximum absolute value of the slopes of linear parts. Since the middle function does not increase the maximum absolute value of the slopes, it does not increase the Lipschitz constant, which shows that $\Lip \phi_d \le \Lip \phi_0$.

Denote $Q(M,\delta) := \bigcup_{m=0}^{M-1} [\frac{m}{M}, \frac{m+1}{M}-\delta \cdot 1_{\{m<M-1\}} ]$ and define, for $i=0,1,\dots,d$,
\[
E_i := \{ (x_1,x_2,\dots,x_d)\in [0,1]^d: x_j\in Q(M,\delta), j>i \},
\]
then $E_0 = \bigcup_{\Bm\in \{0,1,\dots,M-1 \}^d} Q_\Bm$ and $E_d = [0,1]^d$. We assert that
\[
|\phi_i(\Bx) - h(\Bx)|\le \cE + i\delta^{\alpha \land 1}, \quad \forall x\in E_i, i=0,1,\dots,d.
\]

We prove the assertion by induction. By construction, it is true for $i=0$. Assume the assertion is true for some $i$, we will prove that it also holds for $i+1$. For any $\Bx\in E_{i+1}$, at least two of $\Bx-\delta \Be_{i+1}$, $\Bx$ and $\Bx+\delta \Be_{i+1}$ are in $E_i$. Therefore, by assumption and the inequality $|h(\Bx)-h(\Bx\pm \delta \Be_{i+1})|\le \delta^{\alpha \land 1}$, at least two of the following inequalities hold
\begin{align*}
|\phi_i(\Bx-\delta \Be_{i+1}) - h(\Bx)| &\le |\phi_i(\Bx-\delta \Be_{i+1}) - h(\Bx-\delta \Be_{i+1})| + \delta^{\alpha \land 1} \le \cE + (i+1)\delta^{\alpha \land 1},\\
|\phi_i(\Bx) - h(\Bx)| &\le \cE + i\delta^{\alpha \land 1}, \\
|\phi_i(\Bx+\delta \Be_{i+1}) - h(\Bx)| &\le |\phi_i(\Bx+\delta \Be_{i+1}) - h(\Bx+\delta \Be_{i+1})| + \delta^{\alpha \land 1} \le \cE + (i+1)\delta^{\alpha \land 1}.
\end{align*}
In other words, at least two of $\phi_i(\Bx-\delta \Be_{i+1})$, $\phi_i(\Bx)$ and $\phi_i(\Bx+\delta \Be_{i+1})$ are in the interval $[h(\Bx)-\cE - (i+1)\delta^{\alpha \land 1}, h(\Bx)+\cE + (i+1)\delta^{\alpha \land 1} ]$. Hence, their middle value $\phi_{i+1}(\Bx) = \mid(\phi_i(\Bx-\delta \Be_{i+1}), \phi_i(\Bx), \phi_i(\Bx+\delta \Be_{i+1}))$ must be in the same interval, which means
\[
|\phi_{i+1}(\Bx) -h(\Bx)| \le \cE + (i+1)\delta^{\alpha \land 1}.
\]
So the assertion is true for $i+1$.

Recall that
\[
\delta^{\alpha \land 1} = \left( \frac{1}{3M^{\alpha \lor 1}} \right)^{\alpha \land 1} =
\begin{cases}
\frac{1}{3} M^{-\alpha} \quad &\alpha \ge 1, \\
(3M)^{-\alpha} \quad &\alpha < 1,
\end{cases}
\]
and $M=\lfloor (WL)^{2/d}\rfloor$. Since $E_d=[0,1]^d$, let $\phi := \phi_d$, we have
\begin{align*}
\| \phi - h\|_{L^\infty([0,1]^d)} &\le \cE + d\delta^{\alpha \land 1} \\
&\le (6r+3)(r+1)d^r \lfloor (WL)^{2/d}\rfloor^{-\alpha} + d \lfloor (WL)^{2/d}\rfloor^{-\alpha} \\
&\le 6(r+1)^2 d^{r \lor 1} \lfloor (WL)^{2/d}\rfloor^{-\alpha},
\end{align*}
and we complete the proof.

\section{Approximation by norm constrained neural networks}\label{sec: app norm constraint}

This section studies the approximation of H\"older function $h\in \cH^\alpha([0,1]^d)$ by norm constrained neural networks $\cN\cN(W,L,K)$. Since the ReLU function is $1$-Lipschitz, it is easy to see that, for any $\phi_\theta\in \cN\cN(W,L,K)$ with parameters $\theta$,
\[
\Lip (\phi_\theta) \le \kappa(\theta) \le K.
\]
However, it was shown by \cite{huster2019limitations} that some simple $1$-Lipschitz functions, such as $f(x) = |x|$, can not be represented by $\cN\cN(W,L,K)$ for any $K<2$. Their result implies that norm constrained neural networks have a restrictive expressive power. Nevertheless, since two-layer neural networks are universal, $\cN\cN(W,L,K)$ can approximate any continuous functions when $W$ and $K$ are sufficiently large. Recall that we have denoted the approximation error as
\[
\cE(\cH^\alpha, \cN\cN(W,L,K),[0,1]^d) := \sup_{h\in \cH^\alpha} \inf_{\phi \in \cN\cN(W,L,K)} \| h- \phi\|_{L^\infty ([0,1]^d)}.
\]
Our main results can be summarized in the following theorem.

\begin{theorem}\label{app theorem norm constraint}
Let $d\in \bN$ and $\alpha = r+\alpha_0>0$ with $r\in \bN_0 $ and $\alpha_0\in (0,1]$. Then, there exists $c>0$ such that for any $W\ge c K^{(2d+\alpha)/(2d+2)}$ and $L \ge 2\lceil \log_2 (d+r) \rceil+2$,
\[
\cE(\cH^\alpha, \cN\cN(W,L,K),[0,1]^d) \lesssim K^{-\alpha/(d+1)}.
\]
\end{theorem}

The proof idea is similar to the proof of Theorem \ref{app theorem width depth}. We explicitly construct neural networks to approximate the local Taylor polynomials. But, in stead of controlling the Lipschitz constant of the constructed function as in Theorem \ref{app theorem width depth}, we need to control the norm of weighs in neural network.

\subsection{Approximation of polynomials}

Similar to the proof of Theorem \ref{app theorem width depth}, we first consider the approximation of the quadratic function $f(x)=x^2$ and then extend the approximation to monomials.

\begin{lemma}\label{square}
For any $k\in \bN$, there exists $\phi_k \in \cN\cN(k,1,3)$ such that $\phi_k(x)=0$ for $x\le 0$, $\phi_k(x)\in[0,1]$ for $x\in [0,1]$ and
\[
\left|x^2 - \phi_k(x) \right| \le \frac{1}{2k^2}, \quad x\in [0,1].
\]
\end{lemma}
\begin{proof}
The construction is based on the integral representation of $x^2$:
\begin{equation}\label{x^2 representation}
x^2 = \int_0^x 2x-2b db = \int_0^x 2\sigma(x-b) db = \int_0^1 2\sigma(x-b) db, \quad x\in [0,1].
\end{equation}
We can approximate the integral by Riemann sum. For any $k\in \bN$, define
\[
\phi_k(x) = \frac{1}{k} \sum_{i=1}^k 2 \sigma \left(x-\frac{2i-1}{2k} \right).
\]
Then, by Proposition \ref{basic construct}, $\phi_k\in \cN\cN(k,1,K)$ with
\[
K =  \sum_{i=1}^k \frac{2}{k} \left(1+\frac{2i-1}{2k}\right) = 3.
\]
It is easy to see that $\phi_k(x)=0$ for $x\le 0$. Since $\phi_k$ is an increasing function, we have $0=\phi_k(0)\le \phi_k(x) \le \phi_k(1) =1$ for $x\in[0,1]$.

For any $x\in(0,1]$, let us denote $i_x = \lceil kx \rceil \in \{1,\dots,k\}$, then $x\in ((i_x-1)/k,i_x/k]$. If $i<i_x$, then
\[
\int_{(i-1)/k}^{i/k} 2\sigma(x-b) db = \int_{(i-1)/k}^{i/k} 2x-2b db = \frac{2x}{k} - \frac{2i-1}{k^2} = \frac{2}{k} \sigma\left(x-\frac{2i-1}{2k} \right) .
\]
If $i>i_x$, then
\[
\int_{(i-1)/k}^{i/k} 2\sigma(x-b) db = 0 = \frac{2}{k} \sigma\left(x-\frac{2i-1}{2k} \right) .
\]
Therefore,
\begin{align*}
\left|x^2 - \phi_k(x)\right| &= \left| \sum_{i=1}^k \int_{(i-1)/k}^{i/k} 2\sigma(x-b) db -  \sum_{i=1}^k\frac{2}{k}\sigma \left(x-\frac{2i-1}{2k} \right) \right| \\
&= \left| \int_{(i_x-1)/k}^{i_x/k} 2\sigma(x-b)  - 2\sigma \left(x-\frac{2i_x-1}{2k} \right) db \right| \\
&\le \int_{(i_x-1)/k}^{i_x/k} 2 \left| b - \frac{2i_x-1}{2k} \right| db = \frac{1}{2k^2},
\end{align*}
where we use the Lipschitz continuity of ReLU in the third inequality. 
\end{proof}

\begin{remark}
The construction here is based on the integral representation (\ref{x^2 representation}), which can be regarded as an infinite width neural network.
It is different from the construction in \cite{yarotsky2017error,lu2021deep} and Lemma \ref{product app}, which use the teeth function $T_{i} = T_1 \circ T_{i-1} = T_1 \circ \cdots \circ T_1$ to construct the approximator
$f_k(x) = x- \sum_{i=1}^k 4^{-i} T_i(x)$ that achieves the approximation error $|x^2-f_k(x)|\le 2^{-2(k+1)}$. 
Since $T_1\in \cN\cN(2,2,7)$, by Proposition \ref{basic construct}, this compositional property implies $T_i\in \cN\cN(2,2i,7^i)$ and consequently $f_k \in \cN\cN(2k+1,2k,\frac{4}{3}(\frac{7}{4})^{k+1}-\frac{4}{3})$. Hence, in the construction of \cite{yarotsky2017error,lu2021deep}, the approximation error decays exponentially on the depth but only polynomially on the norm constraint $K$. On the contrary, in our construction, the network has a finite norm constraint but the approximation error decays only quadratically on the width.
\end{remark}

As in Lemma \ref{product app} and Corollary \ref{poly app}, using the relation $xy = 2\left((\frac{x+y}{2})^2- (\frac{x}{2})^2 - (\frac{y}{2})^2 \right)$, we can approximate the product function by neural networks and then further approximate any monomials $x_1\cdots x_d$.

\begin{lemma}\label{product}
For any $k\in \bN$, there exists $\psi_k \in \cN\cN(6k,2,216)$ such that $\psi_k:[-1,1]^2 \to [-1,1]$, $\psi_k(x,y)=0$ if $xy=0$ and
\[
|xy - \psi_k(x,y)| \le \frac{3}{k^2}, \quad x,y\in [-1,1].
\]
\end{lemma}
\begin{proof}
Let $\phi_k \in \cN\cN(k,1,3)$ be the network in Lemma \ref{square} and define $\widetilde{\phi}_k(x) = \phi_k(x)+\phi_k(-x)$. By Proposition \ref{basic construct}, $\widetilde{\phi}_k \in \cN\cN(2k,1,6)$. Since $\phi_k(x)=0$ for $x\le 0$, we have $\widetilde{\phi}_k(x) = \phi_k(|x|)$ and the approximation error is 
\[
\left| x^2 - \widetilde{\phi}_k(x) \right| = \left| x^2 - \phi_k(|x|) \right| \le \frac{1}{2k^2}, \quad x\in [-1,1].
\]
Using the fact that $xy = 2\left((\frac{x+y}{2})^2- (\frac{x}{2})^2 - (\frac{y}{2})^2 \right)$, we consider the function
\[
\widetilde{\psi}_k(x,y) := 2\widetilde{\phi}_k\left(\frac{1}{2}x + \frac{1}{2}y \right) -  2\widetilde{\phi}_k\left(\frac{1}{2}x \right) -  2\widetilde{\phi}_k\left(\frac{1}{2}y\right).
\]
Then, $\widetilde{\psi}_k(x,y)=0$ if $xy=0$, and, for any $x,y\in[-1,1]$,
\[
\left|xy - \widetilde{\psi}_k(x,y)\right| \le 2\left| \left(\frac{x+y}{2}\right)^2 - \widetilde{\phi}_k\left(\frac{x+y}{2}\right)\right| + 2\left| \left(\frac{x}{2}\right)^2 - \widetilde{\phi}_k\left(\frac{x}{2}\right)\right| + 2\left| \left(\frac{y}{2}\right)^2 - \widetilde{\phi}_k\left(\frac{y}{2}\right)\right| \le \frac{3}{k^2}.
\]
By Proposition \ref{basic construct}, $\widetilde{\psi}_k\in \cN\cN(6k, 1, 36)$.

Finally, let $\chi(x) = \sigma(x) - \sigma(-x) -2\sigma(\tfrac{1}{2}x-\tfrac{1}{2}) +2\sigma(-\tfrac{1}{2}x-\tfrac{1}{2})  = (x \lor -1) \land 1$, then $\chi \in \cN\cN(4,1,6)$. We construct the target function as
\[
\psi_k(x,y) = \chi (\widetilde{\psi}_k(x,y)) = (\widetilde{\psi}_k(x,y) \lor -1) \land 1.
\]
Then, for any $x,y\in[-1,1]$,
\[
|xy - \psi_k(x,y)| \le |xy - \widetilde{\psi}_k(x,y)| \le \frac{3}{k^2}.
\]
By Proposition \ref{basic construct}, $\psi_k \in \cN\cN(6k, 2, 216)$.
\end{proof}

\begin{lemma}\label{d product}
For any $d\ge 2$ and $k\in \bN$ , there exists $\phi \in \cN\cN(6d k, 2\lceil \log_2 d \rceil,6^{3\lceil \log_2 d \rceil})$ such that $\phi:[-1,1]^d \to [-1,1]$ and
\[
|x_1\cdots x_d - \phi(\Bx)| \le \frac{6d}{k^2}, \quad \Bx=(x_1,\dots,x_d)^\intercal \in [-1,1]^d.
\]
Furthermore, $\phi(\Bx)=0$ if $x_1\cdots x_d=0$.
\end{lemma}
\begin{proof}
We firstly consider the case $d=2^m$ for some $m\in\bN$. For $m=1$, by Lemma \ref{product}, there exists $\phi_1 \in \cN\cN(6k,2,216)$ such that $\phi_1:[-1,1]^2 \to [-1,1]$ and $|x_1x_2 - \phi_1(x_1,x_2)| \le 3k^{-2}$ for any $x_1,x_2\in [-1,1]$. We define $\phi_m:[-1,1]^{2^m} \to [-1,1]$ inductively by
\[
\phi_{m+1}(x_1,\dots,x_{2^{m+1}}) = \phi_1(\phi_m(x_1,\dots,x_{2^m}),\phi_m(x_{2^m+1},\dots,x_{2^{m+1}})).
\]
Then, $\phi_m(x_1,\dots,x_{2^m}) =0$ if $x_1\cdots x_{2^m}=0$ because this equation is true for $m=1$. Next, we inductively show that $\phi_m \in \cN\cN(3k 2^m, 2m,216^m)$ and
\[
|x_1\cdots x_{2^m} - \phi_m(x_1,\dots,x_{2^m})| \le (2^m - 1)\epsilon.
\]
where we denote $\epsilon:= 3k^{-2}$, i.e. the approximation error of $\phi_1$.

It is obvious that the assertion is true for $m=1$ by construction. Assume that the assertion is true for some $m\in \bN$, we will prove that it is true for $m+1$. By Proposition \ref{basic construct} and the construction of $\phi_{m+1}$, we have $\phi_{m+1} \in \cN\cN(3k 2^{m+1}, 2m+2,216^{m+1})$. For any $x_1,\dots,x_{2^{m+1}} \in [-1,1]$, we denote $s_1 := x_1\cdots x_{2^m}$, $t_1:=x_{2^m+1}\cdots x_{2^{m+1}}$, $s_2:= \phi_m(x_1,\dots,x_{2^m})$ and $t_2:=\phi_m(x_{2^m+1},\dots,x_{2^{m+1}})$, then $s_1,t_1,s_2,t_2\in [-1,1]$. By the hypothesis of induction,
\[
|s_1 - s_2|, |t_1-t_2| \le (2^m - 1)\epsilon.
\]
Therefore,
\begin{align*}
|x_1 \cdots x_{2^{m+1}} - \phi_{m+1}(x_1,\dots,x_{2^{m+1}})| 
=& |s_1t_1 - \phi_1(s_2,t_2)| \\
\le & |s_1t_1 - s_1t_2| + |s_1t_2 - s_2t_2| + |s_2t_2 - \phi_1(s_2,t_2)| \\
\le & |t_1 - t_2| + |s_1 - s_2| + \epsilon 
\le  (2^{m+1} - 1)\epsilon.
\end{align*}
Hence, the assertion is true for $m+1$.

For general $d\ge 2$, we choose $m=\lceil \log_2 d \rceil$, then $2^{m-1} < d \le 2^m$. We define the target function $\phi:[-1,1]^d \to [-1,1]$ by
\[
\phi(\Bx):= \phi_m \left(
\begin{pmatrix}
\Id_d \\
\boldsymbol{0}_{(2^m-d)\times d}
\end{pmatrix} \Bx +
\begin{pmatrix}
\boldsymbol{0}_{d\times 1} \\
\boldsymbol{1}_{(2^m-d)\times 1}
\end{pmatrix}
\right),
\]
where $\Id_d$ is $d\times d$ identity matrix, $\boldsymbol{0}_{p\times q}$ is $p\times q$ zero matrix and $\boldsymbol{1}_{(2^m-d)\times 1}$ is all ones vector. By Proposition \ref{basic construct}, $\phi \in \cN\cN(3k 2^m, 2m,216^m) \subseteq \cN\cN(6d k, 2\lceil \log_2 d \rceil,6^{3\lceil \log_2 d \rceil})$ and the approximation error is
\[
|x_1\cdots x_d - \phi(\Bx)| \le (2^m - 1)\epsilon \le 2d \epsilon = 6d k^{-2}.
\]
Furthermore, $\phi(\Bx)=0$ if $x_1\cdots x_d=0$ because $\phi_m$ has such property.
\end{proof}

\subsection{Proof of Theorem \ref{app theorem norm constraint}}

In Lemma \ref{d product}, we construct norm constrained neural networks to approximate monomials. Now, we can approximate any $h\in \cH^\alpha$ by approximating its local Taylor expansion
\begin{equation}\label{Taylor series}
p(\Bx) = \sum_{\Bn\in \{0,1,\dots,N\}^d} \psi_\Bn(\Bx) \sum_{\|\Bs\|_1\le r} \frac{\partial^\Bs h(\frac{\Bn}{N})}{\Bs !} \left(\Bx - \frac{\Bn}{N} \right)^\Bs,
\end{equation}
where the functions $\{\psi_\Bn\}_{\Bn}$ form a partition of unity of $[0,1]^d$ and each $\psi_\Bn$ is supported on a sufficiently small neighborhood of $\Bn/N$.

\begin{theorem}\label{app upper bound}
For any $N,k\in \bN$ and $h\in \cH^\alpha$ with $\alpha = r+\alpha_0$, where $r\in \bN_0 $ and $\alpha_0\in (0,1]$, there exists $\phi\in \cN\cN(W,L,K)$ where
\begin{align*}
W &= 6(r+1)(d+r) d^r (N+1)^d k, \\
L &= 2\lceil \log_2 (d+r) \rceil + 2, \\
K &= 6^{3\lceil \log_2 (d+r) \rceil+1}(r+1)d^rN(N+1)^d,
\end{align*}
such that
\[
\| h-\phi \|_{L^\infty([0,1]^d)} \le 2^d d^r(N^{-\alpha} + 6(r+1) (d+r) k^{-2} ).
\]
\end{theorem}

\begin{proof}
Let
\[
\psi(t) = \sigma(1-|t|) = \sigma(1-\sigma(t)-\sigma(-t)) \in [0,1], \quad t\in \bR,
\]
then $\psi \in \cN\cN(2,2,3)$ and the support of $\psi$ is $[-1,1]$. For any $\Bn=(n_1,\dots,n_d)\in \{0,1,\dots,N\}^d$, define
\[
\psi_\Bn(\Bx) := \prod_{i=1}^{d} \psi(Nx_i-n_i), \quad \Bx=(x_1,\dots,x_d)^\intercal \in \bR^d,
\]
then $\psi_\Bn$ is supported on $\{\Bx\in \bR^d: \|\Bx-\tfrac{\Bn}{N}\|_\infty \le \tfrac{1}{N} \}$. The functions $\{\psi_\Bn\}_\Bn$ form a partition of unity of the domain $[0,1]^d$:
\[
\sum_{\Bn\in \{0,1,\dots,N\}^d} \psi_\Bn(\Bx) = \prod_{i=1}^{d} \sum_{n_i=0}^N \psi(Nx_i-n_i) \equiv 1, \quad \Bx\in [0,1]^d.
\]

Let $p(\Bx)$ be the local Taylor expansion (\ref{Taylor series}). For convenience, we denote $p_{\Bn,\Bs}(\Bx):= \psi_\Bn(\Bx) (\Bx-\frac{\Bn}{N})^\Bs$ and $c_{\Bn,\Bs}:=\partial^\Bs h(\frac{\Bn}{N})/\Bs !$. Then, $p_{\Bn,\Bs}$ is supported on $\{\Bx\in \bR^d: \|\Bx-\tfrac{\Bn}{N}\|_\infty \le \tfrac{1}{N} \}$ and
\[
p(\Bx) = \sum_{\Bn\in \{0,1,\dots,N\}^d} \sum_{\|\Bs\|_1\le r} c_{\Bn,\Bs} p_{\Bn,\Bs}(\Bx).
\]
By lemma \ref{Taylor Theorem}, the approximation error is
\begin{align*}
|h(\Bx) -p(\Bx)| &= \left| \sum_\Bn \psi_\Bn(\Bx) h(\Bx) - \sum_{\Bn} \psi_\Bn(x) \sum_{\|\Bs\|_1\le r} c_{\Bn,\Bs} \left(\Bx - \frac{\Bn}{N} \right)^\Bs \right| \\
&\le \sum_\Bn \psi_\Bn(\Bx) \left| h(\Bx) - \sum_{\|\Bs\|_1\le r} c_{\Bn,\Bs} \left(\Bx - \frac{\Bn}{N} \right)^\Bs \right| \\
&= \sum_{\Bn: \|\Bx-\tfrac{\Bn}{N}\|_\infty < \tfrac{1}{N}} \left| h(\Bx) - \sum_{\|\Bs\|_1\le r} c_{\Bn,\Bs} \left(\Bx - \frac{\Bn}{N} \right)^\Bs \right| \\
&\le \sum_{\Bn: \|\Bx-\tfrac{\Bn}{N}\|_\infty < \tfrac{1}{N}} d^r \left\| \Bx - \frac{\Bn}{N} \right\|_\infty^\alpha \\
&\le 2^d d^r N^{-\alpha}.
\end{align*}

Let $\Phi_D\in \cN\cN(6D k, 2\lceil \log_2 D \rceil,6^{3\lceil \log_2 D \rceil})$ be the $D$-product function constructed in Lemma \ref{d product}. Then, we can approximate $p_{\Bn,\Bs}$ by
\[
\phi_{\Bn,\Bs}(\Bx) := \Phi_{d+\|\Bs\|_1}(\psi(Nx_1-n_1),\dots,\psi(Nx_d-n_d),\dots,x_i-\tfrac{n_i}{N},\dots),
\]
where the term $x_i-n_i/N$ appears in the input only when $s_i\neq 0$ and it repeats $s_i$ times. (When $d=1$ and $\Bs=\boldsymbol{0}$, we simply let $\phi_{n,\boldsymbol{0}}(x) =\psi(Nx-n) $.) Since $x_i-n_i/N = \sigma(x_i-n_i/N) - \sigma(-x_i+n_i/N)$ and $\|\Bs\|_1\le r$, by Proposition \ref{basic construct}, we have $\phi_{\Bn,\Bs}\in \cN\cN(6(d+r) k, 2\lceil \log_2 (d+r) \rceil + 2, 6^{3\lceil \log_2 (d+r) \rceil +1}N)$. By Lemma \ref{d product}, the approximation error is
\[
|p_{\Bn,\Bs}(\Bx) - \phi_{\Bn,\Bs}(\Bx)| \le 6(d+r) k^{-2}.
\]
Since $\Phi_D(t_1,\dots,t_D)=0$ when $t_1t_2\cdots t_D=0$, $\phi_{\Bn,\Bs}$ is supported on $\{\Bx\in \bR^d: \|\Bx-\tfrac{\Bn}{N}\|_\infty \le \tfrac{1}{N} \}$.

Now, we can approximate $p(\Bx)$ by
\[
\phi(\Bx) = \sum_{\Bn\in \{0,1,\dots,N\}^d} \sum_{\|\Bs\|_1\le r} c_{\Bn,\Bs} \phi_{\Bn,\Bs}(\Bx).
\]
Observe that $|c_{\Bn,\Bs}|=|\partial^\Bs f(\frac{\Bn}{N})/\Bs !|\le 1$ and the number of terms in the inner summation is
\[
\sum_{\|\Bs\|_1\le r} 1 = \sum_{j=0}^r \sum_{\|\Bs\|_1=j} 1 \le \sum_{j=0}^r d^j \le (r+1)d^r.
\]
The approximation error is, for any $\Bx\in [0,1]^d$,
\begin{align*}
|p(\Bx) - \phi(\Bx)| =& \left| \sum_\Bn \sum_{\|\Bs\|_1\le r} c_{\Bn,\Bs} p_{\Bn,\Bs}(\Bx) - \sum_\Bn \sum_{\|\Bs\|_1\le r} c_{\Bn,\Bs} \phi_{\Bn,\Bs}(\Bx) \right| \\
\le & \sum_\Bn \sum_{\|\Bs\|_1\le r} |c_{\Bn,\Bs}| |p_{\Bn,\Bs}(\Bx) - \phi_{\Bn,\Bs}(\Bx)| \\
\le& \sum_{\Bn: \|\Bx-\tfrac{\Bn}{N}\|_\infty < \tfrac{1}{N}} \sum_{\|\Bs\|_1\le r} |p_{\Bn,\Bs}(\Bx) - \phi_{\Bn,\Bs}(\Bx)| \\
\le & 6\cdot 2^d(r+1) (d+r) d^r k^{-2}.
\end{align*}
Hence, the total approximation error is
\[
|h(\Bx) - \phi(\Bx)| \le |h(\Bx) - p(\Bx)| + |p(\Bx) - \phi(\Bx)| \le 2^d d^r(N^{-\alpha} + 6(r+1) (d+r) k^{-2} ).
\]
Finally, by Proposition \ref{basic construct}, $\phi\in \cN\cN(6(r+1)(d+r) d^r (N+1)^d k, 2\lceil \log_2 (d+r) \rceil + 2, 6^{3\lceil \log_2 (d+r) \rceil+1}(r+1)d^rN(N+1)^d)$.
\end{proof}

Using the construction in Theorem \ref{app upper bound}, we can give a proof of Theorem \ref{app theorem norm constraint}.

\begin{proof}[Proof of Theorem \ref{app theorem norm constraint}]
We choose $N= \lceil k^{2/\alpha} \rceil$ in the Theorem \ref{app upper bound}, then it shows the existence of $\phi \in \cN\cN(W,L,K)$ with
\begin{align*}
W &= 6(r+1)(d+r) d^r (N+1)^d k \asymp k^{2d/\alpha+1}, \\
L &= 2\lceil \log_2 (d+r) \rceil + 2, \\
K &= 6^{3\lceil \log_2 (d+r) \rceil+1}(r+1)d^rN(N+1)^d \asymp k^{2(d+1)/\alpha},
\end{align*}
such that $\| h-\phi \|_{L^\infty([0,1]^d)} \le 2^d d^r(N^{-\alpha} + 6(r+1) (d+r) k^{-2} ) \lesssim k^{-2}$.
Therefore, $k \asymp K^{\alpha/(2d+2)}$, $W \asymp k^{2d/\alpha+1} \asymp K^{(2d+\alpha)/(2d+2)}$ and we have the approximation bound
\[
\| h-\phi \|_{L^\infty([0,1]^d)} \lesssim k^{-2} \lesssim K^{-\alpha/(d+1)}.
\]
Since increasing $W$ and $L$ can only decrease the approximation error, the bound holds for any $W\gtrsim K^{(2d+\alpha)/(2d+2)}$ and $L \ge 2 \lceil \log_2 (d+r) \rceil +2$.
\end{proof}

\section{Approximation lower bounds}\label{sec: app lower bound}

In Corollary \ref{app bound depth and width} and Theorem \ref{app theorem norm constraint}, we upper bound the approximation error for the H\"older class $\cH^\alpha$ by the size and norm constraint of neural network:
\begin{align}
\cE(\cH^\alpha,\cN\cN(W,L),[0,1]^d) &\lesssim (WL / (\log W \log L))^{-2\alpha/d}, \label{eq: app upper bound depth and width} \\
\cE(\cH^\alpha, \cN\cN(W,L,K),[0,1]^d) &\lesssim K^{-\alpha/(d+1)}.
\end{align}
This section studies the lower bounds of the approximation error. Our main idea is to find the connection between the approximation accuracy and the complexity (Pseudo-dimension and Rademacher complexity) of neural network classes.

\subsection{Lower bounding by Pseudo-dimension}

Let us first derive approximation lower bounds through the Pseudo-dimension. Intuitively, if a function class $\cF$ can approximate a function class $\cH$ of high complexity with small precision, then $\cF$ should also have high complexity. In other words, if we use a function class $\cF$ with $\Pdim(\cF)\le n$ to approximate a complex function class, we should be able to get a lower bound of the approximation error. Mathematically, we can define a nonlinear $n$-width using Pseudo-dimension: let $\cB$ be a normed space and $\cH\subseteq \cB$, we define
\[
\omega_n (\cH,\cB):= \inf_{\cF_n} \sup_{h\in \cH} \inf_{f\in\cF_n} \|h-f\|_\cB,
\]
where $\cF_n$ runs over all classes in $\cB$ with $\Pdim(\cF_n)\le n$. Since we only consider the continuous function class $\cB=C([0,1]^d)$ equipped with the sup-norm (or $L^\infty$ norm), we can simply denote
\[
\omega_n(\cH):= \omega_n (\cH,C([0,1]^d)) = \inf_{\Pdim(\cF)\le n} \cE(\cH,\cF,[0,1]^d).
\]
The $n$-width $\omega_n$ was firstly introduced by Maiorov and Ratsaby \cite{maiorov1999degree,ratsaby1997value}. They also gave upper and lower estimates of the $n$-width for Sobolev spaces. 

\begin{lemma}[\cite{maiorov1999degree}]\label{n-width lower bound}
For any $\alpha>0$, 
$
\omega_n(\cH^\alpha) \gtrsim n^{-\alpha/d}.
$
\end{lemma}

Combining Lemma \ref{n-width lower bound} with the Pseudo-dimension bound for neural networks from \cite{bartlett2019nearly}, we can derive lower bound for the approximation error. This lower bound shows that the upper bound (\ref{eq: app upper bound depth and width}) is asymptotically optimal up to a logarithm factor.

\begin{corollary}\label{app lower bound depth and width}
For any $\alpha>0$ and $L\ge 2$, 
\[
\cE(\cH^\alpha,\cN\cN(W,L),[0,1]^d) \gtrsim (W^2L^2 \log(WL))^{-\alpha/d}.
\]
\end{corollary}
\begin{proof}
We choose $n=\Pdim(\cN\cN(W,L))$, then by Lemma \ref{n-width lower bound} and the definition of the $n$-width, 
\[
\cE(\cH^\alpha,\cN\cN(W,L),[0,1]^d) \ge \omega_n(\cH) \gtrsim n^{-\alpha/d}.
\]
For ReLU neural networks, \cite{bartlett2019nearly} showed that the pseudo-dimension can be bounded as
\[
n=\Pdim(\cN\cN(W,L)) \lesssim UL \log U,
\]
where $U$ is the number of parameters and $U\asymp W^2L$ when $L\ge 2$. Thus, $n\lesssim W^2L^2 \log(WL)$ and the conclusion follows easily.
\end{proof}

\begin{remark}
So far, the approximation error is characterized by the number of neurons $WL$, we can also estimate the error by the number of weights $U$. To see this, let the width $W$ be sufficiently large and fixed, then the number of weights $U\asymp W^2L\asymp L$ and (\ref{eq: app upper bound depth and width}) implies
\[
\cE(\cH^\alpha,\cN\cN(W,L),[0,1]^d) \lesssim (U / \log U )^{-2\alpha/d}.
\]
For the lower bound, \cite{goldberg1995bounding} showed that the Pseudo-dimension of a ReLU neural network with $U$ parameters can be bounded as $\Pdim \lesssim U^2$. Hence, the argument in the proof of Corollary \ref{app lower bound depth and width} implies 
\[
\cE(\cH^\alpha,\cN\cN(W,L),[0,1]^d) \gtrsim U^{-2\alpha/d}.
\]
This shows that the upper bound is also optimal in terms of the number of parameters.
\end{remark}

\begin{remark}
The $n$-width $\omega_n$ is different from the famous continuous $n$-th width $\widetilde{\omega}_n$ introduced by \cite{devore1989optimal}:
\[
\widetilde{\omega}_n(\cH,\cB):= \inf_{\Ba,T_n} \sup_{h\in \cH} \|h-T_n(\Ba(h))\|_\cB,
\]
where $\Ba:\cH\to \bR^n$ is continuous and $T_n:\bR^n\to \cH$ is any mapping. In neural network approximation, $\Ba$ maps the target function $h\in\cH$ to the parameters in neural network and $T_n$ is the realization mapping that associates the parameters to the function realized by neural network. Applying the results in \cite{devore1989optimal}, one can show that the approximation error of $\cH^\alpha$ is lower bounded by $cU^{-\alpha/d}$, where $U$ is the number of parameters in the network, see also \cite{yarotsky2017error,yarotsky2020phase}. However, we have obtained an upper bound $\lesssim (U/\log U)^{-2\alpha/d}$ for these function classes. The inconsistency is because the parameters in our construction does not continuously depend on the target function and hence it does not satisfy the requirement in the $n$-width $\widetilde{\omega}_n$. This implies that we can get better approximation order by taking advantage of the incontinuity.
\end{remark}

\subsection{Lower bounding by Rademacher complexity}

We present two methods that give lower bounds for approximation error of norm constrained neural networks $\cN\cN(W,L,K)$. Both methods try to use the Rademacher complexity (Lemma \ref{Rademacher bound}) to lower bound the approximation capacity. The first method is inspired by the lower bound of $n$-width in Lemma \ref{n-width lower bound} and its proof given by \cite{maiorov1999degree}, which characterized the approximation order by pseudo-dimension. This method compares the packing numbers of neural networks $\cN\cN(W,L,K)$ and the target function class $\cH^\alpha$ on a suitably chosen data set. The second method establishes the lower bound by finding a linear functional that distinguishes the approximator and target classes. Using the second method, we give explicit constant on the approximation lower bound in Theorem \ref{explicit lower bound}, but it only holds for $\cH^1=\Lip 1$. The main lower bound is stated in the following theorem.

\begin{theorem}\label{app lower bound norm constraint}
Let $d\in \bN$ and $d>2\alpha>0$, then for any $W,L \in \bN$, $W\ge 2$ and $K \ge 1$,
\[
\cE(\cH^\alpha, \cN\cN(W,L,K),[0,1]^d) \gtrsim (K \sqrt{L})^{-2\alpha/(d-2\alpha)}.
\]
\end{theorem}

To prove Theorem \ref{app lower bound norm constraint}, let us begin with the estimation of the packing number of $\cH^\alpha$. We first construct a series of subsets $\cH^\alpha_N \subseteq \cH^\alpha$ with high complexity and simple structure. To this end, we choose a $C^\infty$ function $\psi:\bR^d\to [0,\infty)$ which satisfies $\psi(\boldsymbol{0})=1$ and $\psi(\Bx)=0$ for $\|\Bx\|_\infty\ge 1/4$, and let $C_{\psi,\alpha}>0$ be a constant such that $C_{\psi,\alpha} \psi \in \cH^\alpha(\bR^d)$. For any $N\in\bN$, we consider the function class
\begin{equation}\label{Halpha_N}
\cH^\alpha_N := \left\{ h_{\Ba}(\Bx) = \frac{C_{\psi,\alpha}}{N^{\alpha}} \sum_{\Bn\in \{0,\dots,N-1\}^d} a_{\Bn} \psi(N\Bx-\Bn): \Ba\in \cA_N \right\},
\end{equation}
where we denote $\cA_N:=\{ \Ba=(a_\Bn)_{\Bn \in\{0,\dots,N-1\}^d }: a_\Bn\in\{1,-1\}\}$ as the set of all sign vectors indexed by $\Bn$. Observe that, for the function $\psi_\Bn(\Bx) := \frac{C_{\psi,\alpha}}{N^{\alpha}} \psi(N\Bx-\Bn)$,
\begin{align*}
\sup_{\Bx\in \bR^d} |\partial^\Bs \psi_\Bn(\Bx)| &=  N^{\|\Bs\|_1-\alpha} C_{\psi,\alpha} \sup_{\Bx\in \bR^d} |\partial^\Bs \psi(\Bx)| \le 1, \quad &\|\Bs\|_1 \le r, \\
\sup_{\Bx\neq \By} \frac{|\partial^\Bs \psi_\Bn(\Bx)- \partial^\Bs \psi_\Bn(\By)|}{\|\Bx-\By\|_\infty^{\alpha_0}} &= N^{r-\alpha} C_{\psi,\alpha} \sup_{\Bx\neq \By} \frac{|\partial^\Bs \psi(\Bx)- \partial^\Bs \psi(\By)|}{N^{-{\alpha_0}} \|\Bx-\By\|_\infty^{\alpha_0}} \le 1, \quad &\|\Bs\|_1 = r,
\end{align*}
where $\alpha = r+\alpha_0>0$, with $r\in \bN_0, \alpha_0\in (0,1]$ and we use the fact $C_{\psi,\alpha} \psi \in \cH^\alpha(\bR^d)$. Therefore, $\psi_\Bn$ is also in $\cH^\alpha(\bR^d)$. Since the functions $\psi_\Bn$ have disjoint supports and $a_{\Bn} \in \{1,-1\}$, one can check that each $h_\Ba$ is in $\cH^\alpha(\bR^d)$ and hence $\cH^\alpha_N \subseteq \cH^\alpha$.

Next, we consider the packing number of $\cH^\alpha_N$ on the set $\Lambda_N:=\{\Bn/N: \Bn\in \{0,\dots,N-1\}^d \}$. For convenience, we will denote the function values of a function class $\cF$ on $\Lambda_N$ by
\[
\cF(\Lambda_N) := \{ (f(\Bn/N))_{\Bn\in \{0,\dots,N-1\}^d} :f\in \cF \} \subseteq \bR^m,
\]
where $m=|\Lambda_N|=N^d$ is the cardinality of $\Lambda_N$. Observe that, for $h_\Ba \in \cH^\alpha_N$,
\begin{equation}\label{hs values}
h_\Ba(\Bn/N) = \frac{C_{\psi,\alpha}}{N^{\alpha}} \sum_{\Bi\in \{0,\dots,N-1\}^d} a_{\Bi} \psi(\Bn-\Bi) = \frac{C_{\psi,\alpha}}{N^{\alpha}} a_{\Bn},
\end{equation}
where the last equality is because $\psi(\Bn-\Bi)=1$ if $\Bn=\Bi$ and $\psi(\Bn-\Bi)=0$ if $\Bn\neq\Bi$. We conclude that
\[
\cH^\alpha_N(\Lambda_N) = \{ C_{\psi,\alpha} N^{-\alpha} \Ba: \Ba\in \cA_N \} = C_{\psi,\alpha} N^{-\alpha} \cA_N.
\]
We will estimate the packing number of $\cH^\alpha_N(\Lambda_N)$ under the metric
\begin{equation}\label{rho2}
\rho_2(\Bx,\By) := \left( \frac{1}{m} \sum_{i=1}^{m} (x_i-y_i)^2 \right)^{1/2} = m^{-1/2} \|\Bx-\By\|_2, \quad \Bx,\By\in \bR^m.
\end{equation}
The following combinatorial lemma is sufficient for our purpose.

\begin{lemma}\label{combination}
Let $\cA:=\{\Ba=(a_1,\dots,a_m): a_i\in\{1,-1\}\}$ be the set of all sign vectors on $\bR^m$.
For any $m\ge 8$, there exists a subset $\cB\subseteq \cA$ whose cardinality $|\cB|\ge 2^{m/8}$, such that any two sign vectors $\Ba\neq \Ba'$ in $\cB$ are different in more than $\lfloor m/8 \rfloor$ places.
\end{lemma}
\begin{proof}
For any $\Ba \in \cA$, let $U(\Ba)$ be the set of all $\Ba'$ which are different from $\Ba$ in at most $k=\lfloor m/8 \rfloor$ places. Then,
\begin{align*}
|U(\Ba)| &\le \sum_{i=0}^{k} \binom{m}{i} \le (k+1) \binom{m}{k} \le (k+1) \left( \frac{me}{k} \right)^k \\
&\le \left(\frac{m}{8}+1\right) (16e)^{m/8} \le 2^{m/8} \cdot 64^{m/8} = 2^{7m/8}.
\end{align*}
We can construct the set $\cB= \{\Ba_1,\dots,\Ba_n\}$ as follows. We take $\Ba_1$ arbitrarily. Suppose the elements $\Ba_1,\dots,\Ba_j$ have been chosen, then $\Ba_{j+1}$ is taken arbitrarily from $\cA\setminus (\cup_{i=1}^j U(\Ba_i))$. Then, by construction, $\Ba_{j+1}$ and $\Ba_i$ ($1\le i\le j$) are different in more than $\lfloor m/8 \rfloor$ places. We do this process until the set $\cA\setminus (\cup_{i=1}^n U(\Ba_i))$ is empty. Since
\[
2^m = |\cA| \le \sum_{i=1}^n |U(\Ba_i)|\le n 2^{7m/8},
\]
we must have $|\cB|=n\ge 2^{m/8}$.
\end{proof}

By Lemma \ref{combination}, when $m=N^d\ge 8$, there exists a subset $\cB_N\subseteq \cA_N$ whose cardinality $|\cB_N|\ge 2^{m/8}$, such that any two vectors $\Ba\neq \Ba'$ in $\cB_N$ are different in more than $\lfloor m/8 \rfloor$ places. Thus,
\[
\rho_2(\Ba,\Ba') = m^{-1/2} \|\Ba-\Ba'\|_2 \ge 2 m^{-1/2} \lfloor m/8 \rfloor^{1/2} > 1/2.
\]
By equation (\ref{hs values}), this implies that
\[
\rho_2(h_\Ba(\Lambda_N), h_{\Ba'}(\Lambda_N)) > \frac{C_{\psi,\alpha}}{2 N^{\alpha}}.
\]
In other words, $\{ h_\Ba(\Lambda_N): \Ba\in \cB_N \}$ is a $\frac{1}{2}C_{\psi,\alpha}N^{-\alpha}$-packing of $\cH^\alpha_N(\Lambda_N)$ and hence we can lower bound the packing number
\begin{equation}\label{pack lower bound}
\cN_p(\cH^\alpha(\Lambda_N), \rho_2, \tfrac{1}{2}C_{\psi,\alpha}N^{-\alpha}) \ge \cN_p(\cH^\alpha_N(\Lambda_N), \rho_2, \tfrac{1}{2}C_{\psi,\alpha}N^{-\alpha}) \ge 2^{m/8} = 2^{N^d/8}.
\end{equation}

On the other hand, one can upper bound the packing number of a set in $\bR^m$ by its Rademacher complexity due to Sudakov minoration for Rademacher processes, see \cite[Corollary 4.14]{ledoux1991probability} for example.

\begin{lemma}[Sudakov minoration]\label{Sudakov}
Let $S$ be a subset of $\bR^m$. There exists a constant $C>0$ such that for any $\epsilon>0$,
\[
\log \cN_p(S,\rho_2,\epsilon) \le C \frac{m \cR_m(S)^2 \log \left(2+\frac{1}{\sqrt{m} \cR_m(S)} \right)}{\epsilon^2}.
\]
\end{lemma}

To simplify the notation, we denote $\Phi=\cN\cN(W,L,K)$. Lemma \ref{Rademacher bound} gives upper and lower bounds for the Rademacher complexity of $\Phi(\Lambda_N)$: for $K\ge 1$ and $W\ge 2$,
\[
\frac{1}{2\sqrt{2m}}\le \frac{K}{2\sqrt{2m}} \le \cR_m (\Phi(\Lambda_N)) \le \frac{2 K\sqrt{L+2+\log(d+1)}}{\sqrt{m}}.
\]
Together with Lemma \ref{Sudakov}, we can upper bound the packing number
\begin{equation}\label{pack upper bound}
\log \cN_p(\Phi(\Lambda_N), \rho_2, \epsilon) \le C \frac{K^2 L}{\epsilon^2},
\end{equation}
for some constant $C>0$.

Now, we are ready to prove our main lower bound for approximation error in Theorem \ref{app lower bound norm constraint}. The idea is that, if the approximation error $\cE(\cH^\alpha, \cN\cN(W,L,K),[0,1]^d)$ is small enough, then the packing numbers of $\cH^\alpha(\Lambda_N)$ and $\Phi(\Lambda_N)$ are close, and hence we can compare the lower bound (\ref{pack lower bound}) and upper bound (\ref{pack upper bound}). We will show that this leads to a contradiction when the approximation error is too small.

\begin{proof}[Proof of Theorem \ref{app lower bound norm constraint}]
Denote $\Lambda_N:=\{\Bn/N: \Bn\in \{0,\dots,N-1\}^d \}$ and $\Phi=\cN\cN(W,L,K)$ as above. We have shown (by (\ref{pack lower bound}) and (\ref{pack upper bound})) that, when $N^d\ge 8$, there exists $C_1,C_2>0$ such that the packing number
\begin{equation}\label{pack lower bound2}
\log_2 \cN_p(\cH^\alpha(\Lambda_N), \rho_2, 3C_1N^{-\alpha}) \ge N^d/8,
\end{equation}
and for any $\epsilon>0$,
\begin{equation}\label{pack upper bound2}
\log_2 \cN_p(\Phi(\Lambda_N), \rho_2, \epsilon) \le C_2 \frac{K^2 L}{\epsilon^2}.
\end{equation}

Assume the approximation error $\cE(\cH^\alpha, \Phi,[0,1]^d) < C_1N^{-\alpha}$, where $N\ge 8^{1/d}$ will be chosen later. Using (\ref{pack lower bound2}), let $\cF$ be a subset of $\cH^\alpha$ such that $\cF(\Lambda_N)$ is a $3C_1N^{-\alpha}$-packing of $\cH^\alpha(\Lambda_N)$ with $\log_2 |\cF(\Lambda_N)|\ge N^d/8$. By assumption, for any $f_i\in \cF$, there exists $g_i\in \Phi$ such that $\|f_i-g_i\|_{L^\infty} \le C_1N^{-\alpha}$. Let $\cG$ be the collection of all $g_i$. Then, $\log_2|\cG(\Lambda_N)| \ge N^d/8$ and, for any $g_i\neq g_j$ in $\cG$,
\begin{align*}
&\rho_2(g_i(\Lambda_N),g_j(\Lambda_N)) \\
\ge& \rho_2(f_i(\Lambda_N),f_j(\Lambda_N)) - \rho_2(f_i(\Lambda_N),g_i(\Lambda_N)) - \rho_2(g_j(\Lambda_N),f_j(\Lambda_N)) \\
\ge& \rho_2(f_i(\Lambda_N),f_j(\Lambda_N)) - \|f_i-g_i\|_{L^\infty} - \|g_j-f_j\|_{L^\infty} \\
>& 3C_1N^{-\alpha} - C_1N^{-\alpha} - C_1N^{-\alpha} \\
=& C_1N^{-\alpha}.
\end{align*}
In other words, $\cG(\Lambda_N)$ is a $C_1N^{-\alpha}$-packing of $\Phi(\Lambda_N)$. Combining with (\ref{pack upper bound2}), we have
\[
\frac{N^d}{8} \le \log_2 \cN_p(\Phi(\Lambda_N), \rho_2, C_1N^{-\alpha}) \le C_2 \frac{K^2 L}{C_1^2 N^{-2\alpha}},
\]
which is equivalent to
\begin{equation}\label{contradiction}
N^{d-2\alpha} \le 8C_1^{-2}C_2 K^2L.
\end{equation}
Now, we choose $N=\max\{ \lceil (9C_1^{-2}C_2 K^2L)^{1/(d-2\alpha)} \rceil, \lceil 8^{1/d} \rceil \}$, then (\ref{contradiction}) is always false. This contradiction implies $\cE(\cH^\alpha, \Phi, [0,1]^d) \ge C_1N^{-\alpha} \gtrsim (K^2L)^{-\alpha/(d-2\alpha)}$.
\end{proof}

Finally, we provide an alternative method to prove the lower bound in Theorem \ref{app lower bound norm constraint} when $\alpha=1$. We observe that, for any $h\in \Lip 1$ and $\phi \in \cN\cN(W,L,K)$, by Hahn-Banach theorem,
\[
\| h- \phi\|_{C ([0,1]^d)} = \sup_{\|T\|\neq 0} \frac{|Th- T\phi|}{\|T\|} \ge \sup_{\|T\|\neq 0} \frac{|Th|- |T\phi|}{\|T\|},
\]
where $T$ is any bounded linear functional on $C ([0,1]^d)$ with operator norm $\|T\|\neq 0$. Thus, for any nonzero linear functional $T$,
\begin{align*}
&\cE(\Lip 1, \cN\cN(W,L,K), [0,1]^d)
\ge \sup_{h\in \Lip 1} \inf_{\phi \in \cN\cN(W,L,K)}  \frac{|Th|- |T_n\phi|}{\|T\|} \\
\ge& \frac{1}{\|T\|} \left(\sup_{h\in \Lip 1} |T h| - \sup_{\phi \in \cN\cN(W,L,K)} |T\phi|\right)
= \frac{1}{\|T\|} \left(\sup_{h\in \Lip 1} T h - \sup_{\phi \in \cN\cN(W,L,K)} T\phi \right)
\end{align*}
Hence, to provide a lower bound of $\cE(\Lip 1, \cN\cN(W,L,K),[0,1]^d)$, we only need to find a linear functional $T$ that distinguishes $\Lip 1$ and $\cN\cN(W,L,K)$. In order to use the Rademacher complexity bounds for neural networks (Lemma \ref{Rademacher bound}), we will consider the functional
\begin{equation}\label{Tn}
T_n f := \frac{1}{n} \sum_{i=1}^{n} f(\Bx_i) - \int_{[0,1]^d} f(\Bx) d\Bx, \quad f\in C([0,1]^d),
\end{equation}
where the points $\Bx_1,\dots,\Bx_n \in [0,1]^d$ will be chosen appropriately. Notice that, when $\{\Bx_i\}_{i=1}^n$ are randomly chosen from the uniform distribution on $[0,1]^d$, $T_nf$ is the difference of empirical average and expectation. The optimal transport theory \cite{villani2008optimal} provides a lower bound for $\sup_{h\in \Lip 1}T_n h$, while the Rademacher complexity upper bounds $\sup_{\phi\in \cN\cN(W,L,K)}T_n \phi$ in expectation by symmetrization argument.

\begin{theorem}\label{explicit lower bound}
For any $W,L \in \bN$, $K \ge 1$ and $d\ge 3$,
\[
\cE(\Lip 1, \cN\cN(W,L,K), [0,1]^d) \ge c_d \left(K \sqrt{L+2+\log(d+1)}\right)^{-2/(d-2)},
\]
where $c_d = (d-2\sqrt{2}) 4^{-d/(d-2)} (d+1)^{-(d+1)/(d-2)}$.
\end{theorem}

\begin{proof}
Define the functional $T_n$ on $C([0,1]^d)$by (\ref{Tn}). It is easy to check that $\|T_n\| \le 2$. We have shown that
\[
\cE(\Lip 1, \cN\cN(W,L,K),[0,1]^d) \ge \frac{1}{2} \left(\sup_{h\in \Lip 1} T_n h - \sup_{\phi \in \Phi} T_n\phi\right)
\]
where we denote $\Phi=\cN\cN(W,L,K)$ to simplify the notation. Our analysis is divided into three steps.

\textbf{Step 1}: Lower bounding $\sup T_n h$. Observe that, by the Kantorovich-Rubinstein duality \cite{villani2008optimal},
\[
\sup_{h\in \Lip 1} T_n h = \cW_1 \left(\frac{1}{n} \sum_{i=1}^n \delta_{\Bx_i}, \cU \right) := \inf_{\mu} \int_{[0,1]^d \times [0,1]^d} \|\Bx-\By\|_\infty d\mu(\Bx,\By)
\]
is the $1$-Wasserstein distance (see (\ref{W_p})) between the discrete distribution $\frac{1}{n} \sum_{i=1}^n \delta_{\Bx_i}$ and the uniform distribution $\cU$ on $[0,1]^d$, where the infimum is taken over all joint probability distribution (also called coupling) $\mu$ on $[0,1]^d \times [0,1]^d$, whose marginal distributions are $\frac{1}{n} \sum_{i=1}^n \delta_{\Bx_i}$ and $\cU$ respectively.

We notice that, for any $r\in [0,1/2]$,
\begin{align*}
&\cU \left( \left\{ \By\in [0,1]^d: \min_{1\le i\le n} \|\Bx_i-\By\|_\infty \ge rn^{-1/d} \right\} \right) \\
=& 1- \cU \left( \left\{ \By\in [0,1]^d: \min_{1\le i\le n} \|\Bx_i-\By\|_\infty < rn^{-1/d} \right\} \right) \\
\ge & 1 - \sum_{i=1}^n \cU \left( \left\{ \By\in [0,1]^d: \|\Bx_i-\By\|_\infty < rn^{-1/d} \right\} \right) \\
=& 1 - n (2rn^{-1/d})^d = 1 - 2^dr^d.
\end{align*}
Hence, for any coupling $\mu$ and $r\in [0,1/2]$,
\begin{align*}
\int_{[0,1]^d \times [0,1]^d} \|\Bx-\By\|_\infty d\mu(\Bx,\By)
=& \int_{ \cup_{i=1}^n\{\Bx_i \} \times [0,1]^d} \|\Bx-\By\|_\infty d\mu(\Bx,\By) \\
\ge& \int_{ \cup_{i=1}^n\{\Bx_i \} \times [0,1]^d} \min_{1\le i\le n}\|\Bx_i-\By\|_\infty d\mu(\Bx,\By) \\
=& \int_{[0,1]^d} \min_{1\le i\le n}\|\Bx_i-\By\|_\infty d \cU(\By) \\
\ge& (1-2^dr^d) rn^{-1/d}.
\end{align*}
As a consequence, for any $n$ points $\Bx_1,\dots,\Bx_n \in [0,1/2]^d$,
\[
\sup_{h\in \Lip 1} T_n h \ge \sup_{r\in [0,1/2]} (1-2^dr^d) rn^{-1/d} = 2^{-1}d(d+1)^{-1-1/d} n^{-1/d},
\]
where the supremum is attained when $r=2^{-1}(d+1)^{-1/d}$.

\textbf{Step 2}: Upper bounding $\sup T_n \phi$. Let $X_{1:n} = \{X_i\}_{i=1}^n$ be $n$ i.i.d. samples from the uniform distribution $\cU$ on $[0,1]^d$. Denote the empirical distribution by $\widehat{\cU}_n = \frac{1}{n} \sum_{i=1}^n \delta_{X_i}$. We observe that
\begin{align*}
\bE_{X_{1:n}} d_\Phi(\cU,\widehat{\cU}_n) 
&= \bE_{X_{1:n}} \left[ \sup_{\phi\in \Phi} \frac{1}{n} \sum_{i=1}^{n} \phi(X_i) - \bE_{X\sim \cU} [\phi(X)] \right] \\
&= \bE_{X_{1:n}} \left[ \sup_{\phi\in \Phi} \frac{1}{n} \sum_{i=1}^{n} \phi(X_i) - \int_{[0,1]^d} \phi(\Bx) d\Bx \right].
\end{align*}
Symmetrization argument (Lemma \ref{symmetrization}) shows that
\[
\bE_{X_{1:n}} d_\Phi(\cU,\widehat{\cU}_n) \le 2 \bE_{X_{1:n}} \left[ \cR_n(\Phi(X_{1:n})) \right],
\]
where we denote $\Phi(X_{1:n}) := \{(\phi(X_1),\dots,\phi(X_n))\in \bR^n: \phi \in \Phi\}$. 

By Lemma \ref{Rademacher bound}, for any $X_{1:n} \subseteq [0,1]^d$,
\[
\cR_n(\Phi(X_{1:n})) \le 2K \sqrt{L+2+\log(d+1)} n^{-1/2}.
\]
Hence, there exists $\Bx_1,\dots,\Bx_n \in [0,1]^d$ such that
\[
\sup_{\phi\in \cN\cN(W,L,K)} T_n \phi \le \bE_{X_{1:n}} d_\Phi(\cU,\widehat{\cU}_n) \le 4K \sqrt{L+2+\log(d+1)} n^{-1/2}.
\]

\textbf{Step 3}: Optimizing $n$. We have shown that there exists $T_n$ such that
\begin{align*}
\cE(\Lip 1, \cN\cN(W,L,K), [0,1]^d) \ge \frac{1}{2} \left(\sup_{h\in \Lip 1} T_n h - \sup_{\phi \in \Phi} T_n\phi\right) \ge ds^{-1} n^{-1/d} - 2t n^{-1/2},
\end{align*}
where $s = 4(d+1)^{1+1/d}$ and $t = K \sqrt{L+2+\log(d+1)}$. In order to optimize over $n$, we can choose
\[
n = \left\lfloor (s t)^{\frac{2d}{d-2}} \right\rfloor.
\]
Then, since $s t\ge 2$, we have $n \ge (s t)^{\frac{2d}{d-2}} -1 \ge \frac{1}{2} (s t)^{\frac{2d}{d-2}}$ and
\begin{align*}
\cE(\Lip 1, \cN\cN(W,L,K), [0,1]^d) &\ge d s^{-1} (s t)^{-2/(d-2)} - 2\sqrt{2}t(s t)^{-d/(d-2)} \\
&= (d-2\sqrt{2}) s^{-d/(d-2)} t^{-2/(d-2)} \\
&= c_d \left(K \sqrt{L+2+\log(d+1)}\right)^{-2/(d-2)},
\end{align*}
where $c_d = (d-2\sqrt{2}) 4^{-d/(d-2)} (d+1)^{-(d+1)/(d-2)}$.
\end{proof}

\chapter{Distribution Approximation by Generative Networks}\label{chapter: dis app}

This chapter studies the expressive power of ReLU neural networks for generating distributions. Specifically, for a low-dimensional probability distribution $\nu$ on $\bR^k$, we consider how well a high-dimensional probability distribution $\mu$ defined on $\bR^d$ can be approximated by the push-forward distribution $\phi_\# \nu$, using the neural network $\phi: \bR^k\to \bR^d$ as a transportation map. To quantify the approximation error, we consider three typical types of metrics (discrepancies) used in generative models: 

\begin{itemize}[parsep=0pt]
\item For $p\in [1,\infty)$, the $p$-th Wasserstein distance (with respect to $\|\cdot\|_\infty$) between two probability measures on $\bR^d$ is the optimal transportation cost defined as
\begin{equation}\label{W_p}
\cW_p(\mu,\gamma) := \inf_{\tau\in \prod(\mu,\gamma)} (\bE_{\tau(\Bx,\By)} \|\Bx-\By\|_\infty^p)^{1/p},
\end{equation}
where $\prod(\mu,\gamma)$ denotes the set of all joint probability distributions $\tau(\Bx,\By)$ whose marginals are respectively $\mu$ and $\gamma$. A distribution $\tau\in \prod(\mu,\gamma)$ is called a coupling of $\mu$ and $\gamma$. There always exists an optimal coupling that achieves the infimum \cite{villani2008optimal}. The Kantorovich-Rubinstein duality gives an alternative definition of $\cW_1$:
\[
\cW_1(\mu,\gamma) = \sup_{\Lip(f) \le 1} \bE_{\mu} [f] - \bE_{\gamma} [f].
\]
This duality is used in Wasserstein GAN \cite{arjovsky2017wasserstein} to estimate the distance between the target and generated distributions. More generally, $\cW_p$ can be estimated by certain Besov norms of negative smoothness under some conditions \cite{weed2019estimation}.
	
\item Let $\cH_K$ be a reproducing kernel Hilbert space (RKHS) with kernel $K:\bR^d \times \bR^d \to \bR$ \cite{aronszajn1950theory,berlinet2011reproducing}. The maximum mean discrepancy (MMD) between two probability distributions on $\bR^d$ is defined by \cite{gretton2012kernel}:
\begin{equation}\label{MMD}
\MMD(\mu,\gamma) := \sup_{\|f\|_{\cH_K}\le 1} \bE_{\mu} [f] - \bE_{\gamma} [f].
\end{equation}
Note that MMD and the Wasserstein distance $\cW_1$ are special cases of integral probability metrics, see (\ref{IPM}).
	
\item The $f$-divergences can be defined for all convex functions $f:(0,\infty) \to \bR$ with $f(1)=0$ as follows: Given two probability distributions $\mu,\gamma$ that are absolutely continuous with respect to some probability measure $\tau$, let their Radon-Nikodym derivatives be $p=d\mu/d\tau$ and $q=d\gamma/d\tau$. Then the $f$-divergence of $\mu$ from $\gamma$ is defined as
\begin{equation}\label{f-div}
\cD_f(\mu \| \gamma) := \int_{\bR^d} f\left(\frac{p(\Bx)}{q(\Bx)} \right) q(\Bx) d\tau(\Bx) = \int_{q>0} f\left(\frac{p(\Bx)}{q(\Bx)} \right) q(\Bx) d\tau(\Bx) +  f^*(0)\mu(q=0),
\end{equation}
where we denote $f(0):=\lim_{t\downarrow 0}f(t)$, $f^*(0):=\lim_{t\to \infty} f(t)/t$ and we adopt the convention that $f(\frac{p(\Bx)}{q(\Bx)}) q(\Bx)=0$ if $p(\Bx)=q(\Bx)=0$, and $f(\frac{p(\Bx)}{q(\Bx)}) q(\Bx)= f^*(0)p(\Bx)$ if $q(\Bx)=0$ and $p(\Bx)\neq 0$. It can be shown that the definition is independent of the choice of $\tau$ and hence we can always choose $\tau = (\mu+\gamma)/2$. 
\end{itemize}

Let $\cD$ denote the metrics introduced above, then our goal is to estimate the quantity
\[
\cD(\mu,\cN\cN(W,L)_\# \nu):= \inf_{\phi \in \cN\cN(W,L)} \cD(\mu, \phi_\# \nu),
\]
where $\nu$ is an absolutely continuous probability distribution.

\section{Approximation in Wasserstein distances}\label{sec: app W_p}

The basic idea of our approach is depicted as follows. In order to bound the approximation error $\cW_p(\mu,\cN\cN(W,L)_\#\nu)$, we first approximate the target distribution $\mu$ by a discrete probability measure $\gamma$, and then construct a neural network $\phi$ such that the push-forward measure $\phi_\# \nu$ is close to the discrete distribution $\gamma$. By the triangle inequality for Wasserstein distances, one has
\begin{align*}
\cW_p(\mu,\cN\cN(W,L)_\#\nu) 
\le& \cW_p(\mu,\gamma) + \cW_p(\gamma,\cN\cN(W,L)_\#\nu) \\
\le&  \cW_p(\mu,\gamma) + \sup_{\tau\in \cP(n)} \cW_p(\tau,\cN\cN(W,L)_\#\nu),
\end{align*}
where $\gamma \in \cP(n)$ and $\cP(n)$ is the set of all discrete probability measures supported on at most $n$ points, that is,
\[
\cP(n) := \left\{ \sum _{i=1}^n p_i\delta_{\Bx_i}: \sum_{i=1}^n p_i=1, p_i\ge 0, \Bx_i\in \bR^d \right\}.
\]
Taking the infimum over all $\gamma \in \cP(n)$, we get
\begin{equation}\label{triangle inequality}
\cW_p(\mu,\cN\cN(W,L)_\#\nu)
\le \cW_p(\mu,\cP(n)) + \sup_{\tau\in \cP(n)} \cW_p(\tau,\cN\cN(W,L)_\#\nu),
\end{equation}
where $\cW_p(\mu,\cP(n)):=\inf_{\gamma\in \cP(n)}\cW_p(\mu,\gamma)$ measures the distance between $\mu$ and discrete distributions in $\cP(n)$. The next lemma shows that the second term vanishes as long as the width $W$ and depth $L$ of the neural network in use are sufficiently large.

\begin{lemma}\label{app discrete measure}
Suppose that $W\ge 7d+1$, $L\ge 2$ and $p\in [1,\infty)$. Let $\nu$ be an absolutely continuous probability distribution on $\bR$. If $n\le \frac{W-d-1}{2} \lfloor \frac{W-d-1}{6d} \rfloor \lfloor \frac{L}{2} \rfloor +2$, then for any $\mu\in \cP(n)$,
\[
\cW_p(\mu,\cN\cN(W,L)_\#\nu) =0.
\]
\end{lemma}
\begin{proof}
Without loss of generality, we assume that $m:=n-1\ge 1$ and $\mu=\sum_{i=0}^{m} p_i \delta_{\Bx_i}$ with $p_i>0$ for all $0\le i\le m$. For any $\epsilon$ that satisfies $0<\epsilon< (mp_i)^{1/p}\|\Bx_i-\Bx_{i-1}\|_\infty$ for all $i=1,\dots,m$, we are going to construct a neural network $\phi \in \cN\cN(W,L)$ such that $\cW_p(\mu,\phi_\# \nu) \le \epsilon$.

By the absolute continuity of $\nu$, we can choose $2m$ points
\[
z_{1/2} < z_1 < z_{3/2} < \dots < z_{m-1/2}< z_m
\]
such that
\begin{align*}
&\nu((-\infty,z_{1/2})) = p_0, \\
&\nu((z_{i-1/2},z_{i})) = \frac{\epsilon^p}{m\|\Bx_i-\Bx_{i-1}\|_\infty^p}, &&1\le i\le m,\\
&\nu((z_i,z_{i+1/2})) = p_i -\frac{\epsilon^p}{m\|\Bx_i-\Bx_{i-1}\|_\infty^p}, &&1\le i\le m-1,\\
&\nu((z_m,\infty)) = p_m -\frac{\epsilon^p}{m\|\Bx_m-\Bx_{m-1}\|_\infty^p}.
\end{align*}
We define the continuous piecewise linear function $\phi :\bR \to \bR^d$ by
\[
\phi(z) := 
\begin{cases}
\Bx_0  &z \in (-\infty, z_{1/2}), \\
\frac{z_i-z}{z_i-z_{i-1/2}}\Bx_{i-1} + \frac{z-z_{i-1/2}}{z_i-z_{i-1/2}}\Bx_{i} &z\in [z_{i-1/2},z_i), \quad 1\le i\le m, \\
\Bx_i  & z\in [z_{i},z_{i+1/2}), \quad 1\le i\le m-1, \\
\Bx_m  &z\in [z_m,\infty).
\end{cases}
\]
Since $\phi \in \cS^d(z_{1/2},\dots, z_m)$ has $2m= 2n-2\le (W-d-1) \lfloor \frac{W-d-1}{6d} \rfloor \lfloor \frac{L}{2} \rfloor+2$ breakpoints, by Lemma \ref{CPwL}, $\phi \in \cN\cN(W,L)$.

In order to estimate $\cW_p(\mu,\phi_\# \nu)$, let us denote the line segment joining $\Bx_{i-1}$ and $\Bx_i$ by $\cL_i :=\{(1-t)\Bx_{i-1} + t\Bx_i\in \bR^d: 0<t\le 1 \}$. Then $\phi_\#\nu$ is supported on $\cup_{i=1}^m \cL_i \cup \{\Bx_0\}$ and $\phi_\#\nu(\{\Bx_0\}) = p_0$, $\phi_\# \nu(\{\Bx_i\}) =p_i-\frac{\epsilon^p}{m\|\Bx_i-\Bx_{i-1}\|_\infty^p}$, $\phi_\# \nu(\cL_i) =p_i$ for $i=1,\dots,m$. By considering the sum of product measures
\[
\gamma = \delta_{\Bx_0} \times \phi_\#\nu|_{\{\Bx_0\}} + \sum_{i=1}^{m} \delta_{\Bx_i} \times \phi_\#\nu|_{L_i},
\]
which is a coupling of $\mu$ and $\phi_\#\nu$, we have
\begin{align*}
\cW_p^p(\mu,\phi_\#\nu) &\le \int_{\bR^d \times \bR^d} \|\Bx-\By\|_\infty^p d\gamma(\Bx,\By) \\
&= \sum_{i=1}^m \int_{\cL_i \setminus \{\Bx_i\}} \|\Bx_i -\By\|_\infty^p d\phi_\#\nu(\By) \\
&\le \sum_{i=1}^m \|\Bx_i-\Bx_{i-1}\|_\infty^p \phi_\#\nu(\cL_i \setminus \{\Bx_i\}) \\
&= \epsilon^p.
\end{align*}
Letting $\epsilon\to0$, we have $\cW_p^p(\mu,\cN\cN(W,L)_\#\nu)=0$, which completes the proof.
\end{proof}

As a consequence of the triangle inequality (\ref{triangle inequality}) and Lemma \ref{app discrete measure}, our approximation problem is reduced to the estimation of the approximation error $\cW_p(\mu,\cP(n))$. We study the case that the target distribution $\mu$ has finite absolute $q$-moment
\[
M_q(\mu) := \left( \int_{\bR^d} \|\Bx\|_\infty^q d\mu(\Bx) \right)^{1/q} <\infty, \quad 1\le q<\infty.
\]

\begin{theorem}\label{app finite moment}
Let $p\in [1,\infty)$ and $\nu$ be an absolutely continuous probability distribution on $\bR$. Assume that $\mu$ is a probability distribution on $\bR^d$ with finite absolute $q$-moment $M_q(\mu)<\infty$ for some $q>p$. Then, for any $W\ge 7d+1$ and $L\ge 2$,
\[
\cW_p(\mu,\cN\cN(W,L)_\# \nu) \le C (M_q^q(\mu)+1)^{1/p}
\begin{cases}
(W^2L)^{-1/d}, &q>p+p/d,  \\
(W^2L)^{-1/d}(\log_2 W^2L)^{1/d}, &p<q\le p+p/d,
\end{cases}
\]
where $C$ is a constant depending only on $p$, $q$ and $d$.
\end{theorem}

Note that the number of parameters of a neural network with width $W$ and depth $L$ is $U \asymp W^2L$ when $L\ge 2$, hence the theorem upper bounds the approximation error by the number of parameters. Although we restrict the source distribution $\nu$ to be one-dimensional, the result can be easily generalized to absolutely continuous distributions on $\bR^k$ such as multivariate Gaussian and uniform distributions. It can be done simply by projecting these distributions to one-dimensional distributions using linear mappings (the projection can be realized on the first layer of neural network). An interesting consequence of Theorem \ref{app finite moment} is that we can approximate high-dimensional distributions by low-dimensional distributions if we use neural networks as transport maps.

In generative adversarial network, the Wasserstein distance $\cW_1(\mu,\phi_\# \nu)$ is estimated by a discriminator parameterized by a neural network $\cF$:
\[
d_\cF(\mu,\phi_\# \nu) := \sup_{f\in \cF} \bE_{\mu} [f] - \bE_{\phi_\# \nu} [f].
\]
The discriminative network $\cF$ is often regularized (by weight clipping or other methods) so that the Lipschitz constant of any $f\in \cF$ is bounded by some constant $K>0$. For such network, we have
\[
d_\cF(\mu,\phi_\# \nu) \le K \cW_1(\mu,\phi_\# \nu).
\]
Hence, Theorem \ref{app finite moment} also gives upper bounds on the neural network distance $d_\cF$ used in Wasserstein GANs.

To prove Theorem \ref{app finite moment}, we will need the following lemma to estimate the Wasserstein distances of two distributions.

\begin{lemma}\label{Wp estimate method}
If two probability measures $\mu$ and $\gamma$ on $\cX\subseteq \bR^d$ can be decomposed into non-negative measures as $\mu = \sum_{j\ge 1} \mu_j$ and $\gamma = \sum_{j\ge 1} \gamma_j$ such that $\mu_j(\cX) = \gamma_j(\cX)$ for all $j\ge 1$, then
\[
\cW_p^p(\mu,\gamma) \le \sum_{j\ge 1} \mu_j(X) \cW_p^p \left(\frac{\mu_j}{\mu_j(X)}, \frac{\gamma_j}{\gamma_j(X)} \right).
\]
In particular, if the support of $\mu$ can be covered by $n$ balls $B(\Bx_j,r)=\{\By\in\bR^d: \|\Bx_j-\By\|_\infty \le r\}$, $j=1,\dots,n$, then there exists $c_j\ge 0$ such that $\sum_{j=1}^n c_j =1$ and 
\[
\cW_p \left(\mu, \sum_{j=1}^{n} c_j \delta_{\Bx_j} \right) \le r.
\]
\end{lemma}
\begin{proof}
Let $\tau_j$ be the optimal coupling of $\frac{\mu_j}{\mu_j(X)}$ and $\frac{\gamma_j}{\gamma_j(X)}$, then it is easy to check that
\[
\tau = \sum_{j\ge 1} \mu_j(\cX) \tau_j
\]
is a coupling of $\mu$ and $\gamma$. Hence,
\begin{align*}
\cW_p^p(\mu,\gamma) &\le \int_\cX\int_\cX \|\Bx-\By\|_\infty^p d\tau(\Bx,\By) \\
&= \sum_{j\ge 1} \mu_j(\cX) \int_\cX\int_\cX \|\Bx-\By\|_\infty^p d\tau_j(\Bx,\By) \\
&= \sum_{j\ge 1} \mu_j(\cX) \cW_p^p \left(\frac{\mu_j}{\mu_j(\cX)}, \frac{\gamma_j}{\gamma_j(\cX)} \right).
\end{align*}

For the second part of the lemma, let $A$ be the support of $\mu$, denote $A_1 := A \cap B(\Bx_1,r)$ and $A_{j+1} := A\cap B(\Bx_{j+1},r) \setminus (\cup_{i=1}^{j} A_i)$, then $\{A_j:j=1,\dots,n \}$ is a partition of $A$. This partition induces a decomposition of $\mu = \sum_{j=1}^n \mu|_{A_j}$.

Let $c_j = \mu(A_j)$, then $\sum_{j=1}^n c_j =1$ and if $c_j\neq 0$,
\begin{align*}
\cW_p^p (c_j^{-1} \mu|_{A_j}, \delta_{\Bx_j}) \le c_j^{-1} \int\int \|\Bx-\By\|_\infty^p d\delta_{\Bx_j}(\Bx) d\mu|_{A_j}(\By) 
\le r^p.
\end{align*}
By the first part of the lemma, we have
\begin{align*}
\cW_p^p \left(\mu, \sum_{j=1}^{n} c_j \delta_{\Bx_j} \right) \le \sum_{j=1}^n c_j \cW_p^p (c_j^{-1} \mu|_{A_j}, \delta_{\Bx_j}) 
\le \sum_{j=1}^n c_j r^p = r^p,
\end{align*}
which completes the proof.
\end{proof}

Now, using Lemma \ref{Wp estimate method}, we can give upper bounds of the approximation error $\cW_p(\mu,\cP(n))$ for distribution $\mu$ with finite moment.

\begin{theorem}\label{app finite moment by discrete}
Let $\mu$ be a probability distribution on $\bR^d$ with finite absolute $q$-moment $M_q(\mu)<\infty$ for some $q>p\ge 1$. Then for any $n\in \bN$, 
\[
\cW_p(\mu,\cP(n)) \le C_{p,q,d}(M_q^q(\mu)+1)^{1/p}
\begin{cases}
n^{-1/d}, &q>p+p/d,  \\
(n/\log_2 n)^{-1/d}, &p<q\le p+p/d,
\end{cases}
\]
where $C_{p,q,d}$ is a constant depending only on $p$, $q$ and $d$.
\end{theorem}

\begin{proof}
Let $B_0 = \{\Bx\in \bR^d: \|\Bx\|_\infty\le 1 \}$ and $B_j = (2^jB_0)\setminus (2^{j-1}B_{0})$ for $j\in \bN$, then $\{B_j\}_{j\ge 0}$ is a partition of $\bR^d$. For any $k\ge 0$, we denote $E_k:=\bR^d \setminus \cup_{j=0}^k B_j$. Let $\mu_j := \mu(B_j)^{-1} \mu|_{B_j}$ and $\widetilde{\mu}_k:= \mu(E_k)^{-1} \mu|_{E_k}$, then for each $k\ge 0$, we can decompose $\mu$ as 
\[
\mu = \sum_{j=0}^k \mu(B_j) \mu_j + \mu(E_k) \widetilde{\mu}_k.
\]
By Markov's inequality, we have
\[
\mu(B_j) \le \mu(\|\Bx\|_\infty> 2^{j-1}) \le M_q^q(\mu)2^{-q(j-1)}, \quad j\ge 1.
\]
Furthermore, if $\mu(E_k)\neq 0$, 
\begin{align*}
\cW_p^p(\widetilde{\mu}_k, \delta_{\boldsymbol{0}}) &\le \mu(E_k)^{-1} \int_{\|\Bx\|_\infty>2^k} \|\Bx\|_\infty^p d\mu(\Bx) \\
&\le \mu(E_k)^{-1} \int_{\|\Bx\|_\infty>2^k} \|\Bx\|_\infty^p \frac{\|\Bx\|_\infty^{q-p}}{2^{k(q-p)}} d\mu(\Bx)\\
&\le \mu(E_k)^{-1} M_q^q(\mu) 2^{-k(q-p)}.
\end{align*}

Observe that the ball $\{\Bx:\|\Bx\|_\infty\le 2^j\}$ can be covered by at most $C r^{-d}$ balls with radius $2^jr_j$ of the form $B(\By,2^j r)=\{\Bx\in\bR^d: \|\Bx-\By\|_\infty \le 2^jr\}$ for some constant $C>0$. Let $r_0,\dots,r_k$ be $k+1$ positive numbers, then each $B_j \subseteq \{\Bx:\|\Bx\|_\infty\le 2^j\}$ can be covered by at most $n_j := \lceil C r_j^{-d} \rceil$ balls with radius $2^jr_j$. We denote the collection of the centers of these balls by $\Lambda_j$. By Lemma \ref{Wp estimate method}, there exists a probability measure $\gamma_j$ of the form
\[
\gamma_j = \sum_{\Bx\in \Lambda_j} c_j(\Bx) \delta_\Bx
\]
such that $\cW_p(\mu_j,\gamma_j) \le 2^j r_j$.

We consider the probability distribution
\[
\gamma = \sum_{j=0}^k \mu(B_j) \gamma_j + \mu(E_k) \delta_{\boldsymbol{0}} \in \cP\left(1+\sum_{j=0}^k n_j \right).
\]
By Lemma \ref{Wp estimate method}, we have
\begin{align*}
\cW_p^p(\mu,\gamma) \le&\sum_{j=0}^k \mu(B_j) \cW_p^p(\mu_j,\gamma_j) + \mu(E_k) \cW_p^p(\widetilde{\mu}_k, \delta_{\boldsymbol{0}}) \\
\le& r_0^p + M_q^q(\mu) \sum_{j=1}^k 2^{q-j(q-p)} r_j^p + M_q^q(\mu)2^{-k(q-p)}.
\end{align*}

Finally, if $q> p+p/d$, we choose $k = \lfloor \log_2 n \rfloor -1$ and $r_j = C^{1/d} 2^{(j-k)/d}$ for $0\le j\le k$. Then, $1+ \sum_{j=0}^k n_j = 1+ \sum_{j=0}^{k} 2^{k-j} =2^{k+1} \le n$, which implies $\gamma \in \cP(n)$, and
\begin{align*}
\cW_p^p(\mu,\gamma) &\le M_q^q(\mu) 2^{-k(q-p)} + C_d^{p/d} 2^{-pk/d} + 2^q C_d^{p/d} M_q^q(\mu) 2^{-pk/d} \sum_{j=1}^k 2^{-j(q-p-p/d)} \\
&\le \left(M_q^q(\mu) + C_d^{p/d} + C_d^{p/d} M_q^q(\mu) \frac{2^q}{2^{q-p-p/d}-1} \right) 2^{-pk/d} \\
&\le C_{p,q,d}^p (M_q^q(\mu)+1) n^{-p/d}.
\end{align*} 
If $p<q\le p+p/d$, we choose $k=\lceil \frac{p}{d(q-p)} \log_2 n \rceil$ and $r_j = C^{1/d} m^{-1/d}$ for $0\le j\le k$, where $m:= \lfloor \frac{n-1}{k+1} \rfloor$ . Then we have $1+ \sum_{j=0}^k n_j = 1+ (k+1)\lfloor \frac{n-1}{k+1} \rfloor \le n$, which implies $\gamma \in \cP(n)$, and
\begin{align*}
\cW_p^p(\mu,\gamma) &\le M_q^q(\mu) 2^{-k(q-p)} + C_d^{p/d} m^{-p/d} + 2^q C_d^{p/d} M_q^q(\mu) m^{-p/d} \sum_{j=1}^k 2^{-j(q-p)} \\
&\le M_q^q(\mu) n^{-p/d} + C_d^{p/d}\left(1+ M_q^q(\mu) \frac{2^q}{2^{q-p}-1} \right) m^{-p/d} \\
&\le C_{p,q,d}^p (M_q^q(\mu)+1) (n/\log_2 n)^{-p/d}. \qedhere
\end{align*}
\end{proof}

\begin{remark}
The expected Wasserstein distance $\bE \cW_p(\mu,\widehat{\mu}_n)$ between a probability distribution $\mu$ and its empirical distribution $\widehat{\mu}_n$ has been studied extensively in statistics literature \cite{fournier2015rate,bobkov2019one,weed2019sharp,lei2020convergence}. It was shown in \cite{lei2020convergence} that, if $M_q(\mu)<\infty$, the convergence rate of $\bE \cW_p(\mu,\widehat{\mu}_n)$ is $n^{-s(p,q,d)}$ with $s(p,q,d) = \min\{ \frac{1}{d}, \frac{1}{2p}, \frac{1}{p}-\frac{1}{q}\}$, ignoring the logarithm factors. Since $\widehat{\mu}_n\in \cP(n)$, it is easy to see that $\cW_p(\mu,\cP(n))\le \bE \cW_p(\mu, \widehat{\mu}_n)$. In Theorem \ref{app finite moment by discrete}, we construct a discrete measure that achieves the order $1/d\ge s(p,q,d)$, which is slightly better than the empirical measure in some situations.
\end{remark}

By the triangle inequality (\ref{triangle inequality}), we can use Theorem \ref{app finite moment by discrete} and Lemma \ref{app discrete measure} to prove our main approximation bound in Theorem \ref{app finite moment}.

\begin{proof}[Proof of Theorem \ref{app finite moment}]
Inequality (\ref{triangle inequality}) says that, for any $n$,
\begin{align*}
\cW_p(\mu,\cN\cN(W,L)_\#\nu)
\le \cW_p(\mu,\cP(n)) + \sup_{\tau\in \cP(n)} \cW_p(\tau,\cN\cN(W,L)_\#\nu).
\end{align*}
If we choose $n= \frac{W-d-1}{2} \lfloor \frac{W-d-1}{6d} \rfloor \lfloor \frac{L}{2} \rfloor +2$, Lemma \ref{app discrete measure} implies that $\cW_p(\mu,\cN\cN(W,L)_\#\nu) \le \cW_p(\mu,\cP(n))$. By Theorem \ref{app finite moment by discrete}, it can be bounded by
\[
C (M_q^q(\mu)+1)^{1/p}
\begin{cases}
n^{-1/d}, &q>p+p/d  \\
(n/\log_2 n)^{-1/d}, &p<q\le p+p/d
\end{cases}
\]
for some constant $C>0$ depending only on $p$, $q$ and $d$.

Since $W\ge 7d+1$ and $L\ge 2$, a simple calculation shows $W/2\le W-d-1\le W$ and $L/4\le \lfloor L/2 \rfloor \le L/2$, which implies $cW^2L/d \le n \le CW^2L/d$ with $c=1/384$ and $C=1/12$. Hence, $n^{-1/d}\le c^{-1}d^{1/d}(W^2L)^{-1/d} \le 2c^{-1}(W^2L)^{-1/d}$ and $(\log_2 n)^{1/d} \le (\log_2 W^2L)^{1/d}$, which gives us the desired bounds.
\end{proof}

\section{Bounds with intrinsic dimension}

Theorem \ref{app finite moment} essentially shows the approximation error $\cW_p(\mu,\cN\cN(W,L)_\# \nu)$ can be bounded by $(W^2L)^{-1/d}$. Notice that the ambient dimension of $\mu$ is often very large in practical applications and this bound suffers from the curse of dimensionality. However, in practice, the target distribution $\mu$ usually has certain low-dimensional structure, which can help us lessen the curse of dimensionality. To utilize this kind of structures, we introduce a notion of dimension of a measure using the concept of covering number.

\begin{definition}\label{dimension}
For a probability measure $\mu$ on $\bR^d$, the $(\epsilon,\delta)$-covering number (with respect to $\|\cdot\|_\infty$) of $\mu$ is defined as
\[
\cN_c(\mu,\epsilon,\delta) := \inf \{\cN_c(S,\|\cdot\|_\infty,\epsilon): \mu(S)\ge 1-\delta \}.
\]
For $1\le p<\infty$, we define the upper and lower dimensions of $\mu$ as
\begin{align*}
s^*_p(\mu) &:= \limsup_{\epsilon \to 0} \frac{\log_2 \cN_c(\mu,\epsilon,\epsilon^p)}{-\log_2 \epsilon}, \\
s_*(\mu) &:= \lim_{\delta\to 0} \liminf_{\epsilon \to 0} \frac{\log_2 \cN_c(\mu,\epsilon,\delta)}{-\log_2 \epsilon}.
\end{align*}
\end{definition}

We make several remarks on the definition. Since $\cN_c(\mu,\epsilon,\delta)$ increases as $\delta$ decreases, the limit in the definition of lower dimension always exists. The lower dimension $s_*(\mu)$ is the same as the so-called lower Wasserstein dimension in \cite{weed2019sharp}, which was also introduced by \cite{young1982dimension} in dynamical systems. But our upper dimension $s^*_p(\mu)$ is different from the upper Wasserstein dimension in \cite{weed2019sharp}. More precisely, our upper dimension is slightly smaller than the upper Wasserstein dimension in some cases.

To make it easier to interpret our results, we note that $s_*(\mu)$ and $s_p^*(\mu)$ can be bounded from below and above by the well known Hausdorff dimension and Minkowski dimension respectively (see \cite{falconer1997techniques,falconer2004fractal} for instance).

\begin{definition}[Hausdorff and Minkowski dimensions]\label{Hausdorff Minkowski dimensions}
The $\alpha$-Hausdorff measure of a set $S\subseteq \bR^d$ is defined as
\[
H^\alpha(S):= \lim_{\epsilon\to 0} \inf \left\{ \sum_{j=1}^{\infty} (2r_j)^\alpha: S\subseteq \cup_{j=1}^\infty B(\Bx_j,r_j), r_j\le \epsilon \right\},
\]
where $B(\Bx,r)=\{\By\in \bR^d:\|\Bx-\By\|_\infty \le r \}$ is the ball with center $\Bx$ and radius $r$, and the Hausdorff dimension of $S$ is
\[
\dim_H(S) := \inf \{ \alpha: H^\alpha(S) =0 \}.
\]
The upper and the lower Minkowski dimension of $S$ is
\begin{align*}
\overline{\dim}_M(S) := \limsup_{\epsilon\to 0} \frac{\log_2 \cN_c(S,\|\cdot\|_\infty,\epsilon)}{-\log_2 \epsilon}, \\
\underline{\dim}_M(S) := \liminf_{\epsilon\to 0} \frac{\log_2 \cN_c(S,\|\cdot\|_\infty,\epsilon)}{-\log_2 \epsilon}.
\end{align*}
If $\overline{\dim}_M(S) = \underline{\dim}_M(S) = \dim_M(S)$, then $\dim_M(S)$ is called the Minkowski dimension of $S$.
The Hausdorff and (upper) Minkowski dimensions of a measure $\mu$ on $\bR^d$ are defined respectively by
\begin{align*}
\dim_H(\mu) &:= \inf \{ \dim_H(S): \mu(S)=1 \}, \\
\dim_M(\mu) &:= \inf \{ \overline{\dim}_M(S): \mu(S)=1 \}.
\end{align*}
\end{definition}

\begin{proposition}\label{dimension relation} 
For any $1\le p<q<\infty$,
\[
\dim_H(\mu) \le s_*(\mu) \le s_p^*(\mu) \le s_q^*(\mu) \le \dim_M(\mu).
\]
\end{proposition}
\begin{proof}
We first prove $s_*(\mu) \le s_p^*(\mu)\le s_q^*(\mu)$. Since $\cN_c(\mu,\epsilon,\delta)$ increases as $\delta$ decreases, for fixed $\delta$ and sufficiently small $\epsilon$, we have
\[
\frac{\log_2 \cN_c(\mu,\epsilon,\delta)}{-\log_2 \epsilon} \le \frac{\log_2 \cN_c(\mu,\epsilon,\epsilon^p)}{-\log_2 \epsilon} \le \frac{\log_2 \cN_c(\mu,\epsilon,\epsilon^q)}{-\log_2 \epsilon}.
\]
Taking limit $\epsilon \to 0$, we obtain
\[
\liminf_{\epsilon \to 0} \frac{\log_2 \cN_c(\mu,\epsilon,\delta)}{-\log_2 \epsilon} \le s_p^*(\mu)\le s_q^*(\mu).
\]
Taking limit $\delta \to 0$ shows $s_*(\mu) \le s_p^*(\mu)\le s_q^*(\mu)$.

For the inequality $s_q^*(\mu) \le \dim_M(\mu)$, we observe that for any $q$ and any $S$ with $\mu(S)=1$, 
\[
\cN_c(\mu,\epsilon,\epsilon^q) \le \cN_c(\mu,\epsilon,0) \le \cN_c(S,\|\cdot\|_\infty,\epsilon).
\]
A straightforward application of the definitions implies that $s_p^*(\mu) \le \dim_M(\mu)$.

For the inequality $\dim_H(\mu) \le s_*(\mu)$, we follow the idea in \cite{weed2019sharp}. By \cite[Proposition 10.3]{falconer1997techniques}, the Hausdorff dimension of $\mu$ can be expressed as
\[
\dim_H(\mu) = \inf \left\{ s\in \bR : \liminf_{r\to 0} \frac{\log_2 \mu(B(\Bx,r))}{\log_2 r} \le s \mbox{ for } \mu \mbox{-a.e. } \Bx \right\}.
\]
This implies for any $s< \dim_H(\mu)$ that
\[
\mu\left( \left\{ \Bx: \exists r_x>0, \forall r\le r_\Bx, \mu(B(\Bx,r))\le r^s \right\} \right) \ge \mu\left( \left\{ \Bx: \liminf_{r\to 0} \frac{\log_2 \mu(B(\Bx,r))}{\log_2 r}> s \right\} \right) >0.
\]
Consequently, one can show that (see the proof of \cite[Corollary 12.16]{graf2007foundations}), there exists $r_0>0$ and a compact set $K\subseteq \bR^d$ with $\mu(K)>0$ such that $\mu(B(\Bx,r)) \le r^s$ for all $\Bx\in K$ and all $r\le r_0$. 

For any $\delta<\mu(K)/2$ and any $S$ with $\mu(S) \ge 1-\delta$, we have $\mu(S\cap K) \ge \mu(K) - \mu(\bR^d \setminus S) \ge \mu(K)/2$. Observe that any ball with radius $\epsilon$ that intersects $S\cap K$ is contained in a ball $B(\Bx,2\epsilon)$ with $\Bx\in K$. Thus, $S\cap K$ can be covered by $\cN_c(S,\|\cdot\|_\infty,\epsilon)$ balls with radius $2\epsilon$ and centers in $K$. If $2\epsilon\le r_0$, then each ball satisfies $\mu(B(\Bx,2\epsilon)) \le (2\epsilon)^s$ and hence
\[
\cN_c(S,\|\cdot\|_\infty,\epsilon) \ge (2\epsilon)^{-s} \mu(K)/2.
\]
Therefore, for all $\delta<\mu(K)/2$,
\[
\liminf_{\epsilon \to 0} \frac{\log_2 \cN_c(\mu,\epsilon,\delta)}{-\log_2 \epsilon} \ge \liminf_{\epsilon \to 0} \frac{-s(\log_2\epsilon+1) + \log_2 \mu(K) -1}{-\log_2 \epsilon} = s.
\]
Consequently, $s_*(\mu)\ge s$. Since $s< \dim_H(\mu)$ is arbitrary, we have $\dim_H(\mu) \le s_*(\mu)$.
\end{proof}

This proposition indicates that our concepts of dimensions can capture geometric property of the distribution. The four dimensions above can all be regarded as intrinsic dimensions of distributions. For example, if $\mu$ is absolutely continuous with respect to the uniform distribution on a compact manifold of geometric dimension $s$, then $\dim_H(\mu) = \dim_M(\mu)=s$, and hence we also have $s_*(\mu) = s_p^*(\mu) =s$ for all $p\in [1,\infty)$.

In the following theorem, we obtain an upper bound on the neural network approximation error $\cW_p(\mu,\cN\cN(W,L)_\# \nu)$ in terms of the upper dimension of the target distribution $\mu$.

\begin{theorem}\label{app low dim}
Let $p\in [1,\infty)$ and $\nu$ be an absolutely continuous probability distribution on $\bR$. Suppose that $\mu$ is a probability measure on $\bR^d$ with finite absolute $q$-moment $M_q(\mu)<\infty$ for some $q>p$. If $s>s^*_{pq/(q-p)}(\mu)$, then for sufficiently large $W$ and $L$,
\[
\cW_p(\mu,\cN\cN(W,L)_\# \nu) \le Cd^{1/s}(M_q^p(\mu)+1)^{1/p} (W^2L)^{-1/s},
\]
where $C\le 384$ is an universal constant.
\end{theorem}

Notice that the approximation order only depends on the intrinsic dimension of the target distribution, and the bound grows only as $d^{1/s}$ for the ambient dimension $d$. It means that deep neural networks can overcome the course of dimensionality when approximating low-dimensional target distributions in high dimensional ambient spaces.

The proof of Theorem \ref{app low dim} is similar to Theorem \ref{app finite moment}. By the triangle inequality (\ref{triangle inequality}), Theorem \ref{app low dim} is a direct consequence of Lemma \ref{app discrete measure} and the upper bound on $\cW_p(\mu,\cP(n))$ in the next theorem, where we also give a lower bound that indicates the tightness of the upper bound.

\begin{theorem}\label{app low dim by discrete}
Suppose that $1\le p<q< \infty$. Let $\mu$ be a probability measure on $\bR^d$ with finite absolute $q$-moment $M_q(\mu)<\infty$. If $s>s^*_{pq/(q-p)}(\mu)$, then for sufficiently large $n$,
\[
\cW_p(\mu,\cP(n)) \le (M_q^p(\mu)+1)^{1/p} n^{-1/s}.
\]
If $t<s_*(\mu)$, then there exists a constant $C_\mu$ depending on $\mu$ such that
\[
\cW_p(\mu,\cP(n)) \ge C_\mu n^{-1/t}.
\]
\end{theorem}

\begin{proof}
If $s>s^*_{pq/(q-p)}(\mu)$, then there exists $\epsilon_0>0$ such that, $\log_2 \cN_c(\mu,\epsilon,\epsilon^{pq/(q-p)}) < -s \log_2 \epsilon$ for all $\epsilon\in(0,\epsilon_0)$. For any $n>\epsilon_0^{-s}$, we set $\epsilon= n^{-1/s}<\epsilon_0$, then $\cN_c(\mu,\epsilon,\epsilon^{pq/(q-p)})<n$. 

By the definition of $(\epsilon,\delta)$-covering number, there exists $S$ with $\mu(S)\ge 1-\epsilon^{pq/(q-p)}$ such that $S$ is covered by at most $n-1\ge \cN_c(S,\|\cdot\|_\infty,\epsilon)$ balls $B(\Bx_j,\epsilon)$, $j=1,\dots,n-1$. Let $F_1= S\cap B(\Bx_1,\epsilon)$ and $F_j = (S\cap B(\Bx_j,\epsilon)) \setminus (\cup_{1\le i<j} F_i) $ for $2\le j\le n-1$, then $F_j\subseteq B(\Bx_j,\epsilon)$ and $\{ F_j:1\le j\le n-1 \}$ is a partition of $S$.
	
We consider the probability distribution $\gamma = \mu(\bR^d \setminus S) \delta_{\boldsymbol{0}} + \sum_{j=1}^{n-1} \mu(F_j) \delta_{\Bx_j} \in \cP(n)$. Let 
\[
\tau = \delta_{\boldsymbol{0}} \times \mu|_{\bR^d \setminus S} + \sum_{j=1}^{n-1} \delta_{\Bx_j} \times \mu|_{F_j},
\]
then $\tau$ is a coupling of $\gamma$ and $\mu$, and we have
\begin{align*}
\cW_p^p(\mu,\gamma) &\le \int_{\bR^d \times \bR^d} \|\Bx-\By\|_\infty^p d\tau(\Bx,\By) \\
&= \int_{\bR^d\setminus S} \|\By\|_\infty^pd\mu(\By) + \sum_{j=1}^{n-1} \int_{F_j} \|\Bx_j -\By\|_\infty^p d\mu(\By) \\
&\le \mu(\bR^d\setminus S)^{1-p/q} M_q^p(\mu) + \mu(S)\epsilon^p,
\end{align*}
where we use H\"older's inequality in the last step. Since $\mu(\bR^d\setminus S)\le \epsilon^{pq/(q-p)}$, we have
\[
\cW_p^p(\mu,\gamma) \le (M_q^p(\mu)+1) \epsilon^p = (M_q^p(\mu)+1) n^{-p/s}.
\]

The second part of the theorem was also proved in \cite{weed2019sharp}. If $t<s_*(\mu)$, there exists $\delta>0$ and $\epsilon_0>0$ such that $\cN_c(\mu,\epsilon,\delta) >\epsilon^{-t}$ for all $\epsilon\in(0,\epsilon_0)$.  For any $n>\epsilon_0^{-t}$, we set $\epsilon= n^{-1/t}<\epsilon_0$, then $\cN_c(\mu,\epsilon,\delta)>n$.  For any $\gamma = \sum_{i=1}^{n}p_i\delta_{\Bx_i} \in \cP(n)$, let $S=\cup_{i=1}^n B(\Bx_i,\epsilon)$, then $\mu(S)<1-\delta$ due to $\cN_c(\mu,\epsilon,\delta)>n$. This implies
\[
\mu\left( \left\{ y: \min_{1\le i\le n}\|\Bx_i-\By\|_\infty> \epsilon \right\} \right) \ge \delta.
\]
Hence, for any coupling $\tau$ of $\gamma$ and $\mu$,
\begin{align*}
\int_{\bR^d \times \bR^d} \|\Bx-\By\|_\infty^p d\tau(\Bx,\By) 
=& \int_{ \cup_{j=1}^n \{\Bx_j\} \times\bR^d} \|\Bx-\By\|_\infty^p d\tau(\Bx,\By) \\
\ge& \int_{\bR^d} \min_{1\le i\le n}\|\Bx_i-\By\|_\infty^p d\mu(\By) \\
\ge& \delta \epsilon^p = \delta n^{-p/t}.
\end{align*}
Taking infimum in $\tau$ over all the couplings of $\mu$ and $\gamma$, we have $\cW_p(\mu,\gamma)\ge \delta^{1/p}n^{-1/t}$.
\end{proof}

Finally, we remark that \cite{weed2019sharp} gave similar upper bound on the expected error $\bE \cW_p(\mu,\widehat{\mu}_n)$ of the empirical distribution $\widehat{\mu}_n$. But the dimension they introduced is slightly large then $s^*_{pq/(q-p)}(\mu)$, hence our approximation order is better in some cases.

\section{Approximation in maximum mean discrepancies}

In this section, we apply our proof technique in Section \ref{sec: app W_p} to the approximation in the maximum mean discrepancy. This distance was used as the loss function in GANs by \cite{dziugaite2015training,li2015generative}. Empirical evidences \cite{binkowski2018demystifying} show that MMD GANs require smaller discriminative networks than Wasserstein GANs. In the theoretical part, we will derive an approximation bound for the generative networks, where the decaying order is independent of the ambient dimension, in contrast with the approximation in Wasserstein distances.

Let $\cH_K$ be a RKHS with kernel $K:\bR^d \times \bR^d \to \bR$. For simplicity, we make two assumptions on the kernel: 

\begin{assumption}\label{kernel assumption 1}
The kernel $K$ is integrally strictly positive definite: for any finite non-zero signed Borel measure $\mu$ defined on $\bR^d$, we have
\[
\int_{\bR^d}\int_{\bR^d} K(\Bx,\By) d\mu(\Bx) d\mu(\By) > 0.
\]
\end{assumption}

\begin{assumption}\label{kernel assumption 2}
There exists a constant $B>0$ such that 
\[
\sup_{x\in \bR^d} |K(\Bx,\Bx)| \le B.
\]
\end{assumption}

These assumptions are satisfied by many commonly used kernels such as Gaussian kernel $K(x,y)= \exp(|x-y|^2/2\sigma^2)$, Laplacian kernel $K(x,y)=\exp(-\sigma|x-y|)$ and inverse multiquadric kernel $K(x,y) = (|x-y|+c)^{-1/2}$ with $c>0$. It was shown in \cite[Theorem 7]{sriperumbudur2010hilbert} that Assumption \ref{kernel assumption 1} is a sufficient condition for the kernel being characteristic: $\MMD(\mu,\gamma)=0$ if and only if $\mu=\gamma$, which implies that $\MMD$ is a metric on the set of all probability measures on $\bR^d$. We will use Assumption \ref{kernel assumption 2} to get approximation error bound for generative networks.

Let $\mu$ and $\nu$ be the target and source distributions respectively. Using the same argument for Wasserstein distances in Section \ref{sec: app W_p}, we have the following ``triangle inequality'' for approximation error:
\begin{equation}\label{triangle inequality MMD}
\MMD(\mu,\cN\cN(W,L)_\#\nu)
\le \MMD(\mu,\cP(n)) + \sup_{\tau\in \cP(n)} \MMD(\tau,\cN\cN(W,L)_\#\nu),
\end{equation}
where we denote $\MMD(\mu,\cP(n)):=\inf_{\gamma\in \cP(n)}\MMD(\mu,\gamma)$. As in the Wasserstein distance case, when the size of generative network is sufficiently large, for any given $\tau\in \cP(n)$, we can construct $g\in \cN\cN(W,L)$ such that $\MMD(\tau, g_\#\nu)$ is arbitrarily small. The following lemma should be compared with Lemma \ref{app discrete measure}.

\begin{lemma}\label{app discrete measure MMD}
Suppose $W\ge 7d+1$, $L\ge 2$ and the kernel satisfies Assumption \ref{kernel assumption 1} and \ref{kernel assumption 2}. Let $\nu$ be an absolutely continuous probability distribution on $\bR$. If $n\le \frac{W-d-1}{2} \lfloor \frac{W-d-1}{6d} \rfloor \lfloor \frac{L}{2} \rfloor +2$, then for any $\mu\in \cP(n)$,
\[
\MMD(\mu,\cN\cN(W,L)_\#\nu) =0.
\]
\end{lemma}
\begin{proof}
Similar to the proof of Lemma \ref{app discrete measure}, we assume that $m:=n-1\ge 1$ and $\mu=\sum_{i=0}^{m} p_i \delta_{\Bx_i}$ with $p_i>0$ for all $0\le i\le m$. For any $\epsilon<m \min_{1\le i\le m} p_i$, we choose $2m$ points $z_{1/2} < z_1 < z_{3/2} < \dots < z_{m-1/2}< z_m$ such that 
\begin{align*}
&\nu((-\infty,z_{1/2})) = p_0, \\
&\nu((z_{i-1/2},z_{i})) = \frac{\epsilon}{m}, &&1\le i\le m,\\
&\nu((z_i,z_{i+1/2})) = p_i -\frac{\epsilon}{m}, &&1\le i\le m-1,\\
&\nu((z_m,\infty)) = p_m -\frac{\epsilon}{m}.
\end{align*}
Define $\phi :\bR \to \bR^d$ as in the proof of Lemma \ref{app discrete measure}, then $\phi \in \cN\cN(W,L)$.

It remains to estimate $\MMD(\mu, \phi_\#\nu)$. Let us denote the line segment $\cL_i :=\{(1-t)\Bx_{i-1} + t\Bx_i\in \bR^d: 0<t\le 1 \}$, then $\phi_\#\nu(\{\Bx_0\}) = p_0$, $\phi_\# \nu(\{\Bx_i\}) =p_i-\epsilon/m$ and $\phi_\# \nu(\cL_i) =p_i$ for $i=1,\dots,m$. Thanks to \cite[Theorem 1]{sriperumbudur2010hilbert}, one has
\begin{align*}
\MMD(\mu, \phi_\#\nu) &= \left\| \int_{\bR^d} K(\cdot,\Bx)d\mu(\Bx) - \int_{\bR^d} K(\cdot,\Bx)d\phi_\# \nu(\Bx) \right\|_{\cH_K}\\
&= \left\| \sum_{i=1}^m p_i K(\cdot,\Bx_i) - \int_{\bR^d} K(\cdot,\Bx)d\phi_\# \nu(\Bx) \right\|_{\cH_K},
\end{align*}
where the integrals are defined in Bochner sense. Hence,
\begin{align*}
\MMD(\mu, \phi_\#\nu) &= \left\| \frac{\epsilon}{m}\sum_{i=1}^m K(\cdot,\Bx_i) - \sum_{i=1}^m \int_{\cL_i \setminus \{\Bx_i\}} K(\cdot,\Bx)d\phi_\# \nu(\Bx) \right\|_{\cH_K} \\
&\le \sum_{i=1}^m \frac{\epsilon}{m} \|K(\cdot,\Bx_i)\|_{\cH_K} + \int_{\cL_i \setminus \{\Bx_i\}} \| K(\cdot,\Bx) \|_{\cH_K} d\phi_\# \nu(\Bx)  \\
&\le 2\sqrt{B} \epsilon,
\end{align*}
where we use $\| K(\cdot,\Bx) \|_{\cH_K} = \sqrt{K(\Bx,\Bx)}\le \sqrt{B} $ by Assumption \ref{kernel assumption 2} in the last inequality. Letting $\epsilon\to 0$ finishes the proof.
\end{proof}

Lemma \ref{app discrete measure MMD} shows that the second term in the triangle inequality (\ref{triangle inequality MMD}) vanishes. For the first term, we can approximate $\mu$ by its empirical distribution $\widehat{\mu}_n$. The following proposition, which is proved in \cite[Proposition 3.2]{lu2020universal}, gives a high-probability approximation bound for $\MMD(\mu,\widehat{\mu}_n)$.

\begin{proposition}\label{app by emp MMD}
Suppose the kernel satisfies Assumption \ref{kernel assumption 1} and \ref{kernel assumption 2}. Let $\widehat{\mu}_n = \frac{1}{n} \sum_{i=1}^n \delta_{X_i}$, where $X_i$ are i.i.d. samples from probability distribution $\mu$. Then, for any $t>0$, with probability at least $1-2e^{-t}$,
\[
\MMD(\mu, \widehat{\mu}_n) \le \frac{2B^{1/4}}{\sqrt{n}} + \frac{3\sqrt{2t} B^{1/4}}{\sqrt{n}}.
\]
\end{proposition}

By choosing the parameter $t$, we can upper bound $\MMD(\mu,\cP(n))$ and hence get an estimate on the approximation error of generative networks. The result is summarized in the next theorem. 

\begin{theorem}\label{app in MMD}
Suppose the kernel satisfies Assumption \ref{kernel assumption 1} and \ref{kernel assumption 2}. Let $\nu$ be an absolutely continuous probability distribution on $\bR$, then for any probability distribution $\mu$ on $\bR^d$, $W\ge 7d+1$ and $L\ge 2$,
\[
\MMD(\mu,\cN\cN(W,L)_\#\nu) \le 160 \sqrt{d} B^{1/4} (W^2L)^{-1/2}. 
\]
\end{theorem}

\begin{proof}
The proof is similar to the proof of Theorem \ref{app finite moment}. By the triangle inequality (\ref{triangle inequality MMD}), for any $n\in \bN$, 
\[
\MMD(\mu,\cN\cN(W,L)_\#\nu)
\le \MMD(\mu,\cP(n)) + \sup_{\tau\in \cP(n)} \MMD(\tau,\cN\cN(W,L)_\#\nu).
\]
We choose $n= \frac{W-d-1}{2} \lfloor \frac{W-d-1}{6d} \rfloor \lfloor \frac{L}{2} \rfloor +2$, then Lemma \ref{app discrete measure MMD} and Proposition \ref{app by emp MMD} imply
\[
\MMD(\mu,\cN\cN(W,L)_\#\nu) \le \MMD(\mu,\cP(n)) \le \frac{8B^{1/4}}{\sqrt{n}},
\]
where we set $t=2$ in Proposition \ref{app by emp MMD} to guarantee the existence of $\widehat{\mu}_n \in \cP(n)$ that satisfies the upper bound. Since $W\ge 7d+1$ and $L\ge 2$, it is easy to check that $n\ge cW^2L/d$ with $c=1/384$. Hence,
\[
\MMD(\mu,\cN\cN(W,L)_\#\nu) \le 160 \sqrt{d} B^{1/4}(W^2L)^{-1/2}.\qedhere
\]
\end{proof}

\section{Approximation in f-divergences}\label{sec: f-div}

This section considers the approximation capacity of generative networks in $f$-divergences. These divergences are widely used in generative adversarial networks \cite{goodfellow2014generative,nowozin2016f}. For example, the vanilla GAN tries to minimize the Jensen–Shannon divergence of the generated distribution and the target distribution. However, it was shown by \cite{arjovsky2017towards} that the disjoint supports of these distributions cause instability and vanishing gradients in training the vanilla GAN. Nevertheless, for completeness, we discuss the approximation properties of generative networks in $f$-divergences and make a comparison with Wasserstein distances and MMD. Our discussions are based on the following proposition.

\begin{proposition}\label{prop:f-div}
Assume that $f:(0,\infty)\to \bR$ is a strictly convex function with $f(1)=0$. If $\mu \perp \gamma$, then $D_f(\mu \| \gamma) = f(0)+f^*(0)>0$ is a constant.
\end{proposition}

\begin{proof}
By the convexity of $f$, the right derivative
$$
f'_+(t) := \lim_{\epsilon \downarrow 0} \frac{f(t+\epsilon)-f(t)}{\epsilon}
$$
always exists and is finite on $(0,\infty)$. Let
$$
\widetilde{f}(t) := f(t) - f'_+(1)(t-1) \ge 0,
$$
then $\widetilde{f}$ is strictly convex, decreasing on $(0,1)$ and increasing on $(1,\infty)$ with $\widetilde{f}(1)=0$. It is easy to check that $f$ and $\widetilde{f}$ induce the same divergence $D_f=D_{\widetilde{f}}$. Hence, substituting by $\widetilde{f}$ if necessary, we can always assume that $f$ is strictly convex with the global minimum $f(1)=0$.

By Lebesgue's decomposition theorem, we have
$$
\mu = \mu_a + \mu_s,
$$
where $\mu_a \ll \gamma$ and $\mu_s \perp \gamma$. A simple calculation shows that
$$
D_f(\mu \| \gamma) = \int_{\bR^d} f\left(\frac{ d\mu_a}{ d\gamma} \right) d \gamma +  f^*(0)\mu_s(\bR^d).
$$
Since $\mu$ and $\gamma$ are singular, we have $\mu_a=0$ and $\mu_s(\bR^d)=1$. Therefore, $D_f(\mu \| \gamma) = f(0)+f^*(0)>0$.
\end{proof}

Suppose that $f$ satisfies the assumption of Proposition \ref{prop:f-div}. Let $\mu$ and $\nu$ be target and source distributions on $\bR^d$ and $\bR^k$ respactively. Let $\phi:\bR^k \to \bR^d$ be a ReLU neural network.
We argue that to approximate $\mu$ by $\phi_\#\nu$ in $f$-divergences,
the dimension of $\nu$ should be no less than the intrinsic dimension of $\mu$. 

If $k<d$ and $\mu$ is absolutely continuous with respect to the Lebesgue measure, then $\mu\perp \phi_\#\nu$ and hence $D_f(\mu \| \phi_\#\nu)$ is a constant, which means we cannot approximate the target distribution $\mu$ by $\phi_\#\nu$.
More generally, we can consider the target distributions $\mu$ that are absolutely continuous with respect to the Riemannian measure \cite{pennec2006intrinsic} on some Riemannian manifold $\cM$ with dimension $s\le d$, which is a widely used assumption in applications. If $k<s$, then $\phi_\#\nu$ is supported on a manifold $\cN$ whose dimension is less than $s$ and the intersection $\cM\cap\cN$ has zero Riemannian measure on $\cM$. It implies that $\mu$ and $\phi_\# \nu$ are singular and hence $D_f(\mu \| \phi_\#\nu)$ is a positive constant. Therefore, in order to approximate $\mu$ in $f$-divergence, it is necessary that $k\ge s$.

Even when $k=s$, there still exists target distribution $\mu$ that cannot be approximated by ReLU neural networks. As an example, consider the case that $k=s=1$, $\nu$ is the uniform distribution on $[0,1]$ and $\mu$ is the uniform distribution on the unit circle $S^1\subseteq\bR^2$. Since the ReLU network $\phi:\bR \to \bR^2$ is a continuous piecewise linear function, $\phi([0,1])$ must be a union of line segments. Therefore, the intersection of $\phi([0,1])$ and the unit circle contains at most finite points, and thus its $\mu$-measure is zero. Hence, $\mu$ and $\phi_\#\nu$ are always singular and $D_f(\mu \| \phi_\#\nu)$ is a positive constant, no matter how large the network size is. In this example, it is not really possible to find any meaningful $\phi$ by minimizing $D_f(\mu \| \phi_\#\nu)$ using gradient decent methods, because the gradient always vanishes. A more detailed discussion of this phenomenon can be found in \cite{arjovsky2017towards}. 

On the other hand, a positive gap between two distributions in $f$-divergence does not necessarily mean that the distributions have gap in all aspects. In the above example of unit circle, we can actually choose a $\phi$ such that $\phi([0,1])$ is arbitrarily close to the unit circle in Euclidian distance, provided that the size of the network is sufficiently large. For such a $\phi$, the push-forward distribution $\phi_\#\nu$ and the target distribution $\mu$ generate similar samples, but their $f$-divergence is still $f(0)+f^*(0)$. This inconsistency shows that $f$-divergences are generally less adequate as metrics for the task of generating samples.

In summary, in order to approximate the target distribution in $f$-divergences, the dimension of the source distribution cannot be less than the intrinsic dimension of the target distribution. Even when the dimensions of the target distribution and the source distribution are the same, there exist some regular target distributions that cannot be approximated in $f$-divergences. In contrast, Theorem \ref{app finite moment} and \ref{app in MMD} show that we can use one-dimensional source distributions to approximate high-dimensional target distributions in Wasserstein distances and MMD, and the finite moment condition is already sufficient. It suggests that, from an approximation point of view, Wasserstein distances and MMD are more adequate as metrics of distributions for generative models.

\chapter{Error Analysis of GANs}\label{chapter: rate}

We combine the results in previous chapters to analyze the convergence rates of GANs. In Section \ref{sec: error decomp}, we decompose the error into optimization error, generator and discriminator approximation error, and generalization error. As discussed in Chapter \ref{chapter: sample comp}, the generalization error can be controlled by the complexity of the function class. The discriminator approximation error can be estimated using the function approximation bounds for neural networks in Chapter \ref{chapter: fun app}. Lastly, the results in Chapter \ref{chapter: dis app} provides bounds for the generator approximation error. Hence, if the optimization is successful, we can analyze the error of the GAN estimators. We will first assume that the target distribution has bounded support in Section \ref{sec: rate} and then extent the result to different settings in Section \ref{sec: extensions}.

\section{Convergence rates of GAN estimators} \label{sec: rate}

Let $\mu$ be an unknown target probability distribution on $\bR^d$, and let $\nu$ be a known and easy-to-sample distribution on $\bR^k$ such as uniform or Gaussian distribution. Suppose we have $n$ i.i.d. samples $\{X_i\}_{i=1}^n$ from $\mu$ and $m$ i.i.d. samples $\{Z_i\}_{i=1}^m$ from $\nu$.
Denote the corresponding empirical distributions by $\widehat{\mu}_n = \frac{1}{n} \sum_{i=1}^{n} \delta_{X_i}$ and $\widehat{\nu}_m = \frac{1}{m} \sum_{i=1}^{m} \delta_{Z_i},$  respectively. Recall that generative adversarial networks learn the target distribution $\mu$ by solving the optimization problems
\begin{align}
\argmin_{g \in \cG} d_\cF(\widehat{\mu}_n, g_\# \nu) &= \argmin_{g \in \cG} \sup_{f\in \cF} \left\{ \frac{1}{n} \sum_{i=1}^{n} f(X_i) - \bE_{\nu} [f \circ g] \right\}, \label{GAN opt 1} \\
\argmin_{g \in \cG} d_\cF(\widehat{\mu}_n, g_\# \widehat{\nu}_m) &= \argmin_{g \in \cG} \sup_{f\in \cF} \left\{ \frac{1}{n} \sum_{i=1}^{n} f(X_i) - \frac{1}{m} \sum_{j=1}^{m} f(g(Z_j)) \right\}, \label{GAN opt 2}
\end{align}
where $\cF$ is the discriminator class and $\cG$ is the generator class. And we define the GAN estimator $g^*_n$ and $g^*_{n,m}$ as the solutions of the optimization problems with optimization error $\epsilon_{opt}\ge 0$:
\begin{align}
g^*_n &\in \left\{g\in \cG: d_\cF(\widehat{\mu}_n, g_\# \nu) \le \inf_{\phi\in \cG} d_\cF(\widehat{\mu}_n, \phi_\# \nu) + \epsilon_{opt} \right\}, \label{gan estimator g_n 2} \\
g^*_{n,m} &\in \left\{g\in \cG: d_\cF(\widehat{\mu}_n, g_\# \widehat{\nu}_m) \le \inf_{\phi\in \cG} d_\cF(\widehat{\mu}_n, \phi_\# \widehat{\nu}_m) + \epsilon_{opt} \right\}. \label{gan estimator g_nm 2}
\end{align}
The performance is evaluated by the IPM $d_\cH(\mu,\gamma)$ between the target $\mu$ and the learned distribution $\gamma = (g^*_n)_\# \nu$ or $\gamma = (g^*_{n,m})_\# \nu$. In Lemma \ref{error decomposition}, we decompose the the error $d_\cH(\mu,\gamma)$ into four terms: optimization error, generalization error, discriminator approximation error and generator approximation error. Using the results in previous chapters, we can bound the  generalization error and approximation error separately. For simplicity, we first consider the case when $\mu$ is supported on the compact set $[0,1]^d$ and extend it to different situations in the next section.

\begin{theorem}\label{rate theorem}
Suppose the target $\mu$ is supported on $[0,1]^d$, the source distribution $\nu$ is absolutely continuous on $\bR$ and the evaluation class is $\cH = \cH^\alpha(\bR^d)$. Then, there exist a generator $\cG = \{g\in \cN\cN(W_1,L_1): g(\bR) \subseteq [0,1]^d \}$ with
\[
W_1^2L_1 \lesssim n,
\]
and a discriminator $\cF = \cN\cN(W_2,L_2) \cap \Lip(\bR^d, K,1)$ with
\[
W_2L_2 \lesssim n^{1/2} (\log n)^2, \quad K \lesssim (\widetilde{W}_2\widetilde{L}_2)^{2+\sigma(4\alpha-4)/d} \widetilde{L}_2 2^{\widetilde{L}_2^2},
\]
where $\widetilde{W}_2 = W_2/\log_2 W_2$ and $\widetilde{L}_2 = L_2/\log_2 L_2$, such that the GAN estimator (\ref{gan estimator g_n 2}) satisfies
\[
\bE [d_\cH(\mu,(g^*_n)_\# \nu)] - \epsilon_{opt} \lesssim n^{-\alpha/d} \lor n^{-1/2}(\log n)^{c(\alpha,d)},
\]
where $c(\alpha,d)=1$ if $2\alpha = d$, and $c(\alpha,d)=0$ otherwise.

If furthermore $m\gtrsim n^{2+2\alpha/d}(\log n)^6$, then the GAN estimator (\ref{gan estimator g_nm 2}) satisfies
\[
\bE [d_\cH(\mu,(g^*_{n,m})_\# \nu)] - \epsilon_{opt} \lesssim n^{-\alpha/d} \lor n^{-1/2}(\log n)^{c(\alpha,d)}.
\]
\end{theorem}

\begin{proof}
For the GAN estimator $g^*_n$, by Lemma \ref{error decomposition}, we have the error decomposition
\begin{equation}\label{decomp}
d_\cH(\mu,(g^*_n)_\# \nu) \le \epsilon_{opt} + 2\cE(\cH,\cF,\Omega)  + \inf_{g \in \cG} d_\cF(\widehat{\mu}_n,g_\# \nu) + d_\cH(\mu,\widehat{\mu}_n).
\end{equation}
We choose the generator class $\cG$ with $W_1^2L_1 \lesssim n$ that satisfies the condition in Lemma \ref{app discrete measure}. Then
\[
\inf_{g \in \cG} d_\cF(\widehat{\mu}_n,g_\# \nu) \le K \inf_{g \in \cG} \cW_1(\widehat{\mu}_n,g_\# \nu)=0,
\]
since $\cF \subseteq \Lip([0,1]^d,K)$. By Theorem \ref{app theorem width depth} (see also the discussion after Theorem \ref{app theorem width depth}), for our choice of the discriminator class $\cF$,
\[
\cE(\cH,\cF,[0,1]^d) \lesssim (W_2L_2 / (\log_2 W_2 \log_2 L_2))^{-2\alpha/d} \lesssim n^{-\alpha/d},
\]
where we can choose $W_2L_2 \asymp n^{1/2} (\log n)^2$ so that the last inequality holds. By Lemma \ref{symmetrization} and \ref{entropy integral},
\[
\bE [d_\cH(\mu,\widehat{\mu}_n)] \lesssim \inf_{0< \delta<1/2}\left( \delta + \frac{3}{\sqrt{n}} \int_{\delta}^{1/2} \sqrt{\log \cN_c(\cH(X_{1:n}),\|\cdot\|_\infty,\epsilon)} d\epsilon \right).
\]
Since the samples $X_{1:n}$ from $\mu$ are supported on $[0,1]^d$, we have
\[
\log \cN_c(\cH(X_{1:n}),\|\cdot\|_\infty,\epsilon) \le \log \cN_c(\cH,\|\cdot\|_{L^\infty([0,1]^d)},\epsilon)  \lesssim \epsilon^{-d/\alpha},
\]
where the last inequality is from the entropy bound in \cite{kolmogorov1961} (see also Lemma \ref{holder covering number}). Thus, if we denote $\eta=d/(2\alpha)$, then
\[
\bE [d_\cH(\mu,\widehat{\mu}_n)] \lesssim \inf_{0< \delta<1/2}\left( \delta + n^{-1/2} \int_{\delta}^{1/2} \epsilon^{-\eta} d\epsilon \right).
\]
When $\eta<1$, one has
\[
\bE [d_\cH(\mu,\widehat{\mu}_n)] \lesssim \inf_{0< \delta<1/2}\left( \delta + (1-\eta)^{-1}n^{-1/2} (2^{\eta-1} - \delta^{1-\eta}) \right) \lesssim n^{-1/2}.
\]
When $\eta = 1$, one has
\[
\bE [d_\cH(\mu,\widehat{\mu}_n)] \lesssim \inf_{0< \delta<1/2}\left( \delta + n^{-1/2} (-\log 2 - \log \delta) \right)\lesssim n^{-1/2} \log n,
\]
where we take $\delta=n^{-1/2}$ in the last step. When $\eta>1$, one has
\[
\bE [d_\cH(\mu,\widehat{\mu}_n)] \lesssim \inf_{0< \delta<1/2}\left( \delta + (\eta-1)^{-1}n^{-1/2} (\delta^{1-\eta} -2^{\eta-1}) \right) \lesssim n^{-1/(2\eta)} = n^{-\alpha/d},
\]
where we take $\delta=n^{-1/(2\eta)}$. Combining these cases together, we have
\begin{equation}\label{holder complexity}
\bE [d_\cH(\mu,\widehat{\mu}_n)] \lesssim n^{-\alpha/d} \lor n^{-1/2}(\log n)^{c(\alpha,d)},
\end{equation}
where $c(\alpha,d)=1$ if $2\alpha = d$, and $c(\alpha,d)=0$ otherwise. In summary, by (\ref{decomp}), we have
\[
\bE [d_\cH(\mu,(g^*_n)_\# \nu)] - \epsilon_{opt} \lesssim n^{-\alpha/d} \lor n^{-1/2}(\log n)^{c(\alpha,d)}.
\]

For the estimator $g^*_{n,m}$, we only need to estimate the extra term $\bE [d_{\cF \circ \cG}(\nu,\widehat{\nu}_m)]$ by Lemma \ref{error decomposition}. We can bound this generalization error by the entropy integral and further bound it by the pseudo-dimension $\Pdim(\cF \circ \cG)$ of the network $\cF \circ \cG$ (see Lemma \ref{Rc bound by pdim}):
\[
\bE [d_{\cF \circ \cG}(\nu,\widehat{\nu}_m)] \lesssim  \sqrt{\frac{ \Pdim(\cF \circ \cG) \log m}{m}}.
\]
It was shown in \cite{bartlett2019nearly} that the pseudo-dimension of a ReLU neural network satisfies the bound $\Pdim(\cN\cN(W,L)) \lesssim UL \log U$, where $U\asymp W^2L$ is the number of parameters. Hence,
\[
\bE [d_{\cF \circ \cG}(\nu,\widehat{\nu}_m)] \lesssim \sqrt{\frac{(W_1^2L_1+W_2^2L_2)(L_1+L_2)\log(W_1^2L_1+W_2^2L_2) \log m}{m}}.
\]
Since we have chosen $W_2L_2 \lesssim n^{1/2} (\log n)^2$ and $W_1^2L_1 \lesssim n$,  we have
\[
\bE [d_{\cF \circ \cG}(\nu,\widehat{\nu}_m)] \lesssim \sqrt{\frac{(n+n(\log n)^4)(n+n^{1/2} (\log n)^2) \log n \log m}{m}} \lesssim \sqrt{\frac{n^2 (\log n)^5 \log m}{m}}.
\]
Hence, if $m\gtrsim n^{2+2\alpha/d}(\log n)^6$, then $\bE [d_{\cF \circ \cG}(\nu,\widehat{\nu}_m)] \lesssim n^{-\alpha/d}$ and, by Lemma \ref{error decomposition},
\[
\bE [d_\cH(\mu,(g^*_{n,m})_\# \nu)] - \epsilon_{opt} \lesssim n^{-\alpha/d} \lor n^{-1/2}(\log n)^{c(\alpha,d)}. \qedhere
\]
\end{proof}

We make several remarks on the theorem and the proof.

\begin{remark}
If $\alpha=1$, then $\cH^1([0,1]^d) = \Lip([0,1]^d, 1,1)$ and $d_{\cH^1}$ is the Wasserstein distance $\cW_1$ on $[0,1]^d$. In this case, the required Lipschitz constant of the discriminator network is reduced to $K \lesssim \widetilde{W}_2^2\widetilde{L}_2^3 2^{\widetilde{L}_2^2}$. If we choose the depth $L_2$ to be a constant, then the Lipschitz constant can be chosen to have the order of $K \lesssim \widetilde{W}_2^2 \lesssim n(\log n)^2$.
\end{remark}

\begin{remark}
For simplicity, we assume that the source distribution $\nu$ is on $\bR$. This is not a restriction, because any absolutely continuous distribution on $\bR^k$ can be projected to an absolutely continuous distribution on $\bR$ by linear mapping. Hence, the same result holds for any absolutely continuous source distribution on $\bR^k$.
The requirement on the generator that $g(\bR) \subseteq [0,1]^d$ is also easy to satisfy by adding an additional clipping layer to the output and using the fact that
\[
\min\{ \max\{ x, -1\}, 1 \} = \sigma(x+1) - \sigma(x-1) -1, \quad x\in \bR.
\]
\end{remark}

\begin{remark}\label{proof essential remark}
Our error decomposition for GANs in Lemma \ref{error decomposition} is different from the classical bias-variance decomposition for regression in the sense that the generalization error $d_\cF(\mu,\widehat{\mu}_n) \land d_\cH(\mu,\widehat{\mu}_n) \le d_\cH(\mu,\widehat{\mu}_n)$ depends on the evaluation class $\cH$. The proof of Theorem \ref{rate theorem} essentially shows that we can choose the generator class and the discriminator class sufficiently large to reduce the approximation error so that the learning rate of GAN estimator is not slower than that of the empirical distribution.
\end{remark}

\begin{remark}
We give explicit estimate of the Lipschitz constant of the discriminator, because it is essential in bounding the generator approximation error in our analysis. Alternatively, one can also bound the size of the weights in the discriminator network and then estimate the Lipschitz constant. For example, by using the construction in \cite{yarotsky2017error}, one can bound the weights as $\cO(\epsilon^{-c})$ for some $c>0$, where $\epsilon$ is the approximation error. Then convergence rates can be obtained for the discriminator network with bounded weights (the bound depends on the sample size $n$).
\end{remark}

\begin{remark}
The bound on the expectation $\bE [d_\cH(\mu,(g^*_n)_\# \nu)]$ can be turned into a high probability bound by using concentration inequalities \cite{boucheron2013concentration,shalevshwartz2014understanding,mohri2018foundations}. For example, by McDiarmid's inequality, one can shows that, for all $t> 0$,
\begin{equation}\label{probability bound}
\bP \left( d_\cH(\mu,\widehat{\mu}_n) \ge \bE [d_\cH(\mu,\widehat{\mu}_n)] + t \right) \le  \exp(-nt^2 /2),
\end{equation}
because for any $\{X_i\}_{i=1}^n$ and $\{X_i'\}_{i=1}^n$ that satisfies $X_i'=X_i$ except for $i=j$, we have
\[
\left|\sup_{h\in \cH} \left( \bE_\mu[h] - \frac{1}{n}\sum_{i=1}^n h(X_i) \right) - \sup_{h\in \cH} \left( \bE_\mu[h] - \frac{1}{n}\sum_{i=1}^n h(X_i') \right) \right| \le \sup_{h\in \cH} \frac{1}{n} \left| h(X_j) - h(X_j') \right| \le \frac{2}{n}.
\]
Since other error terms in (\ref{decomp}) can be bounded independent of the random samples, if we choose $\exp(-nt^2 /2) =\delta$ in (\ref{probability bound}), then it holds with probability at least $1-\delta$ that
\[
d_\cH(\mu,(g^*_n)_\# \nu) - \epsilon_{opt} - \sqrt{\frac{2\log(1/\delta)}{n}} \lesssim n^{-\alpha/d} \lor n^{-1/2}(\log n)^{c(\alpha,d)}.
\]
\end{remark}

\begin{remark}\label{remark: opt}
There is an optimization error term in our results of convergence rates. So, in order to estimate the full error of GANs used in practice, one also need to estimate the optimization error, which is still a very difficult problem at present. Fortunately, our error analysis is independent of the optimization, so it is possible to combine it with other analysis of optimization. In the theorem, we give bounds on the network size so that GANs can achieve the optimal convergence rates of learning distributions. In practice, as the network size and sample size get larger, the training becomes more difficult and hence the optimization error may become larger. So there is a trade-off between the optimization error and the bounds derived here. This trade-off can provide some guide on the choice of network size in practice.
\end{remark}

It has been demonstrated that Lipschitz continuity of the discriminator is a key condition for a stable training of GANs \cite{arjovsky2017towards,arjovsky2017wasserstein}. In the original Wasserstein GAN \cite{arjovsky2017wasserstein}, the Lipschitz constraint on the discriminator is implemented by weight clipping. In the follow-up works, several regularization methods have been proposed to enforce Lipschitz condition, such as gradient penalty \cite{gulrajani2017improved,petzka2018regularization}, weight normalization \cite{miyato2018spectral} and weight penalty \cite{brock2019large}. In Theorem \ref{rate theorem}, we directly assume that the discriminator $\cF = \cN\cN(W_2,L_2) \cap \Lip(\bR^d, K,1)$ has a bounded Lipschitz constant $K$. In the next theorem, we assume that the discriminator is a norm constrained neural network $\cN\cN(W,L,K)$ and control the Lipschitz constant by the norm constraint.

\begin{theorem}\label{rate theorem norm}
Suppose the target $\mu$ is supported on $[0,1]^d$, the source distribution $\nu$ is absolutely continuous on $\bR$ and the evaluation class is $\cH = \cH^\alpha(\bR^d)$, where $\alpha = r+\alpha_0>0$ with $r\in \bN_0 $ and $\alpha_0\in (0,1]$. Then, there exist a generator $\cG = \{g\in \cN\cN(W_1,L_1): g(\bR) \subseteq [0,1]^d \}$ with
\[
W_1^2L_1 \lesssim n,
\]
and a constant $c>0$ such that, if the discriminator is chosen as $\cF=\cN\cN(W_2,L_2,K)$ with
\[
W_2\ge c K^{(2d+\alpha)/(2d+2)},\quad  L_2 \ge 4 \lceil \log_2 (d+r) \rceil+2, \quad K \asymp n^{(d+1)/d},
\]
then the GAN estimator (\ref{gan estimator g_n 2}) satisfies
\[
\bE[d_{\cH}(\mu,(g^*_n)_\# \nu)] - \epsilon_{opt} \lesssim n^{-\alpha/d} \lor n^{-1/2} (\log n)^{c(\alpha,d)},
\]
where $c(\alpha,d)=1$ if $2\alpha = d$, and $c(\alpha,d)=0$ otherwise.

If furthermore $W_2 \asymp K^{(2d+\alpha)/(2d+2)} \asymp n^{1+\alpha/(2d)}$, $L_2 \asymp 1$ and $m\gtrsim n^{3+3\alpha/d}(\log n)^2$, then the GAN estimator (\ref{gan estimator g_nm 2}) satisfies
\[
\bE [d_\cH(\mu,(g^*_{n,m})_\# \nu)] - \epsilon_{opt} \lesssim n^{-\alpha/d} \lor n^{-1/2}(\log n)^{c(\alpha,d)}.
\]
\end{theorem}
\begin{proof}
The proof is similar to Theorem \ref{rate theorem}. By Theorem \ref{app theorem norm constraint} and our choice of $W_2$ and $L_2$, the discriminator approximation error satisfies
\[
\cE(\cH, \cF,[0,1]^d) \lesssim K^{-\alpha/(d+1)}.
\]
If we choose $K \asymp n^{(d+1)/d}$, then $\cE(\cH, \cF,[0,1]^d) \lesssim n^{-\alpha/d}$. Since any $f\in \cF$ is $K$-Lipschitz,
\[
\inf_{g \in \cG} d_\cF(\widehat{\mu}_n,g_\# \nu) \le K \inf_{g \in \cG} \cW_1 (\widehat{\mu}_n,g_\# \nu) = 0,
\]
by choosing the generator class $\cG$ with $W_1^2L_1 \lesssim n$ that satisfies the condition in Lemma \ref{app discrete measure}. As in the proof of Theorem \ref{rate theorem}, the generalization error  
\[
\bE[d_{\cH}(\mu,\widehat{\mu}_n)] \lesssim n^{-\alpha/d} \lor n^{-1/2} (\log n)^{c(\alpha,d)}.
\]
The convergence rate for the estimator $g^*_n$ follows from the error decomposition Lemma \ref{error decomposition}.

For the estimator $g^*_{n,m}$, we only need to estimate the extra term $\bE [d_{\cF \circ \cG}(\nu,\widehat{\nu}_m)]$. As in the proof of Theorem \ref{rate theorem}, we have
\begin{align*}
\bE [d_{\cF \circ \cG}(\nu,\widehat{\nu}_m)] &\lesssim \sqrt{\frac{(W_1^2L_1+W_2^2L_2)(L_1+L_2)\log(W_1^2L_1+W_2^2L_2) \log m}{m}} \\
&\lesssim \sqrt{\frac{(n+n^{2+\alpha/d})n \log n \log m}{m}} \lesssim \sqrt{\frac{n^{3+\alpha/d} \log n \log m}{m}}.
\end{align*}
Hence, if $m\gtrsim n^{3+3\alpha/d}(\log n)^2$, then $\bE [d_{\cF \circ \cG}(\nu,\widehat{\nu}_m)] \lesssim n^{-\alpha/d}$ and, by Lemma \ref{error decomposition},
\[
\bE [d_\cH(\mu,(g^*_{n,m})_\# \nu)] - \epsilon_{opt} \lesssim n^{-\alpha/d} \lor n^{-1/2}(\log n)^{c(\alpha,d)}. \qedhere
\]
\end{proof}

The constrained optimization problem (\ref{GAN opt 1}) with $\cF=\cN\cN(W,L,K)$ might still be difficult to compute in practice. Instead, we can use the corresponding regularized optimization
\begin{equation}\label{regularized GAN}
\argmin_{g\in \cG} d_{\cF,\lambda}(\widehat{\mu}_n, g_\# \nu) := \argmin_{g\in \cG} \sup_{\phi_\theta\in \cF} \bE_{\widehat{\mu}_n}[\phi_\theta] - \bE_{g_\# \nu} [\phi_\theta] - \lambda \kappa(\theta)^2, \quad \lambda> 0,
\end{equation}
where $\cF=\cN\cN(W,L)$ is a neural network without norm constraints and $\kappa(\theta)$ is the norm constraint defined by (\ref{norm constraint}). The following proposition shows the relation of regularized problem (\ref{regularized GAN}) and the constrained optimization problem (\ref{GAN opt 1}).

\begin{proposition}\label{regularized IPM}
For any probability distributions $\mu$ and $\gamma$ defined on $\bR^d$, any $\lambda,K>0$,
\[
d_{\cF,\lambda}(\mu,\gamma) = \frac{d_{\cF_K}(\mu,\gamma)^2}{4\lambda K^2},
\]
where $\cF=\cN\cN(W,L)$ and $\cF_K:=\cN\cN(W,L,K)$.
\end{proposition}
\begin{proof}
Observe that, for any $a\ge 0$,
\[
\sup_{\phi_\theta\in \cF, \kappa(\theta)=a} \bE_\mu[\phi_\theta] - \bE_\gamma[\phi_\theta] = a \sup_{\phi_\theta\in \cF, \kappa(\theta)=1} \bE_\mu[\phi_\theta] - \bE_\gamma[\phi_\theta],
\]
because if $\phi_\theta$ is parameterized by $\theta= ((A_0,\Bb_0),\dots,(A_L,\Bb_L))$, then $a\phi_\theta$ can be parameterized by $\theta'= ((A_0,\Bb_0),\dots,(A_{L-1},\Bb_{L-1}), (aA_L,a\Bb))$ and $\kappa(\theta')= a \kappa(\theta)$. Thus,
\[
d_{\cF_K}(\mu,\gamma) = \sup_{0\le a\le K} \sup_{\phi_\theta\in \cF, \kappa(\theta)=a} \bE_\mu[\phi_\theta] - \bE_\gamma[\phi_\theta] = K \sup_{\phi_\theta\in \cF, \kappa(\theta)=1} \bE_\mu[\phi_\theta] - \bE_\gamma[\phi_\theta].
\]
Therefore,
\begin{align*}
d_{\cF,\lambda}(\mu,\gamma) &= \sup_{\phi_\theta\in \cF} \bE_{\mu}[\phi_\theta] - \bE_{\gamma} [\phi_\theta] - \lambda \kappa(\theta)^2 \\
&= \sup_{a\ge 0} \sup_{\phi_\theta\in \cF, \kappa(\theta)=a} \bE_\mu[\phi_\theta] - \bE_\gamma[\phi_\theta] - \lambda a^2 \\
&= \sup_{a\ge 0} \frac{a}{K} d_{\cF_K}(\mu,\gamma) - \lambda a^2 \\
&= \frac{d_{\cF_K}(\mu,\gamma)^2}{4\lambda K^2},
\end{align*}
where the supremum is achieved at $a=\frac{1}{2\lambda K} d_{\cF_K}(\mu,\gamma)$ in the last equality.
\end{proof}

Combining Proposition \ref{regularized IPM} with Theorem \ref{rate theorem norm}, we can obtain the learning rate of the solution of the regularized optimization problem (\ref{regularized GAN}).

\begin{corollary}\label{GAN convergence rate reg}
Under the assumption of Theorem \ref{rate theorem norm}, let $W_2,L_2,K$ be the parameters in Theorem \ref{rate theorem norm} and $\lambda = \frac{1}{4K^2} \asymp n^{-2(d+1)/d}$, then for any GAN estimator $g^*_{n,\lambda} \in \cG$ satisfying
\[
d_{\cF,\lambda}(\widehat{\mu}_n, (g^*_{n,\lambda})_\# \nu) \le \argmin_{g\in \cG} d_{\cF,\lambda}(\widehat{\mu}_n, g_\# \nu) + \epsilon_{opt},
\]
where $\cF=\cN\cN(W_2,L_2)$, we have
\[
\bE[d_{\cH^\alpha}(\mu,(g^*_{n,\lambda})_\# \nu)] - \sqrt{\epsilon_{opt}} \lesssim n^{-\alpha/d} \lor n^{-1/2} (\log n)^{c(\alpha,d)},
\]
where $c(\alpha,d)=1$ if $2\alpha = d$, and $c(\alpha,d)=0$ otherwise.
\end{corollary}
\begin{proof}
Since $\lambda= \frac{1}{4K^2}$, by Proposition \ref{regularized IPM},
\begin{align*}
d_{\cF_K}(\widehat{\mu}_n, (g^*_{n,\lambda})_\# \nu)^2 &= d_{\cF,\lambda}(\widehat{\mu}_n, (g^*_{n,\lambda})_\# \nu) \le \argmin_{g\in \cG} d_{\cF,\lambda}(\widehat{\mu}_n, g_\# \nu) + \epsilon_{opt} \\
&= \argmin_{g\in \cG} d_{\cF_K}(\widehat{\mu}_n, g_\# \nu)^2 + \epsilon_{opt},
\end{align*}
where we denote $\cF_K=\cN\cN(W_2,L_2,K)$. As a consequence,
\[
d_{\cF_K}(\widehat{\mu}_n, (g^*_{n,\lambda})_\# \nu) \le \sqrt{\argmin_{g\in \cG} d_{\cF_K}(\widehat{\mu}_n, g_\# \nu)^2 + \epsilon_{opt}} \le \argmin_{g\in \cG} d_{\cF_K}(\widehat{\mu}_n, g_\# \nu) + \sqrt{\epsilon_{opt}},
\]
which means $g^*_{n,\lambda} \in \cG$ satisfies (\ref{gan estimator g_n 2}) with discriminator $\cF_K$ and optimization error $\sqrt{\epsilon_{opt}}$. Hence, we can apply Theorem \ref{rate theorem norm}.
\end{proof}

\section{Extensions} \label{sec: extensions}

In this section, we extend the error analysis to the following cases: (1) the target distribution concentrates around a low-dimensional set, (2) the target distribution has a smooth density function and, (3) the target distribution has an unbounded support. For simplicity, we will assume that the discriminator is $\cF = \cN\cN(W_2,L_2) \cap \Lip(\bR^d, K,1)$ as in Theorem \ref{rate theorem}, which has an explicit bound on the Lipschitz constant. For norm constraint neural network $\cF= \cN\cN(W_2,L_2,K)$ as in Theorem \ref{rate theorem norm}, similar results can be derived using the same argument.

\subsection{Learning low-dimensional distributions}

The convergence rate in Theorem \ref{rate theorem} suffers from the curse of dimensionality. In practice, the ambient dimension is usually large, which makes the convergence very slow. However, in many applications, high-dimensional complex data such as images, texts and natural languages, tend to be supported on approximate lower-dimensional manifolds. To take into account this fact, we assume that the target distribution $\mu$ has a low-dimensional structure. 

Recall that the Minkowski dimension (see Definition \ref{Hausdorff Minkowski dimensions}) measures how the covering number of a set decays when the radius of covering balls converges to zero. The Minkowski dimension of a manifold is the same as the geometry dimension of the manifold. For function classes defined on a set with a small Minkowski dimension, it is intuitive to expect that the covering number only depends on the intrinsic Minkowski dimension, rather than the ambient dimension. \cite{kolmogorov1961} gave a comprehensive study on such problems. We will need the following useful lemma in our analysis.

\begin{lemma}[\cite{kolmogorov1961}]\label{holder covering number}
If $\cX \subseteq \bR^d$ is a compact set with $\dim_M(\cX)=d^*$, then
\[
\log \cN_c(\cH^\alpha(\cX), \|\cdot\|_{L^\infty(\cX)},\epsilon) \lesssim \epsilon^{-d^*/\alpha} \log(1/\epsilon).
\]
If, in addition, $\cX$ is connected, then
\[
\log \cN_c(\cH^\alpha(\cX), \|\cdot\|_{L^\infty(\cX)},\epsilon) \lesssim \epsilon^{-d^*/\alpha}.
\]
\end{lemma}

For regression, \cite{nakada2020adaptive} showed that deep neural networks can adapt to the low-dimensional structure of data, and the convergence rates do not depend on the nominal high dimensionality of data, but on its lower intrinsic dimension. We will show that similar results hold for GANs by analyzing the learning rates of a target distribution that concentrates on a low-dimensional set. To be concrete, we make the following assumption on the target distribution.

\begin{assumption}\label{low-dim assumption}
\textnormal{The target $X\sim \mu$ has the form $X=\widetilde{X} + \xi$, where $\widetilde{X}$ and $\xi$ are independent, $\widetilde{X} \sim \widetilde{\mu}$ is supported on some compact set $\cX\subseteq [0,1]^d$ with $\dim_M(\cX)=d^*$, and $\xi$ has zero mean $\bE[\xi] =0$ and bounded variance $V=\bE[\|\xi\|_\infty^2 ] <\infty$.
}
\end{assumption}

The next theorem shows that the convergence rates of the GAN estimators only depend on the intrinsic dimension $d^*$, when the network architectures are properly chosen.

\begin{theorem}
Suppose the target $\mu$ satisfies assumption \ref{low-dim assumption}, the source distribution $\nu$ is absolutely continuous on $\bR$ and the evaluation class is $\cH = \cH^\alpha(\bR^d)$. Then, there exist a generator $\cG = \{g\in \cN\cN(W_1,L_1): g(\bR) \subseteq [0,1]^d \}$ with
\[
W_1^2L_1 \lesssim n,
\]
and a discriminator $\cF = \cN\cN(W_2,L_2) \cap \Lip(\bR^d, K,1)$ with
\[
W_2L_2 \lesssim n^{d/(2d^*)} (\log n)^2, \quad K \lesssim (\widetilde{W}_2\widetilde{L}_2)^{2+\sigma(4\alpha-4)/d} \widetilde{L}_2 2^{\widetilde{L}_2^2},
\]
where $\widetilde{W}_2 = W_2/\log_2 W_2$ and $\widetilde{L}_2 = L_2/\log_2 L_2$, such that the GAN estimator (\ref{gan estimator g_n 2}) satisfies
\[
\bE [d_\cH(\mu,(g^*_n)_\# \nu)] - \epsilon_{opt} - 2d V^{(\alpha \land 1)/2} \lesssim (n^{-\alpha/d^*}  \lor n^{-1/2}) \log n.
\]

If furthermore
\[
m\gtrsim
\begin{cases}
n^{(3d+4\alpha)/(2d^*)}(\log n)^6 \quad & d^*\le d/2,\\
n^{1+(d+2\alpha)/d^*}(\log n)^4 \quad & d^*> d/2,
\end{cases}
\]
then the GAN estimator (\ref{gan estimator g_nm 2}) satisfies
\[
\bE [d_\cH(\mu,(g^*_{n,m})_\# \nu)] - \epsilon_{opt} - 2d V^{(\alpha \land 1)/2} \lesssim (n^{-\alpha/d^*}  \lor n^{-1/2}) \log n.
\]
\end{theorem}

\begin{proof}
For any i.i.d. observations ${X_{1:n}}=\{X_i \}_{i=1}^n$ from $\mu$, where $X_i=\widetilde{X}_i + \xi_i$ with $\widetilde{X}_i\sim \widetilde{\mu}$, we denote $\widehat{\mu}_n = \frac{1}{n} \sum_{i=1}^{n} \delta_{X_i}$ and $\widehat {\widetilde{\mu}}_n = \frac{1}{n} \sum_{i=1}^{n} \delta_{\widetilde{X}_i}$. As in the proof of Theorem \ref{rate theorem}, by Lemma \ref{general error decomposition} and Lemma \ref{app discrete measure}, we have
\begin{align*}
\bE [d_\cH(\mu,(g^*_n)_\# \nu)] &\le d_\cH(\mu,\widetilde{\mu}) + \bE [d_\cH(\widetilde{\mu},(g^*_n)_\# \nu)] \\
&\le d_\cH(\mu,\widetilde{\mu}) + 2\cE(\cH,\cF,[0,1]^d) +  \bE [d_\cH(\widetilde{\mu},\widehat{\mu}_n)] + \epsilon_{opt},
\end{align*}
and there exists a discriminator $\cF$ with $W_2L_2 \asymp n^{d/(2d^*)} (\log n)^2$ such that
\[
\cE(\cH,\cF,[0,1]^d) \lesssim (W_2L_2 / (\log_2 W_2 \log_2 L_2))^{-2\alpha/d} \lesssim n^{-\alpha/d^*}.
\]

For the term $d_\cH(\mu,\widetilde{\mu})$, we can bound it as
\begin{equation}\label{low-dim ineq}
d_\cH(\mu,\widetilde{\mu}) = \sup_{h\in \cH} \bE_{\xi}[\bE_{\widetilde{X}} [h(\widetilde{X}+ \xi) - h(\widetilde{X})] ] \le d \bE_{\xi}[\|\xi\|_\infty^{\alpha \land 1}] \le d V^{(\alpha \land 1)/2},
\end{equation}
where we use the Lipschitz inequality $|h(\widetilde{X}+ \xi) - h(\widetilde{X})| \le d \|\xi\|_\infty^{\alpha \land 1}$ for the second inequality, and Jensen's inequality for the last inequality.

For the generalization error, we have
\[
\bE_{X_{1:n}} [d_\cH(\widetilde{\mu},\widehat{\mu}_n)] \le \bE_{\widetilde{X}_{1:n}} d_\cH(\widetilde{\mu}, \widehat {\widetilde{\mu}}_n) + \bE_{\xi_{1:n}} \bE_{\widetilde{X}_{1:n}} d_\cH(\widehat {\widetilde{\mu}}_n, \widehat{\mu}_n).
\]
Using H\"older continuity of $h$, we have
\begin{align}
\bE_{\xi_{1:n}} \bE_{\widetilde{X}_{1:n}} d_\cH(\widehat {\widetilde{\mu}}_n, \widehat{\mu}_n) &= \bE_{\xi_{1:n}} \bE_{\widetilde{X}_{1:n}} \sup_{h\in \cH} \frac{1}{n} \sum_{i=1}^n h(\widetilde{X}_i+ \xi_i) - h(\widetilde{X}_i) \notag \\
&\le d \bE_{\xi_{1:n}} \frac{1}{n} \sum_{i=1}^n \|\xi_i\|_\infty^{\alpha \land 1} \label{lip ineq} \\
&\le d V^{(\alpha \land 1)/2}. \notag
\end{align}
To estimate $\bE_{\widetilde{X}_{1:n}} d_\cH(\widetilde{\mu}, \widehat {\widetilde{\mu}}_n)$, recall that we have denoted $\cH(\widetilde{X}_{1:n}) := \{(h(\widetilde{X}_1),\dots,h(\widetilde{X}_n)):h\in \cH \} \subseteq \bR^n$. Since $\widetilde{\mu}$ is supported on $\cX$ with $\dim_M(\cX)=d^*$ by Assumption \ref{low-dim assumption}, the covering number of $\cH(\widetilde{X}_{1:n})$ with respect to the distance $\|\cdot\|_\infty$ on $\bR^n$ can be bounded by the covering number of $\cH$ with respect to the $L^\infty(\cX)$ distance. Hence,
\[
\log \cN_c(\cH(\widetilde{X}_{1:n}),\|\cdot\|_\infty,\epsilon) \le \log \cN_c(\cH^\alpha(\cX), \|\cdot\|_{L^\infty(\cX)},\epsilon) \lesssim \epsilon^{-d^*/\alpha}\log(1/\epsilon),
\]
by Lemma \ref{holder covering number}. Therefore, by Lemma \ref{symmetrization} and \ref{entropy integral},
\begin{align*}
\bE_{\widetilde{X}_{1:n}} d_\cH(\widetilde{\mu}, \widehat {\widetilde{\mu}}_n) &\le 8\bE_{\widetilde{X}_{1:n}} \inf_{0< \delta<1/2}\left( \delta + \frac{3}{\sqrt{n}} \int_{\delta}^{1/2} \sqrt{\log \cN_c(\cH(\widetilde{X}_{1:n}),\|\cdot\|_\infty,\epsilon)} d\epsilon \right) \\
&\lesssim \inf_{0< \delta<1/2}\left( \delta + n^{-1/2} \int_{\delta}^{1/2} \epsilon^{-d^*/(2\alpha)} \log (1/\epsilon) d\epsilon \right) \\
&\lesssim \inf_{0< \delta<1/2}\left( \delta + n^{-1/2} \log (1/\delta) \int_{\delta}^{1/2} \epsilon^{-d^*/(2\alpha)} d\epsilon \right).
\end{align*}
A calculation similar to the inequality (\ref{holder complexity}) gives
\[
\bE_{\widetilde{X}_{1:n}} d_\cH(\widetilde{\mu}, \widehat {\widetilde{\mu}}_n) \lesssim (n^{-\alpha/d^*}  \lor n^{-1/2}) \log n.
\]
Therefore,
\[
\bE_{X_{1:n}} [d_\cH(\widetilde{\mu},\widehat{\mu}_n)] - d V^{(\alpha \land 1)/2} \lesssim (n^{-\alpha/d^*}  \lor n^{-1/2}) \log n.
\]
In summary, we obtain the desired bound
\[
\bE [d_\cH(\mu,(g^*_n)_\# \nu)] - \epsilon_{opt} - 2d V^{(\alpha \land 1)/2} \lesssim (n^{-\alpha/d^*}  \lor n^{-1/2}) \log n.
\]

For the estimator $g^*_{n,m}$, we use the pseudo-dimension to bound $\bE [d_{\cF \circ \cG}(\nu,\widehat{\nu}_m)]$. Since we have chosen $W_2L_2 \lesssim n^{d/(2d^*)} (\log n)^2$ and $W_1^2L_1 \lesssim n$,
\begin{align*}
\bE [d_{\cF \circ \cG}(\nu,\widehat{\nu}_m)] &\lesssim \sqrt{\frac{(W_1^2L_1+W_2^2L_2)(L_1+L_2)\log(W_1^2L_1+W_2^2L_2) \log m}{m}} \\
&\lesssim \sqrt{\frac{(n+n^{d/d^*}(\log n)^4)(n+n^{d/(2d^*)} (\log n)^2) \log n \log m}{m}} \\
&\lesssim \sqrt{\frac{n^{d/d^*}(n+n^{d/(2d^*)} (\log n)^2) (\log n)^5 \log m}{m}}.
\end{align*}
By our choice of $m$, we always have $\bE [d_{\cF \circ \cG}(\nu,\widehat{\nu}_m)] \lesssim n^{-\alpha/d^*} \log n$. The result then follows from the error decomposition Lemma \ref{error decomposition}.
\end{proof}

\begin{remark}
In the proof, we actually show that the same convergence rate holds for $\widetilde{\mu}$: $\bE [d_\cH(\widetilde{\mu},(g^*_n)_\# \nu)] - \epsilon_{opt} - d V^{(\alpha \land 1)/2} \lesssim (n^{-\alpha/d^*}  \lor n^{-1/2}) \log n$. Note that the constant $d$ is due to the Lipschitz constant of the evaluation class $\cH^\alpha$. When $\alpha=1$, we have a better Lipschitz inequality $|h(\widetilde{X}+ \xi) - h(\widetilde{X})| \le \|\xi\|_\infty$ in inequalities (\ref{low-dim ineq}) and (\ref{lip ineq}). As a consequence, one can check that, for Wasserstein distance,
\[
\bE [d_{\cH^1}(\mu,(g^*_n)_\# \nu)] - \epsilon_{opt} - 2 V^{1/2} \lesssim (n^{-1/d^*}  \lor n^{-1/2}) \log n.
\]
This bound is useful only when the variance term $V^{1/2}$ is negligible, i.e. the data distribution is really low-dimensional. One can regard the variance as a ``measure'' of how well the low-dimension assumption is fulfilled. It is numerically confirmed that several well-known real data have small intrinsic dimensions, while their nominal dimensions are very large \cite{nakada2020adaptive}.
\end{remark}

\subsection{Learning distributions with densities}

When the target distribution $\mu$ has a density function $p_\mu\in \cH^\beta([0,1]^d)$, it was proved in \cite{liang2021how,singh2018nonparametric} that the minimax convergence rates of nonparametric density estimation satisfy
\[
\inf_{\widetilde{\mu}_n} \sup_{p_\mu \in \cH^\beta([0,1]^d)} \bE d_{\cH^\alpha([0,1]^d)}(\mu, \widetilde{\mu}_n) \asymp n^{-(\alpha+\beta)/(2\beta +d)} \lor n^{-1/2},
\]
where the infimum is taken over all estimator $\widetilde{\mu}_n$ with density $p_{\widetilde{\mu}_n} \in \cH^\beta([0,1]^d)$ based on $n$ i.i.d. samples $\{X_i\}_{i=1}^n$ of $\mu$. Ignoring the logarithmic factor, Theorem \ref{rate theorem} gives the same convergence rate with $\beta=0$, which reveals the optimality of the result (since we do not assume the target has density in Theorem \ref{rate theorem}).

Under a priori that $p_\mu\in \cH^\beta$ for some $\beta>0$, it is not possible for the GAN estimators (\ref{gan estimator g_n 2}) and (\ref{gan estimator g_nm 2}) to learn the regularity of the target, because the empirical distribution $\widehat{\mu}_n$ do not inherit the regularity. However, we can use certain regularized empirical distribution $\widetilde{\mu}_n$ as the plug-in for GANs and consider the estimators
\begin{align}
\widetilde{g}^*_n &\in \left\{g\in \cG: d_\cF(\widetilde{\mu}_n, g_\# \nu) \le \inf_{\phi\in \cG} d_\cF(\widetilde{\mu}_n, \phi_\# \nu) + \epsilon_{opt} \right\}, \label{regularized gan estimator g_n} \\
\widetilde{g}^*_{n,m} &\in \left\{g\in \cG: d_\cF(\widetilde{\mu}_n, g_\# \widehat{\nu}_m) \le \inf_{\phi\in \cG} d_\cF(\widetilde{\mu}_n, \phi_\# \widehat{\nu}_m) + \epsilon_{opt} \right\}. \label{regularized gan estimator g_nm}
\end{align}
By choosing the regularized distribution $\widetilde{\mu}_n$, the generator $\cG$ and the discriminator $\cF$ properly, we show that $\widetilde{g}^*_n$ and $\widetilde{g}^*_{n,m}$ can achieve faster convergence rates than the GAN estimators (\ref{gan estimator g_n 2}) and (\ref{gan estimator g_nm 2}), which use the empirical distribution $\widehat{\mu}_n$ as the plug-in. The result can be seen as a complement to the nonparametric results in \cite[Theorem 3]{liang2021how}.

\begin{theorem}\label{learning density}
Suppose the target $\mu$ has a density function $p_\mu \in \cH^\beta([0,1]^d)$ for some $\beta>0$, the source distribution $\nu$ is absolutely continuous on $\bR$ and the evaluation class is $\cH = \cH^\alpha(\bR^d)$. Then, there exist a regularized empirical distribution $\widetilde{\mu}_n$ with density $p_{\widetilde{\mu}_n} \in \cH^\beta([0,1]^d)$, a generator $\cG = \{g\in \cN\cN(W_1,L_1): g(\bR) \subseteq [0,1]^d \}$ with
\[
W_1^2L_1 \lesssim n^{\frac{\alpha+\beta}{2\beta+d} \frac{d+\alpha +\sigma(2\alpha-2)}{\alpha}d},
\]
and a discriminator $\cF = \cN\cN(W_2,L_2) \cap \Lip(\bR^d, K,1)$ with
\[
W_2/ \log_2 W_2 \lesssim n^{\frac{\alpha+\beta}{2\beta+d} \frac{d}{2\alpha}}, \quad L_2 \asymp 1, \quad K \lesssim (W_2/\log_2 W_2)^{2+\sigma(4\alpha-4)/d} \lesssim n^{\frac{\alpha+\beta}{2\beta+d} \frac{d+ \sigma(2\alpha -2)}{\alpha}},
\]
such that the GAN estimator (\ref{regularized gan estimator g_n}) satisfies
\[
\bE [d_\cH(\mu,(\widetilde{g}^*_n)_\# \nu)] - \epsilon_{opt} \lesssim n^{-(\alpha+\beta)/(2\beta +d)} \lor n^{-1/2}.
\]

If furthermore $m \gtrsim n^{\frac{2\alpha+2\beta}{2\beta+d} (\frac{d +\alpha +\sigma(2\alpha-2)}{\alpha}d+1)} (\log n)^2$, then the GAN estimator (\ref{regularized gan estimator g_nm}) satisfies
\[
\bE [d_\cH(\mu,(\widetilde{g}^*_{n,m})_\# \nu)] - \epsilon_{opt} \lesssim n^{-(\alpha+\beta)/(2\beta +d)} \lor n^{-1/2}.
\]
\end{theorem}
\begin{proof}
\cite{liang2021how} and \cite{singh2018nonparametric} showed the existence of regularized empirical distribution $\widetilde{\mu}_n$ with density $p_{\widetilde{\mu}_n} \in \cH^\beta([0,1]^d)$ that satisfies
\[
\bE d_{\cH}(\mu, \widetilde{\mu}_n) \lesssim n^{-(\alpha+\beta)/(2\beta +d)} \lor n^{-1/2}.
\]
By Lemma \ref{general error decomposition}, we can decompose the error as
\[
d_\cH(\mu,(\widetilde{g}^*_n)_\# \nu) \le \epsilon_{opt} + 2\cE(\cH,\cF,[0,1]^d)  + \inf_{g \in \cG} d_\cF(\widetilde{\mu}_n,g_\# \nu) + d_\cH(\mu,\widetilde{\mu}_n).
\]
By Theorem \ref{app theorem width depth}, we can choose a discriminator $\cF$ that satisfies the condition in the theorem such that the discriminator approximation error can be bounded by
\[
\cE(\cH,\cF,[0,1]^d) \lesssim (W_2L_2 / (\log W_2 \log L_2))^{-2\alpha/d} \lesssim n^{-(\alpha+\beta)/(2\beta +d)}.
\]
For the generator approximation error, since $\cF \subseteq \Lip([0,1]^d,K)$,
\[
\inf_{g \in \cG} d_\cF(\widetilde{\mu}_n,g_\# \nu) \le K \inf_{g \in \cG} \cW_1(\widetilde{\mu}_n,g_\# \nu) \lesssim K (W_1^2L_1)^{-1/d},
\]
by Theorem \ref{app finite moment}. Hence, there exists a generator $\cG$ with $W_1^2L_1 \asymp n^{\frac{\alpha+\beta}{2\beta+d} \frac{d+\alpha +\sigma(2\alpha-2)}{\alpha}d}$ such that
\[
\inf_{g \in \cG} d_\cF(\widetilde{\mu}_n,g_\# \nu) \lesssim K (W_1^2L_1)^{-1/d} \lesssim n^{-(\alpha+\beta)/(2\beta +d)}.
\]
In summary, we have
\[
\bE d_\cH(\mu,(g^*_n)_\# \nu) - \epsilon_{opt} \lesssim n^{-(\alpha+\beta)/(2\beta +d)} \lor n^{-1/2}.
\]

For the estimator $\widetilde{g}^*_{n,m}$, we only need to further bound $d_{\cF \circ \cG}(\nu,\widehat{\nu}_m)$ due to Lemma \ref{general error decomposition}. By Lemma \ref{symmetrization} and \ref{Rc bound by pdim}, we can bound it using the pseudo-dimension of $\cF \circ \cG$:
\begin{align*}
\bE [d_{\cF \circ \cG}(\nu,\widehat{\nu}_m)] &\lesssim \sqrt{\frac{(W_1^2L_1+W_2^2L_2)(L_1+L_2)\log(W_1^2L_1+W_2^2L_2) \log m}{m}} \\
&\lesssim \sqrt{\frac{(n^{\frac{\alpha+\beta}{2\beta+d} \frac{d+\alpha +\sigma(2\alpha-2)}{\alpha}d} + n^{\frac{\alpha+\beta}{2\beta+d} \frac{d}{\alpha}} \log n) n^{\frac{\alpha+\beta}{2\beta+d} \frac{d+\alpha +\sigma(2\alpha-2)}{\alpha}d} \log n \log m}{m}} \\
&\lesssim n^{\frac{\alpha+\beta}{2\beta+d} \frac{d+\alpha +\sigma(2\alpha-2)}{\alpha}d} \sqrt{\frac{\log n \log m}{m}}.
\end{align*}
Since $m \gtrsim n^{\frac{2\alpha+2\beta}{2\beta+d} (\frac{d +\alpha +\sigma(2\alpha-2)}{\alpha}d+1)} (\log n)^2$, we have $\bE [d_{\cF \circ \cG}(\nu,\widehat{\nu}_m)] \lesssim n^{-(\alpha+\beta)/(2\beta +d)}$, which finishes the proof.
\end{proof}

As we noted in Remark \ref{proof essential remark}, the proof essentially shows that the convergence rates of $\widetilde{g}^*_n$ and $\widetilde{g}^*_{n,m}$ are not worse than the convergence rate of $\bE d_{\cH}(\mu, \widetilde{\mu}_n)$ if we choose the network architectures properly.

\subsection{Learning distributions with unbounded supports}

So far, we have assumed that the target distribution has a compact support. In this section, we show how to generalize the results to target distributions with unbounded supports. For simplicity, we only consider the case when the target $\mu$ is sub-exponential in the sense that
\begin{equation}\label{tail condition}
\mu(\{\Bx \in \bR^d: \|\Bx\|_\infty > \log t \}) \lesssim t^{-b/d},
\end{equation}
for some $b>0$. The basic idea is to truncate the target distribution and apply the error analysis to the truncated distribution.

\begin{theorem}\label{learning unbounded}
Suppose the target $\mu$ satisfies condition (\ref{tail condition}), the source distribution $\nu$ is absolutely continuous on $\bR$ and the evaluation class is $\cH = \cH^\alpha(\bR^d)$. Then, there exist a generator $\cG = \{g\in \cN\cN(W_1,L_1): g(\bR) \subseteq [-\alpha b^{-1} \log n, \alpha b^{-1} \log n]^d \}$ with
\[
W_1^2L_1 \lesssim n
\]
and a discriminator $\cF = \cN\cN(W_2,L_2) \cap \Lip(\bR^d, K,1)$ with
\[
W_2L_2 \lesssim n^{1/2} (\log n)^{2+d/2}, \quad K \lesssim (\widetilde{W}_2\widetilde{L}_2)^{2+\sigma(4\alpha-4)/d} \widetilde{L}_2 2^{\widetilde{L}_2^2}(2\alpha b^{-1}\log n)^{\alpha-1},
\]
where $\widetilde{W}_2 = W_2/\log_2 W_2$ and $\widetilde{L}_2 = L_2/\log_2 L_2$, such that the GAN estimator (\ref{gan estimator g_n}) satisfies
\[
\bE [d_\cH(\mu,(g^*_n)_\# \nu)] - \epsilon_{opt} \lesssim n^{-\alpha/d} \lor n^{-1/2}(\log n)^{c(\alpha,d)},
\]
where $c(\alpha,d)=1$ if $2\alpha = d$, and $c(\alpha,d)=0$ otherwise.

If furthermore $m\gtrsim n^{2+2\alpha/d}(\log n)^{6+d}$, then the GAN estimator (\ref{gan estimator g_nm}) satisfies
\[
\bE [d_\cH(\mu,(g^*_{n,m})_\# \nu)] - \epsilon_{opt} \lesssim n^{-\alpha/d} \lor n^{-1/2}(\log n)^{c(\alpha,d)}.
\]
\end{theorem}

\begin{proof}
Denote $A_n= [-\alpha b^{-1}\log n, \alpha b^{-1} \log n]^d$, then $1-\mu(A_n) \lesssim n^{-\alpha/d}$ by (\ref{tail condition}). We define an operator $\cT_n : \cP(\bR^d) \to \cP(A_n)$ on the set $\cP(\bR^d)$ of all probability distributions on $\bR^d$ by
\[
\cT_n \gamma = \gamma|_{A_n} + (1-\gamma(A_n)) \delta_{\boldsymbol{0}}, \quad \gamma \in \cP(\bR^d),
\]
where $\mu|_{A_n}$ is the restriction to $A_n$ and $\delta_{\boldsymbol{0}}$ is the point measure on the zero vector. Since any function $h\in \cH$ is bounded $\|h\|_{L^\infty }\le 1$, we have
\begin{align*}
d_\cH(\mu, \cT_n \mu) &= \sup_{h\in \cH} \int_{\bR^d} h(x) d\mu(x) - \int_{\bR^d} h(x) d\cT_n\mu(x) \\
&= \sup_{h\in \cH} \int_{\bR^d \setminus A_n} h(x) d\mu(x) - (1-\mu(A_n)) h(\boldsymbol{0}) \\
&\le 2 (1-\mu(A_n)) \lesssim n^{-\alpha/d}.
\end{align*}
As a consequence, by the triangle inequality,
\[
d_\cH(\mu,(g^*_n)_\# \nu) - d_\cH(\cT_n\mu,(g^*_n)_\# \nu) \le d_\cH(\mu, \cT_n \mu) \lesssim n^{-\alpha/d}.
\]
Since $\cT_n\mu$ and $g_\#\nu$ are supported on $A_n$ for all $g\in \cG$, by Lemma \ref{general error decomposition},
\[
d_\cH(\cT_n\mu,(g^*_n)_\# \nu) \le \epsilon_{opt} + 2\cE(\cH,\cF,A_n)  + \inf_{g \in \cG} d_\cF(\widehat{\mu}_n,g_\# \nu) + d_\cH(\cT_n\mu,\widehat{\mu}_n).
\]

For the discriminator approximation error, we need to approximate any function $h\in \cH^\alpha(A_n)$. We can consider the function $\widetilde{h} \in \cH^\alpha([0,1]^d)$ defined by
\[
\widetilde{h}(\Bx) = \frac{1}{(2\alpha b^{-1} \log n)^\alpha} h(\alpha b^{-1} \log n (2\Bx-1)).
\]
By Theorem \ref{app theorem width depth}, there exists $\widetilde{\phi} \in \cN\cN(W_2,L_2-1) \cap \Lip(\bR^d, K/(2\alpha b^{-1}\log n)^{\alpha-1},1)$ such that $\| \widetilde{h} - \widetilde{\phi} \|_{L^\infty([0,1]^d)} \lesssim (W_2L_2 / (\log W_2 \log L_2))^{-2\alpha/d}$. Define
\begin{align*}
\phi_0 (\Bx) &:= (2\alpha b^{-1} \log n)^\alpha \widetilde{\phi}\left( \tfrac{\Bx}{2\alpha b^{-1} \log n} + \tfrac{1}{2} \right), \\
\phi (\Bx) &:= \min\{ \max\{\phi_0(\Bx),-1\},1 \} = \sigma(\phi_0(\Bx)+1) - \sigma(\phi_0(\Bx)-1) -1,
\end{align*}
then $\phi \in \cN\cN(W_2,L_2) \cap \Lip(\bR^d, K,1)$ and
\[
\|h - \phi \|_{L^\infty(A_n)} \lesssim (W_2L_2 / (\log W_2 \log L_2))^{-2\alpha/d} (\log n)^\alpha.
\]
This shows that, if we choose $W_2L_2 \asymp n^{1/2} (\log n)^{2+d/2}$,
\[
\cE(\cH,\cF,A_n) \lesssim (W_2L_2 / (\log W_2 \log L_2))^{-2\alpha/d} (\log n)^\alpha \lesssim n^{-\alpha/d}.
\]
For the generator approximation error,
\[
\inf_{g \in \cG} d_\cF(\widehat{\mu}_n,g_\# \nu) \le d_\cF(\widehat{\mu}_n, \cT_n \widehat{\mu}_n) + \inf_{g \in \cG} d_\cF(\cT_n \widehat{\mu}_n,g_\# \nu).
\]
By Lemma \ref{app discrete measure}, we can choose a generator $\cG$ with $W_1^2L_1 \lesssim n$ such that the last term vanishes. Since $\|f\|_{L^\infty} \le 1$ for any $f\in \cF$, we have
\[
\bE d_\cF(\widehat{\mu}_n, \cT_n \widehat{\mu}_n) \le \bE [2 \widehat{\mu}_n(\bR^d \setminus A_n)] = 2\bE \left[ \frac{1}{n} \sum_{i=1}^n 1_{\{X_i\notin A_n\}} \right] = 2 \mu(\bR^d \setminus A_n) \lesssim n^{-\alpha/d}.
\]
For the generalization error, as in the proof of Theorem \ref{rate theorem},
\[
\bE d_\cH(\cT_n\mu,\widehat{\mu}_n) \le d_\cH(\cT_n\mu,\mu) + \bE d_\cH(\mu,\widehat{\mu}_n) \lesssim n^{-\alpha/d} \lor n^{-1/2}(\log n)^{c(\alpha,d)}.
\]

In summary, we have shown that
\[
\bE [d_\cH(\mu,(g^*_n)_\# \nu)] - \epsilon_{opt} \lesssim n^{-\alpha/d} \lor n^{-1/2}(\log n)^{c(\alpha,d)}.
\]
The error bound for $g^*_{n,m}$ can be estimated in a similar way.  By Lemma \ref{general error decomposition}, we only need to further bound $\bE [d_{\cF \circ \cG}(\nu,\widehat{\nu}_m)]$, which can be done as in the proof of Theorem \ref{rate theorem}.
\end{proof}
\begin{remark}
\textnormal{When $\alpha=1$, $\cH^1 = \Lip(\bR^d,1,1)$, the metric $d_{\cH^1}$ is the Dudley metric. For the Wasserstein distance $\cW_1$, we let $A_n = [2b^{-1}\log n, 2b^{-1}\log n]^d$, then
\begin{align*}
\cW_1(\mu, \cT_n \mu) &= \sup_{\Lip (h)\le 1} \int_{\bR^d \setminus A_n} h(x)-h(\boldsymbol{0}) d\mu(x) \le \int_{\bR^d \setminus A_n} \|x\|_\infty d\mu(x) \\
&\le \bE[\|X\|_\infty 1_{\{X\notin A_n\}}] = \int_0^\infty \mu(\|X\|_\infty 1_{\{X\notin A_n\}} >t) dt \\
&\lesssim \int_0^{2b^{-1}\log n} n^{-2/d} dt + \int_{2b^{-1}\log n}^\infty 2^{-bt/d} dt \\
&\lesssim n^{-2/d} \log n.
\end{align*}
If we choose the generator $\cG = \{g\in \cN\cN(W_1,L_1): g(\bR) \subseteq A_n \}$ and the discriminator $\cF = \cN\cN(W_2,L_2) \cap \Lip(\bR^d, K,2b^{-1} \log n)$ satisfying the conditions in Theorem \ref{learning unbounded} with $\alpha =1$, one can show that
\[
\bE [\cW_1(\mu,(g^*_n)_\# \nu)] - \epsilon_{opt} \lesssim n^{-1/d} \lor n^{-1/2}(\log n)^{c(\alpha,d)},
\]
where the same convergence rate holds for $\bE \cW_1(\mu,\widehat{\mu}_n)$ by \cite{fournier2015rate}. When $m$ is chosen properly, the same rate holds for the estimator $g^*_{n,m}$.
}
\end{remark}

\chapter{Conclusions and Future Work}\label{chapter: conclusion}

In this thesis, we develop an error analysis for generative adversarial networks. Firstly, we prove a new oracle inequality for GAN estimators, which decomposes the estimation error into optimization error, generator and discriminator approximation error and generalization error. To estimate the discriminator approximation error, we derive error bounds on approximating H\"older functions by deep neural networks with explicit Lipschitz control on the network or norm constraint on the weights. For generator approximation error, we show that generative networks are universal approximators and obtain approximation bounds in terms of width and depth. The generalization error is controlled by the complexity of the function class using statistical learning theory. Finally, we prove the convergence rates for GAN estimators in various settings.

The analysis in this thesis also arises many problems in the theory of deep learning. In the following, we list some possible directions for future research.
\begin{itemize}[parsep=0pt]
\item \textbf{(Optimization)} As noted in Remark \ref{remark: opt}, we do not analyze the optimization error in this thesis, since the optimization dynamic for GAN is still a very challenging problem. For regression and classification, it has been shown that gradient descent can find global minima for over-parameterized neural networks \cite{allenzhu2019convergence,du2019gradient,liu2022loss}. It is interesting to see whether over-parameterization helps the optimization in GANs.

\item \textbf{(Function Approximation)} There is a gap between the upper bound in Theorem \ref{app theorem norm constraint} and lower bound in Theorem \ref{app lower bound norm constraint} for the approximation by norm constrained neural networks. In Theorem \ref{app theorem width depth} and \cite{yarotsky2018optimal,shen2020deep}, the optimal approximation rates, in terms of the numbers of weights and neurons, are derived through bit extraction technique. So, we think it might be possible to apply bit extraction technique to construct norm constrained neural networks that have better approximation rates. A related question is how small the bound for Lipschitz constant in Theorem \ref{app theorem width depth} can be.

\item \textbf{(Distribution Approximation)} We show that generative networks can approximate discrete distributions arbitrarily well in Lemma \ref{app discrete measure}. However, in our construction, the weights of neural network diverge to infinity when the approximation error approaches zero. Using the space-filling approach discovered in \cite{bailey2018size}, a recent paper \cite{perekrestenko2021high} estimated approximation bounds for generative networks with bounded weights, under the assumption that the source distribution is uniform. It will be interesting to see whether their proof techniques can be combined with our analysis and applied to more general settings.

\item \textbf{(Generalization)} The Rademacher complexity upper bound for norm constrained neural networks in Lemma \ref{Rademacher bound} is independent of the width, but depends
on the depth. It is still unclear whether it is possible to obtain size-independent bounds without further assumption on the weights of neural networks. Such bound also affect the lower bound for approximation error, since our lower bound in Theorem \ref{app lower bound norm constraint} is derived through the upper bound for Rademacher complexity.

\item \textbf{(Regularization)} Many regularization techniques \cite{gulrajani2017improved,kodali2017convergence,petzka2018regularization,wei2018improving,thanhtung2019improving} have been applied to GANs and shown to have good empirical performance. In our definition of norm constraint (\ref{norm constraint}), we restrict ourselves to the operator norm induced by $\|\cdot\|_\infty$ for weight matrices. It will be interesting to extend the results to other norms and other regularization methods. A more fundamental question is how different norms affect the approximation and generalization capacity?
\end{itemize}

\bibliographystyle{plain}

\bibliography{References}




\end{document}